\documentclass[11pt]{article}
\usepackage{mathrsfs,amsmath,amsfonts,amssymb,bm,eufrak,pifont,amscd,euscript,amsthm,color,epsfig,xr,tikz,verbatim,enumitem}

\usepackage{authblk}
\usepackage[utf8]{inputenc} %
\usepackage[T1]{fontenc}    %
\usepackage{url}            %
\usepackage{booktabs}       %
\usepackage{amsfonts}  
\usepackage{amsmath} 
\usepackage{nicefrac}       %
\usepackage{microtype}      %
\usepackage{xcolor}
\usepackage{hyperref}  
\usepackage{cleveref}
\usepackage{graphicx}
\usepackage{array}
\usepackage{subcaption,mathrsfs}
\usepackage{natbib}
\usepackage{wrapfig}
\usepackage{tikz}
\usepackage[makeroom]{cancel}
\usepackage{subcaption}

\allowdisplaybreaks
\usetikzlibrary{decorations.markings}

\definecolor{mydarkblue}{rgb}{0,0.08,0.45}

\hypersetup{ %
    pdftitle={},
    pdfauthor={},
    pdfsubject={},
    pdfkeywords={},
    pdfborder=0 0 0,
    pdfpagemode=UseNone,
    colorlinks=true,
    linkcolor=mydarkblue,
    citecolor=mydarkblue,
    filecolor=mydarkblue,
    urlcolor=mydarkblue,
    pdfview=FitH
}

\Crefname{assumption}{Assumption}{Assumption} 
\crefname{assumption}{Assumption}{Assumption} 

\usepackage{aliascnt}

\newenvironment{talign*}
 {\csname align*\endcsname}
 {\endalign}
\newenvironment{talign}
 {\csname align\endcsname}
 {\endalign}

\usepackage{tikz}
\usetikzlibrary{calc}
\usepackage[title]{appendix}

\newtheorem{theorem}{Theorem}[section]
\crefname{theorem}{Theorem}{Theorems}
\Crefname{theorem}{Theorem}{Theorems}
\newaliascnt{lem}{theorem}
\newtheorem{lem}[lem]{Lemma}
\aliascntresetthe{lem}
\crefname{lem}{Lemma}{Lemmas}
\Crefname{lem}{Lemma}{Lemmas}

\newaliascnt{prop}{theorem}
\newtheorem{prop}[prop]{Proposition} 
\aliascntresetthe{prop}
\crefname{prop}{Proposition}{Propositions}
\Crefname{prop}{Proposition}{Propositions}

\newtheorem{ass}{Assumption}
\crefname{ass}{Assumption}{Assumptions}
\Crefname{ass}{Assumption}{Assumptions}

\newaliascnt{rem}{theorem}
\newtheorem{rem}[rem]{Remark} 
\aliascntresetthe{rem}
\crefname{rem}{Remark}{Remarks}
\Crefname{rem}{Remark}{Remarks}

\newaliascnt{example}{theorem}
\newtheorem{example}[example]{Example} 
\aliascntresetthe{example}
\crefname{example}{Example}{Examples}
\Crefname{example}{Example}{Examples}

\newaliascnt{defi}{theorem}
 
\aliascntresetthe{defi}
\crefname{defi}{Definition}{Definitions}
\Crefname{defi}{Definition}{Definitions}

\newaliascnt{cor}{theorem}
 
\aliascntresetthe{cor}
\crefname{cor}{Corollary}{Corollaries}
\Crefname{cor}{Corollary}{Corollaries}
 
\newcommand{\E}{\mathbb{E}}
\newcommand{\R}{\mathbb{R}}
\newcommand{\Qb}{\mathbb{Q}}
\newcommand{\Pb}{\mathbb{P}}
\newcommand{\calH}{\mathcal{H}}

\newcommand{\calT}{\mathcal{T}}

\newcommand{\calI}{\mathcal{I}}
\newcommand{\calF}{\mathcal{F}}

\newcommand{\calU}{\mathcal{U}}
\newcommand{\calZ}{\mathcal{Z}}

\newcommand{\calO}{\mathcal{O}}

\newcommand{\calP}{\mathcal{P}}
\newcommand{\calN}{\mathcal{N}}
\newcommand{\calJ}{\mathcal{J}}

\newcommand{\Gmap}{G}
\newcommand{\dd}{\mathrm{d}}

\newcommand{\bJ}{\mathbf{J}}

\newcommand{\mmd}{\operatorname{MMD}}

\newcommand{\Id}{\mathrm{Id}}
\newcommand{\N}{\mathbb{N}}
\newcommand{\kl}{\mathrm{KL}}

\newcommand{\pgd}{\mathrm{PGD}}
\newcommand{\gd}{\mathrm{GD}}

\newcommand{\tildePtheta}[1]{\Pb_{\theta_{#1,m}}^{\pgd}}
\newcommand{\tildePthetanopgd}[1]{\Pb_{\theta_{#1,m}}}
\newcommand{\tildetheta}[1]{\Delta \theta_{#1,m}^\pgd}
\newcommand{\widehattheta}[1]{\widehat{\Delta\theta_{#1, m}^\pgd}}

\definecolor{mydarkblue}{rgb}{0,0.08,0.45}

\hypersetup{ %
    pdftitle={},
    pdfauthor={},
    pdfsubject={},
    pdfkeywords={},
    pdfborder=0 0 0,
    pdfpagemode=UseNone,
    colorlinks=true,
    linkcolor=mydarkblue,
    citecolor=mydarkblue,
    filecolor=mydarkblue,
    urlcolor=mydarkblue,
    pdfview=FitH
}

\makeatletter
\newcommand{\customlabel}[2]{%
   \protected@write \@auxout {}{\string \newlabel {#1}{{#2}{\thepage}{#2}{#1}{}} }%
   \hypertarget{#1}{}
}
\makeatother

\title{A Gradient Flow Perspective on Minimum MMD Estimation}

\author{
Sophia Seulkee Kang$^{1,\dagger}$,
Louis Sharrock$^2$,
Xiaoyuan Cheng$^{2}$,
Fran\c{c}ois-Xavier Briol$^{2}$,
Zonghao Chen$^{2,\dagger}$\\
\small $^1$Independent Researcher,
\small $^2$University College London, 
\small $^\dagger$Equal contribution.
}

\begin{document}
\maketitle
\begin{abstract}
    Minimum maximum mean discrepancy (MMD) estimation has emerged as a robust and likelihood-free alternative to maximum likelihood estimation for parameter estimation. Yet, despite its practical success, the associated optimization problem remains poorly understood, with theoretical guarantees for existing algorithms hinging on convexity assumptions that rarely hold in practice.
    We address this gap by proposing a preconditioned gradient descent (PGD) scheme, establishing its asymptotic \emph{global} convergence under explicit gradient-dominance and projection-residual conditions. 
    Our approach is inspired by recent progress on MMD gradient flows, a nonparametric descent scheme on the space of probability measures. 
    We provide extensive empirical evidence that our PGD scheme outperforms standard gradient descent across a range of challenging parameter estimation and composite hypothesis testing problems. 
\end{abstract}

\section{Introduction}

Parameter estimation seeks to recover the unknown parameters governing a data-generating mechanism from observed data. Given a parametric family of statistical models $\{\Pb_\theta: \theta \in \Theta\}$, where $\Theta\subseteq\R^p$ is the parameter domain, a standard approach is to estimate these parameters by minimizing a divergence or discrepancy between the model and the data-generating distribution $\Qb$~\citep{wolfowitz1957minimum,parr1980minimum,basu2011statistical}.
Since the latter is unknown, an estimator is typically constructed by replacing it with the empirical distribution consisting of independently and identically distributed (IID) observations.
A canonical choice of such discrepancy is the Kullback--Leibler divergence, for which the resulting minimum-distance estimator coincides with maximum likelihood estimation~\citep{van2000asymptotic}.

Recently, a growing body of work has considered using the maximum mean discrepancy (MMD), a metric based on embeddings of probability measures into a reproducing kernel Hilbert space~\citep{gretton2012kernel,muandet2017kernel}, leading to the so-called \emph{minimum MMD} estimator~\citep{briol2019statistical}.
This choice is attractive for both statistical and computational reasons. Statistically, minimum MMD estimators are robust to model misspecification, such as outlier contamination and covariate shift~\citep{briol2019statistical,cherief2020mmd,dellaporta2023robust}. Computationally, they remain well defined even when the likelihood is unavailable or intractable; such settings are commonly referred to as likelihood-free inference~\citep{niu2021discrepancy,dellaporta2022robust,bharti2023optimally}. These favorable properties have also led to numerous applications of minimum MMD estimators, including for hypothesis testing~\citep{key2025composite,schrab2023mmd,bruck2025distribution,fazeli2023semi}, parametric regression~\citep{alquier2024universal}, time-series modeling~\citep{alquier2026timeseries}, distributionally robust optimization~\citep{dellaporta2025decision}, generalizations of Bayesian inference \citep{dellaporta2022robust,dellaporta2023robust,chen2025total}, calibration of stochastic simulation models in operations research~\citep{su2025differentiable}, copula modeling~\citep{alquier2023estimation}, inference under missing data~\citep{cherief2025parametric}, and learning of summary statistics~\citep{khoo2026minimum}.

Despite the growing popularity of these estimators, very little work has focused on solving the associated optimization problem, $\arg\min_{\theta\in\Theta} \mmd^2(\Pb_\theta,\Qb)$, compared to its accompanying statistical theory. 
Existing methods typically rely on (stochastic) gradient descent~\citep{briol2019statistical,cherief2020mmd}, including the recent R-package implementation of \cite{alquier2025package}, and this can at best be expected to converge to stationary points due to the non-convexity of the objective.
To the best of our knowledge, the only existing optimization guarantee is Proposition~5.2 of \cite{cherief2022finite}, which assumes convexity of the objective, a condition that rarely holds in practice; see \Cref{fig:illustration} for an illustration on a simple two-mode Gaussian mixture model with unknown variances.
Therefore, an open problem remains:
\begin{center}
\emph{``Can we design a gradient-based optimization scheme for minimum MMD estimation with convergence guarantees beyond convex objectives?''}
\end{center}

\definecolor{darkgreen}{RGB}{0,100,0}
\begin{figure*}[t]
    \centering
    \begin{minipage}[t]{0.66\linewidth}
        \centering
        \includegraphics[width=\linewidth]{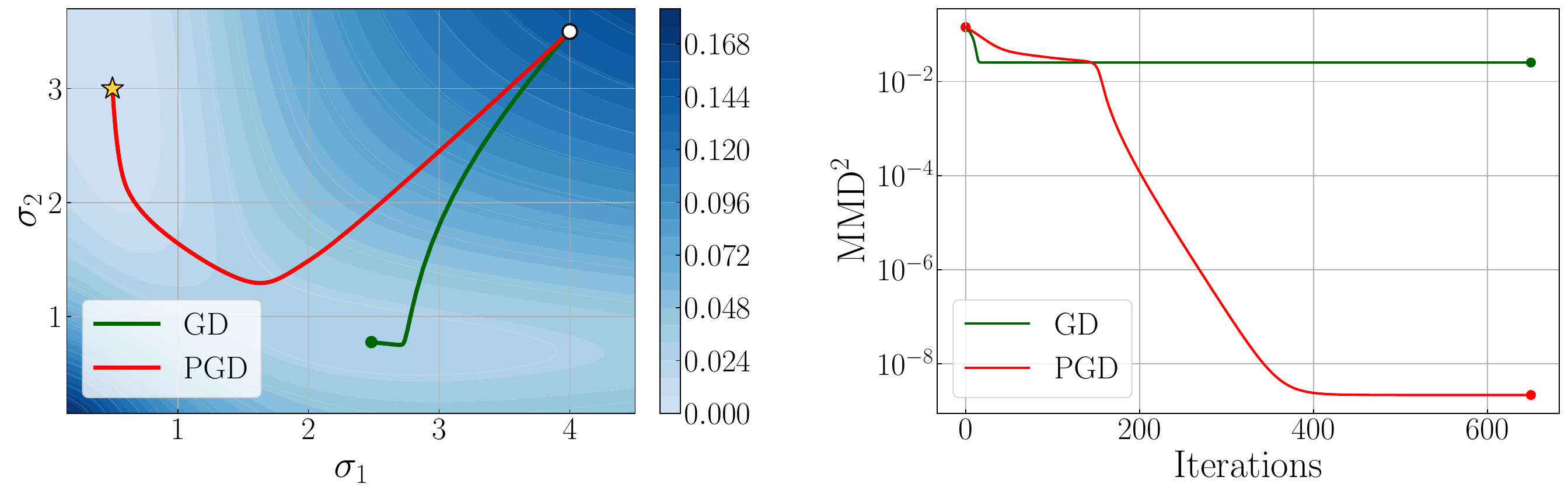}
    \end{minipage}
    \vspace{-10pt}
    \caption{\textit{Illustrative example.} Both the target distribution $\Qb$ and the model distribution $\Pb_\theta$ are two-mode Gaussian mixtures, where $\theta=(\sigma_1,\sigma_2)$ denotes the standard deviations of the two modes. We compare minimum MMD estimators trained by standard gradient descent (GD) and our preconditioned gradient descent (PGD). The left panel shows the optimization trajectories, with the yellow star marking the global minimizer, while the right panel shows the evolution of the MMD objective values.
    The objective landscape is non-convex. See \Cref{fig:lengthscale_landscape} for details.}
    \label{fig:illustration}
\vspace{-10pt}
\end{figure*}

In this paper, we make progress on this problem by proposing a new descent scheme for minimizing the MMD objective, and prove that it can converge to a (not necessarily unique) \emph{global} minimizer under explicit assumptions.
Our approach builds on MMD gradient flow, a descent scheme for minimizing the squared MMD directly over the space of probability measures, without imposing a parametric model representation~\citep{arbel2019maximum}.
Existing MMD gradient flows have been studied primarily in the contexts of generative modelling~\citep{galashov2024deep,hertrich2023generative,neuralwgf,hertrich2024wasserstein,chen2025regularized,tian2026sobolev} and sampling~\citep{korba2021kernel,chen2025stationary,tian2026sobolev}, and are typically run with a fixed kernel throughout the dynamics.
We depart from this fixed-kernel formulation by introducing an \emph{adaptive} nonparametric MMD descent scheme $(\Pb_t)_{t\in\N}$, in which the kernel lengthscale $(\ell_t)_{t\in\N}$ is varied during optimization according to a schedule converging to a prescribed limit $\ell>0$~\citep{chen2021deterministic}.
Our first main result, \Cref{thm:np_rate}, proves that this adaptive scheme satisfies $\lim_{t\to\infty}\mmd(\Pb_t,\Qb)=0$, under a trajectory-wise gradient-dominance condition.

Motivated by this global convergence result, we then design a descent scheme  such that the parametric distributions $(\Pb_{\theta_t})_{t\in\N}$ closely approximate $(\Pb_t)_{t\in\N}$ at each iteration $t$.
This is achieved via projecting the nonparametric descent direction onto the space of directions realizable by the parametric model.
Compared with the standard gradient, the resulting update on $\theta$ involves a preconditioning matrix. We therefore refer to it as \emph{preconditioned} gradient descent (PGD).
We next prove in \Cref{thm:p_rate} that, under gradient dominance and projection residual conditions, the parametric models optimized under PGD, denoted as $(\Pb_{\theta_t}^\pgd)_{t\in\N}$, satisfy $\lim_{t\to\infty}\mmd(\Pb_{\theta_t}^\pgd, \Qb)=\min_{\theta\in\Theta} \mmd(\Pb_{\theta}^\pgd, \Qb)$, even in non-convex settings.
Therefore, our results help bridge the gap between the statistical theory of minimum MMD estimation and its practical computation.
Finally, the effectiveness of the proposed PGD descent scheme is empirically verified on several statistical inference benchmarks where we show that the PGD scheme not only converges to a global minimum, but can also accelerate convergence relative to gradient descent.

\vspace{1mm}
\textbf{Notations:}
Let $\calP_2(\R^d)$ denote the space of probability measures on $\R^d$ with finite second moment.
For a continuous map $T:\R^d\to\R^d$, let $T_\#\mu$ denote the push-forward of $\mu$.
For $\mu\in\calP_2(\R^d)$, define $L_2(\mu)$ as the space of real-valued $\mu$-square-integrable functions, with
$\|f\|_{L_2(\mu)}^2:=\int_{\R^d}|f(x)|^2\,\dd\mu(x)$; define $L_2^d(\mu)$ analogously for $\R^d$-valued functions, with
$\|f\|_{L_2^d(\mu)}^2:=\int_{\R^d}\|f(x)\|^2 \,\dd\mu(x)$.
For a $\R$-valued function $f:\R^d\to\R$, let $\nabla f$ denote its derivative.
For a $\R^d$ valued function $f:\R^d\to\R^d$, let $\nabla\cdot f$ denote its divergence.
For a set $A$, let ${\mathrm{span}}(A)$ denote the its linear span.
Let $\Id$ denote an identity matrix or operator, depending on context.
For $S\in\N$, write $[S]:=\{0,1,\ldots,S\}$.

\vspace{-5pt}
\section{Background \& Related Work}\label{sec:background}
\vspace{-5pt}
In this section, we briefly review the maximum mean discrepancy (MMD), minimum MMD estimation, MMD gradient flows, and discuss closely related work.

\vspace{1mm}
\textbf{Maximum mean discrepancy.}
Let $k:\R^d\times\R^d\to\R$ be a symmetric positive semi-definite kernel. We write $\calH$ for the associated reproducing kernel Hilbert space (RKHS), with inner product $\langle\cdot,\cdot\rangle_{\calH}$ and norm $\|\cdot\|_{\calH}$, and recall that $k$ satisfies a reproducing property; i.e. $f(x)=\langle f,k(\cdot,x)\rangle_{\mathcal{H}}$ for all $f\in \mathcal{H}$~\citep{aronszajn1950theory}.
For any probability measure $\mu$ satisfying $\int \sqrt{k(x,x)} \dd\mu(x) < \infty$, its kernel mean embedding is defined by
$\int k(x,\cdot)\dd\mu(x) \in \calH$~\citep{muandet2017kernel}.
Given two probability measures $\mu$ and $\pi$ for which these embeddings are well defined, the maximum mean discrepancy (MMD)~\citep{gretton2012kernel} is defined as
$\mmd(\mu , \pi):= \left\|\int k(x,\cdot)\dd\mu(x) -\int k(x,\cdot)\dd\pi(x) \right\|_{\mathcal{H}}$.
When the kernel $k$ is $c_0$-universal, e.g. a Gaussian kernel $k(x, y)=\exp(-0.5 \ell^{-2}\|x-y\|^2)$ with lengthscale $\ell>0$, $\operatorname{MMD}(\mu, \pi)=0 \Leftrightarrow \mu=\pi$~\citep{sriperumbudur2010hilbert}, metrizing the weak topology between probability measures~\citep{simon2023metrizing}.

\vspace{1mm}
\textbf{Minimum MMD estimators.}
Given an unknown distribution $\Qb$ on $\R^d$ and a parametric model class
$\{\Pb_\theta:\theta\in\Theta\}$, with $\Theta\subseteq\R^p$, minimum MMD estimation \citep{briol2019statistical} seeks a parameter whose associated model distribution is closest to $\Qb$ in the MMD sense;
\begin{align}\label{eq:mmmd_estimator}
    \theta^\ast :=  \arg\min_{\theta\in\Theta}\mmd^2(\Pb_\theta,\Qb) = \arg\min_{\theta\in\Theta} \Big\{ \E_{x,x'\sim\Pb_\theta}[k(x,x')]
    -2\E_{x\sim\Pb_\theta,\,y\sim\Qb}[k(x,y)]
    + \mathrm{const} \Big\},
\end{align}
where the second equality follows from the reproducing property.
Under continuity of $\theta\mapsto \mmd^2(\Pb_\theta,\Qb)$ and compactness of $\Theta$, a global minimizer $\theta^\ast$ exists, although it need not be unique~\citep[Section 4]{oates2024minimum}. 
Minimum MMD estimation can be viewed as an infinite-dimensional analogue of the generalized method of moments~\citep{hall2004generalized}, where the finite vector of moments is replaced by the kernel mean embedding in the RKHS.
The model is said to be \emph{well-specified} if $\Qb\in\{\Pb_\theta:\theta\in\Theta\}$; in this case, if the kernel is $c_0$-universal then the minimum MMD estimator satisfies $\mmd(\Pb_{\theta^\ast},\Qb)=0$.
Otherwise the model is said to be \emph{misspecified} and $\mmd(\Pb_{\theta^\ast},\Qb)>0$.

As highlighted above, minimum MMD estimators have emerged as an alternative to maximum likelihood estimators, due to its robustness to model misspecification caused by outlier contamination and its applicability in likelihood-free settings.
The robustness property follows from the fact that, for bounded kernels, the contribution of any contaminated observation to the MMD objective remains uniformly controlled, regardless of its magnitude.
The likelihood-free property follows from Eq.~\eqref{eq:mmmd_estimator}: the objective admits an unbiased finite-sample estimator based only on a U-statistic computed from IID samples from the data-generating process $\Qb$ and the model $\Pb_\theta$, without requiring its density, though more advanced estimators sample-based estimators have also been proposed~\citep{niu2021discrepancy,bharti2023optimally}. Under relatively standard regularity conditions, minimum MMD estimators are consistent and satisfy a central limit theorem \citep{briol2019statistical}, allowing for the construction of confidence intervals \citep{su2025differentiable}.

Despite these favorable statistical and computational properties, comparatively less attention has been paid to the optimization problem underlying minimum MMD estimation. The vast majority of the literature relies on gradient descent, which, due to
the common non-convexity of the MMD objective, generally converges only to a stationary point. 
A related natural-gradient-type descent scheme was proposed by \cite{briol2019statistical}; see \Cref{rem:connection_ngd}.
However, this approach also does not come with explicit convergence guarantees.
Without such guarantees, the established robustness and computational advantages of minimum MMD estimators, while statistically appealing, may be difficult to realize in practice.

\vspace{1mm}
\textbf{MMD gradient flow}
Given a target distribution $\Qb$, MMD gradient flow refers to an absolutely continuous curve $(\Pb_s)_{s\geq 0}$ in $\calP_2(\R^d)$ that decreases the squared MMD distance to the target $\Qb$ in the steepest descent direction~\citep{villani2009optimal,ambrosio2005gradient,arbel2019maximum}.
Over recent years, MMD gradient flow has emerged as a popular approach for sampling from unnormalized targets~\citep{korba2021kernel,chen2025stationary,tian2026sobolev} and generative modelling~\citep{hertrich2023generative,hertrich2023wasserstein,hertrich2024wasserstein,galashov2024deep,chen2025regularized,gretton2026wasserstein}.
Let $f_{\Pb,\Qb}(\cdot) = \int k(x, \cdot) \dd \Pb(x) - \int k(x, \cdot) \dd\Qb(x)$.
The MMD flow $(\Pb_s)_{s\geq 0}$ can be described by the continuity equation~\citep{ambrosio2005gradient,arbel2019maximum}:
\begin{equation*}
\partial_s \Pb_s - \nabla\cdot(\Pb_s \nabla f_{\Pb_s,\Qb} ) = 0,
\end{equation*}
describing that the infinitesimal changes in $\Pb_s$ dictated by the vector field $\nabla f_{\Pb_s,\Qb}$. Note that the continuity equation shall be understood in the distribution sense, see Eq. (8.1.3) in \cite{ambrosio2005gradient}.
Its well-posedness under mild conditions on $k$ is established in Proposition~B.8 of \cite{tian2026sobolev}.
In practice, this is implemented by a discrete-time descent scheme.
For all $t\in\N$, given a step size $\gamma_t > 0$, the update is given by
\begin{align}\label{eq:np_mmd_descent}
    \Pb_{t+1} = (\Id - \gamma_t \nabla f_{\Pb_t,\Qb})_\# \Pb_t , \quad \text{where} \quad \nabla f_{\Pb_t,\Qb}(x) = \nabla \left[ \int k(z, x) \dd \Pb_t(z) - \int k(z, x) \dd\Qb(z) \right].
\end{align}
We refer to Eq.~\eqref{eq:np_mmd_descent} as \emph{nonparametric} MMD descent, because the iterates evolve in $\calP_2(\R^d)$, without being constrained by any parametric representation.
Although Eq.~\eqref{eq:np_mmd_descent} is guaranteed to decrease MMD at each step provided the step size $\gamma_t$ is small enough~\citep[Prop. 4]{arbel2019maximum}, the MMD gradient descent scheme $(\Pb_{t})_{t\in[T]}$ might still fail to converge to the target distribution $\Qb$ even with infinite number of steps, i.e., $\lim_{t\to\infty} \mmd(\Pb_t, \Qb) > 0$~\citep{arbel2019maximum,chen2025regularized}.
This lack of convergence has also been observed in practice~\citep{chen2025regularized,belhadji2026weighted}.
Indeed, this is a well-known challenge for MMD flows that arises from the lack of convexity of the mapping $\Pb\mapsto \mmd^2(\Pb, \Qb)$ with respect to the Wasserstein-2 metric~\citep{chen2025regularized}.
 Existing solutions are either noise injection~\citep{arbel2019maximum,chen2025stationary} or entropic regularizations~\citep{nitanda2022convex}, which unfortunately, cannot be used in our setting as they are not easily compatible with parametric MMD descent.

\section{Nonparametric MMD Descent with Adaptive Lengthscale}
In this section, we first establish a convergence guarantee for nonparametric MMD descent with \emph{adaptive} kernel lengthscales. We propose to initiate the MMD gradient flow algorithm with a large lengthscale, then to monotonically reduce it towards some target value as the number of iterations increases.
This is known empirically to improve the convergence behavior of MMD flows~\citep{chen2021deterministic,galashov2024deep,chen2025regularized}, and to even be useful for other statistical applications of the MMD~\citep{schrab2023mmd,biggs2023mmd}. 
The intuition for adaptive lengthscale is analogous to that of annealing methods: adapting the objective function throughout the optimisation routine can significantly improve exploration. We also note that adaptive lengthscale will be straightforward to use directly in the context of parametric updates. 
Hence, the convergence guarantee proved here provides one foundation for the optimization analysis of parametric MMD descent in \Cref{sec:parametric}.

Let $f_{\ell, \Pb, \Qb}$ denote the witness function between $\Pb$ and $\Qb$ under kernel $k_\ell$; and  $\mmd_{\ell}(\Pb, \Qb)$ the corresponding MMD. 
We now add subscripts to emphasize the dependence on the kernel lengthscale $\ell$. 
Given a sequence of lengthscales $(\ell_t)_{t\in\mathbb{N}}$, the adaptive MMD descent scheme is defined by 
\begin{align}\label{eq:adaptive_mmd_descent}
    \Pb_{t+1} = (\Id - \gamma_t \nabla f_{\ell_t, \Pb_t,\Qb})_\# \Pb_t .  
\end{align}
In the remainder of this paper, we focus explicitly on the Gaussian and Mat\'ern kernels,  due to their widespread use for minimum MMD estimation~\citep{briol2019statistical,cherief2020mmd}; see \Cref{rem:extension_general_kernel} for an extension to other kernels. 

\begin{ass}[Kernel]\label{ass:kernel}
Let \(k_\ell\) be either a Gaussian kernel
\begin{align*}
    k_\ell(x,y) = \exp\left(-\frac{\|x-y\|^2}{2\ell^2} \right), 
\end{align*} 
or the Mat\'ern kernel with smoothness parameter \(s>1\),
\begin{align*}
    k_\ell(x,y)
    =
    \frac{2^{1-s}}{\Gamma(s)}
    \left(
        \frac{\sqrt{2s}\|x-y\|}{\ell}
    \right)^s 
    K_s \left( \frac{\sqrt{2s}\|x-y\|}{\ell}
    \right),
\end{align*}
Here, $l>0$ is a lengthscale, and $\Gamma$ and $K_s$ denote the Gamma function and the modified Bessel function of the second kind respectively.
\end{ass}

\begin{theorem}[Global convergence of nonparametric MMD descent] \label{thm:np_rate} 
Let kernel $k$ satisfy \Cref{ass:kernel}. 
Suppose $\Qb\in\calP_2(\R^d)$. 
Let $\Pb_0\in\calP_2(\R^d)$ and let 
$(\Pb_t)_{t\in\N}$ be the adaptive MMD descent defined in Eq.~\eqref{eq:adaptive_mmd_descent} with lengthscale $\ell_t$. 
Let $C>0$ be a constant independent of $t$ and $(\gamma_t)_{t\in\N}$ a monotonically decreasing sequence of step sizes such that $\gamma_t\in(0,\frac{1}{2C})$, $\lim_{t\to\infty} \gamma_t=0$ and $\sum_{t=0}^\infty \gamma_t=\infty$. 
Let $(\ell_t)_{t\in\N}$ be a monotonically decreasing sequence of lengthscales such that $\lim_{t\to\infty} \ell_t = \ell > 0$ and that $\left(\ell_t \ell^{-1}\right)^d \|\nabla f_{\ell, \Pb_t,\Qb} \|_{L_2^d(\Pb_t)}^2 \geq C \mmd_{\ell}^2(\Pb_{t}, \Qb)$ for all $t\in\N$. Then $\lim_{t\to\infty} \mmd_{\ell}^2(\Pb_t, \Qb) = 0$. 
\end{theorem}
The proof is in \Cref{sec:proof_np_descent}. 
\Cref{thm:np_rate} shows that, under its step-size and trajectory-wise gradient-dominance assumptions, an appropriate kernel lengthscale sequence ensures that $\Pb_t$ converges to $\Qb$ in MMD despite the potential non-convexity of the objective.
A non-asymptotic upper bound on $\mmd_{\ell}^2(\Pb_t, \Qb)$ is provided in Eq.~\eqref{eq:term_1}; however, its interpretation is less transparent without further assumptions on $(\ell_t)_{t\in\N}$. We therefore choose to present only the asymptotic result. 

The condition of a diminishing yet cumulatively unbounded step size is standard in convex optimization \cite[e.g.,][]{moulines2011non}. The lengthscale condition 
can be interpreted as a trajectory-wise gradient dominance condition, and can be verified \emph{a posteriori}; see Section~\ref{sec:experiments}. 
Since $\ell_t$ appears only on the left-hand side of the inequality, this suggests that larger kernel lengthscales can make the sufficient condition easier to satisfy. 
This is illustrated in \Cref{fig:lengthscale_landscape}, where increasing the lengthscale yields a more `benign' optimization landscape via more smoothing~\citep{starnes2026gaussian}. 
Therefore, it is natural to initialize with a relatively large lengthscale and gradually decrease it toward $\ell$ as optimization proceeds. 
As $\Pb_t$ approaches $\Qb$, the relevant local optimization landscape may behave more like a quadratic, and global convergence can be expected even with smaller lengthscales. 
Furthermore, since $\ell_t/\ell > 1$, the prefactor on the left-hand side increases with $d$, and therefore the condition is likely less restrictive in higher dimensions. 
A smaller value of $C$ makes the condition easier to satisfy, but at the cost of a slower convergence rate from Eq.~\eqref{eq:term_1}. 
Similar conditions have been used in the convergence analysis of kernel-based gradient flows; see Proposition 8 of \cite{arbel2019maximum}, Proposition 5 of \cite{glaser2021kale}, Theorem 4.1 of \cite{chen2025stationary}, and also in graduated optimization~\citep{hazan2016graduated}. 
Their assumptions involve expectations with respect to an additional injected noise, whereas our condition is deterministic and only needs to be checked along the current \emph{realized} trajectory. 

\section{Parametric MMD Descent with Adaptive Lengthscale} \label{sec:parametric}
In this section, we propose a new approach for training parametric MMD models. The guiding principle is to design a parametric update that remains as faithful as possible to nonparametric MMD descent, and leverage the convergence result in \Cref{thm:np_rate}.

Given an unknown data generating distribution $\Qb$ and a parametrized
family of model distributions $\{\Pb_\theta : \theta \in \Theta\}$, where $\Theta\subseteq\R^p$, minimum MMD estimation seeks
$\arg \min _{\theta\in\Theta} \mmd_{\ell}^2(\Pb_\theta, \Qb)$ for some lengthscale $\ell > 0$.
We consider in this paper parametric models $\Pb_\theta=(\Gmap_\theta)_{\#} \rho$, where $\rho\in\calP_2(\calZ)$ is a simple base probability measure on a latent space $\calZ\subseteq\R^q$, such as a Gaussian or uniform distribution, and $\Gmap_\theta: \calZ \rightarrow \R^d$ is a generator parameterized by $\theta \in \Theta$.
For the analysis, we take $\Theta$ to be convex so that line segments between successive parameter iterates remain in the parameter domain.
The latent dimension $q$ need not equal the data dimension $d$.
A natural way to evolve the parametric model $\Pb_\theta$ is to update the parameter $\theta$ such that the distribution moves through the pushforward map. In particular, we consider an iterative scheme given by $\theta_{t+1}=\theta_t-\gamma_t\Delta \theta_t \in \Theta$ for any $t\in\N$; if the unconstrained update leaves $\Theta$, it is projected back to $\Theta$. Hence,
\begin{align}\label{eq:mmd_descent}
    \Pb_{\theta_{t+1}} = (\Gmap_{\theta_{t+1}})_\# \rho = (\Gmap_{\theta_t - \gamma_t \Delta \theta_t})_\# \rho,  \quad \forall t \in\N .
\end{align}
Here, $\gamma_t>0$ is a step size and $\Delta \theta_t \in \R^p$ is a parameter update to be specified.
A standard choice is to take $\Delta \theta_t$ as the gradient of $\frac{1}{2} \mmd_{\ell}^2(\Pb_\theta, \Qb)$ with respect to $\theta$.
This gives,
via the chain rule,
\begin{align}\label{eq:parametric_mmd_update_old}
    \Delta \theta_t^{\mathrm{GD}}  = \nabla_\theta \left[ \frac{1}{2} \mmd_{\ell}^2(\Pb_\theta, \Qb) \right] =  \E_{z\sim\rho}\left[ \bJ_\theta \Gmap_{\theta_t}(z)^\top \nabla f_{\ell, \Pb_{\theta_t}, \Qb}(\Gmap_{\theta_t}(z)) \right],  %
\end{align}
where $\bJ_\theta \Gmap_{\theta_t}(z) \in\R^{d\times p}$ represents the Jacobian matrix of the mapping $\theta\mapsto \Gmap_\theta(z)$ evaluated at $\theta_t$.
We call the resulting sequence $(\Pb_{\theta_{t}}^\gd)_{t\in\N}$ \emph{parametric MMD descent}.
Although this straightforward update scheme has been used in the majority of the literature, it remains unclear whether it can guarantee $\lim_{t\to\infty} \mmd_{\ell}(\Pb_{\theta_t}^\gd, \Qb) = {\min_{\theta\in\Theta}}\mmd_{\ell}(\Pb_{\theta}, \Qb)$.

Motivated by the global convergence of nonparametric MMD descent established in \Cref{thm:np_rate}, we introduce an alternative parameter update $\Delta \theta$, such that the squared MMDs at the next iterate under the parameterid and nonparametric updates, $\mmd_{\ell}^2(\Pb_{t+1},\Qb)$ and $\mmd_{\ell}^2(\Pb_{\theta_{t+1}},\Qb)$, match up to first order in $\gamma_t$.
Informally, suppose that $\Pb_t \approx \Pb_{\theta_t} = (\Gmap_{\theta_t})_{\#}\rho$ at iteration $t$. Under mild regularity assumptions, both squared MMD objectives at the next iteration $t+1$ admit the following Taylor expansions in $\gamma_t$ up to second-order remainders:
\begin{align*}
    &\mmd_{\ell}^2(\Pb_{t+1}, \Qb) = \mmd_{\ell}^2(\Pb_{t}, \Qb) - 2\gamma_t \E_{z\sim\rho}\left[
    \nabla f_{\ell,\Pb_{t},\Qb} \!\left(\Gmap_{\theta_{t}}(z)\right)^\top
    \nabla f_{\ell_t,\Pb_{t},\Qb} \!\left(\Gmap_{\theta_t}(z)\right)
    \right] + \calO(\gamma_t^2) \\
    &\mmd_{\ell}^2(\Pb_{\theta_{t+1}}, \Qb) = \mmd_{\ell}^2(\Pb_{\theta_t}, \Qb) - 2\gamma_t \E_{z\sim\rho}\left[
    \nabla f_{\ell,\Pb_{\theta_t},\Qb} \!\left(\Gmap_{\theta_t}(z)\right)^\top
    \bJ_\theta \Gmap_{\theta_t}(z) \Delta\theta
    \right] + \calO(\gamma_t^2).
\end{align*}
Comparing the first-order term in $\gamma_t$ above, we are motivated to pick the update $\Delta \theta$ such that $\bJ_\theta \Gmap_{\theta_t}(\cdot) \Delta\theta \approx \nabla f_{\ell_t,\Pb_{t},\Qb}(\Gmap_{\theta_t}(\cdot))$. This naturally leads to a \emph{projection} of the vector field $\nabla f_{\ell_t,\Pb_{t},\Qb}(\Gmap_{\theta_t}(\cdot))$ for the nonparametric descent scheme onto a linear subspace, $\mathrm{span}\{ \bJ_\theta \Gmap_{\theta_t}(\cdot) u : u\in \R^p \}\subset L_2^d(\rho)$ (see \Cref{fig:residual_geometry}), which consists of the directions that can be realized by infinitesimal updates of the parametric model at $\theta_t$.
The parameter update realizing this projection is therefore obtained by solving the following least-squares problem:
\begin{align}\label{eq:calJ}
    \calJ_t(\Delta \theta_t) := \E_{z\sim \rho} \left[\left\| \bJ_{\theta} \Gmap_{\theta_t}(z) \Delta \theta_t - \nabla f_{\ell_t, \Pb_{\theta_t},\Qb}(\Gmap_{\theta_t}(z)) \right\|^2 \right] .
\end{align}
Since $\calJ_t(\Delta \theta_t)$ is quadratic in $\Delta \theta_t\in\R^p$, its minimizer admits a closed-form expression
\begin{align}\label{eq:parametric_mmd_update_new}
    \Delta \theta_t^\pgd =  \left\{ \E_{z\sim\rho}\left[ \bJ_\theta \Gmap_{\theta_t}(z)^\top \bJ_\theta \Gmap_{\theta_t}(z)\right] + \lambda_t \Id \right\}^{-1} \E_{z\sim\rho}\left[ \bJ_\theta \Gmap_{\theta_t}(z)^\top \nabla f_{\ell_t, \Pb_{\theta_t}, \Qb}(\Gmap_{\theta_t}(z)) \right] .
\end{align}
where we include a ridge regularization $\lambda_t \geq 0$.
Taking $\lambda_t=0$ corresponds to an exact projection of $\nabla f_{\ell_t,\Pb_{\theta_t},\Qb}\circ \Gmap_{\theta_t}$, with respect to the $L_2(\rho)$ norm, onto the linear subspace $\mathrm{span}\{ \bJ_\theta \Gmap_{\theta}(\cdot) u : u\in \R^p \}$.
However, this would require inverting $\E_{z\sim\rho} \left[ \bJ_\theta \Gmap_{\theta_t}(z)^\top \bJ_\theta \Gmap_{\theta_t}(z)\right]$, which can be numerically unstable when the matrix has small eigenvalues; this instability is further amplified when the expectations with $\rho$ are estimated from finite samples.
Including a small diagonal matrix $\lambda_t \Id$ avoids this issue.

At each iterate $t$, the choice of $\lambda_t$ balances a trade-off between an \emph{approximation error} of $\Delta\theta^\pgd_t$ towards the exact projection; and another \emph{approximation error} from using a second-order Taylor expansion.
See \Cref{rem:tradeoff} for details.
Compared with the standard update $\Delta \theta_t^\gd$ in Eq.~\eqref{eq:parametric_mmd_update_old}, our update $\Delta \theta_t^\pgd$ includes a \emph{preconditioning} matrix. For this reason, we refer to the resulting sequence $(\Pb_{\theta_{t}}^\pgd)_{t\in\N}$ as \emph{preconditioned} parametric MMD descent.

Before we proceed to present an upper bound  that controls the gap between the nonparametric and preconditioned parametric updates after one step, we first introduce the following assumptions.

\begin{ass}[Regularity of $\Gmap_\theta$]\label{ass:regularity_G}
Suppose that $\theta\mapsto \Gmap_\theta(z)$ is twice continuously differentiable. 
For all $\theta\in\Theta$ and for all \(i,j\in\{1,\dots,p\}\), suppose that $\E_{z\sim\rho}[\|\Gmap_\theta(z)\|^2] \leq \mathfrak{T}_1$,
$\E_{z\sim\rho}[\|\partial_{\theta_i} \Gmap_\theta(z)\|^2] \le \mathfrak{T}_2$ and $\E_{z\sim\rho}[\|\partial_{\theta_j}\partial_{\theta_i} \Gmap_\theta(z)\|^2] \le \mathfrak{T}_3$.
\end{ass}
For notational simplicity, we set $\mathfrak{T}_1=\mathfrak{T}_2=\mathfrak{T}_3=1$ in the proofs.
\begin{ass}[Projection residual]\label{ass:residual}
For each \(t\in\mathbb{N}\), suppose that there exists a minimizer
$u_t^\star \in \arg\min_{u\in\mathbb{R}^p} \calJ_t(u)$ such that, for some constants \(0\leq \mathfrak{R}<1\) and \(0<\calU<\infty\) independent of \(t\),
\begin{align*}
    \calJ_t(u_t^\star)
    \leq \mathfrak{R}^2
    \|\nabla f_{\ell_t,\Pb_{\theta_t},\Qb}\|_{L_2^d(\Pb_{\theta_t})}^2,
    \qquad
    \|u_t^\star\| \leq \calU .
\end{align*}
\end{ass}
\noindent
\Cref{ass:regularity_G} serve two purposes. First, the condition $\E_{z\sim\rho}[\|\Gmap_\theta(z)\|^2]\leq 1$ ensures that the parametric model family satisfies $\{\Pb_\theta:\theta\in\Theta\} \subseteq \calP_2(\R^d)$.
Second, the bounds
$\E_{z\sim\rho} [\|\partial_{\theta_i}\Gmap_\theta(z)\|^2] \leq 1$
and $\E_{z\sim\rho}[\|\partial_{\theta_j} \partial_{\theta_i} \Gmap_\theta(z)\|^2\big]\leq 1$ provide the regularity needed to control first- and second-order variations of the model $\Pb_\theta$ with respect to $\theta$, required for the remainder terms in the Taylor expansion to be well controlled.
It is slightly stronger than the integrability condition imposed by \cite{briol2019statistical} for only the well-posedness of the gradient update in Eq.~\eqref{eq:parametric_mmd_update_old}. 
\begin{wrapfigure}{r}{0.30\textwidth}
    \centering
    \vspace{-0.5em}
    \includegraphics[width=0.30\textwidth]{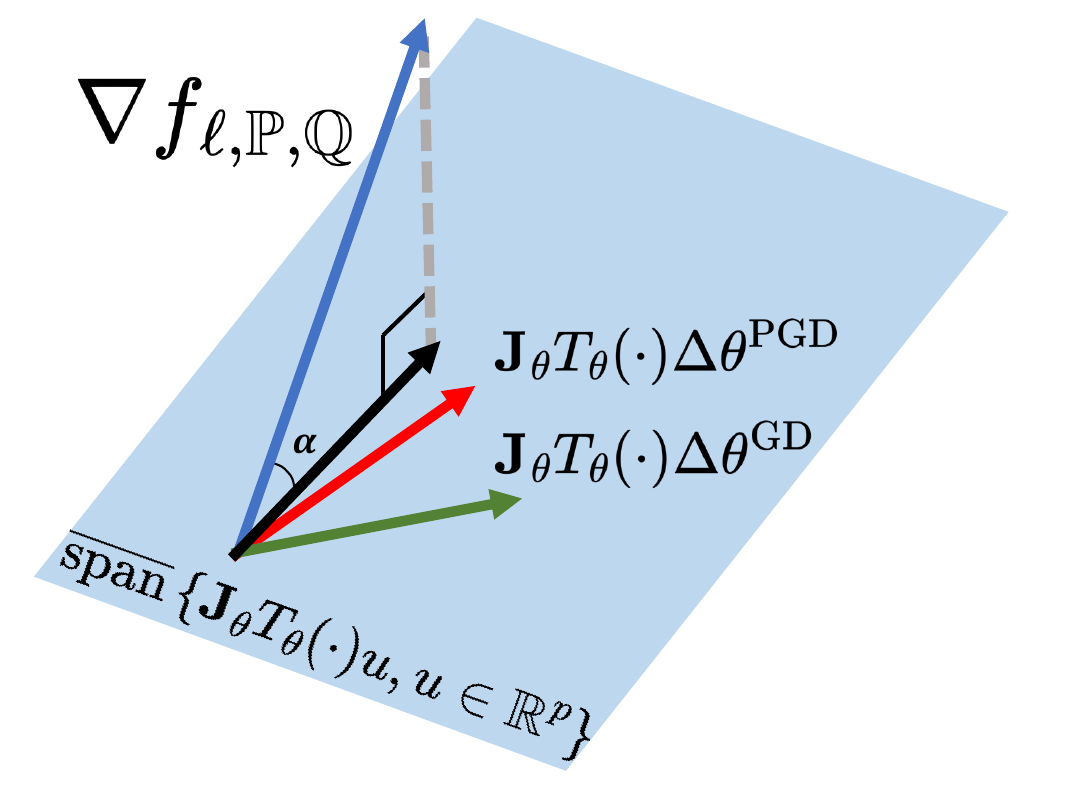}
    \vspace{-20pt}
    \caption{A geometric interpretation of the projection.
    }
    \label{fig:residual_geometry}
\end{wrapfigure}
\Cref{ass:residual} admits a natural geometric interpretation.
For the first inequality of \Cref{ass:residual}, $\mathfrak{R}$ can be interpreted as a uniform-in-$t$ upper bound on $|\sin(\alpha)|$, where $\alpha$ denotes the projection angle shown in  \Cref{fig:residual_geometry}.
Hence the condition $\mathfrak{R}<1$ rules out the degenerate case in which $\nabla f_{\ell_t, \mathbb{P}_{\theta_t}, \mathbb{Q}} \circ \Gmap_{\theta_t}(\cdot)$ is orthogonal to the \emph{tangent subspace} $\operatorname{span}\left\{\mathbf{J}_\theta \Gmap_{\theta_t}(\cdot) u: u \in \mathbb{R}^p\right\}$, the span of all admissible parametric directions.
In that case, the projection is necessarily zero, so the preconditioned descent scheme is not expected to make progress in reducing the MMD. 
The second inequality of \Cref{ass:residual} requires the exact-projection coefficient, displayed by the black vector in \Cref{fig:residual_geometry}, to have a uniform-in-$t$ Euclidean upper bound. 
This is a potentially necessary condition for the projection to be uniformly controlled.
\Cref{ass:regularity_G} can be checked, whereas \Cref{ass:residual} is difficult to verify in practice.

\begin{example}[Sufficient conditions and examples for \Cref{ass:residual}]\label{example:ass}
    A sufficient condition for \Cref{ass:residual} to hold, is that $\nabla f_{\ell, \Pb_\theta, \Qb}(\Gmap_\theta(\cdot )) \in \mathrm{span}\{ \bJ_\theta \Gmap_\theta(\cdot) u, u\in\R^p\}$ with coefficients uniformly-in-$t$ bounded. 
    Then the residual is zero, so the first inequality in \Cref{ass:residual} holds with $\mathfrak{R}=0$, and the second inequality holds from uniformly bounded coefficients. 
    We give a simple example in which the tangent-span condition is satisfied. 
    Let the target distribution $\Qb$ be any rotationally invariant probability measure on $\mathbb{R}^d$, such as an isotropic Gaussian, an isotropic Student-t distribution, or the uniform distribution on a sphere or annulus. Consider the parametric family $\Pb_\theta=(\Gmap_\theta)_\# \operatorname{Unif}(\mathbb{S}^{d-1})$, where $\Gmap_\theta: z \mapsto \theta z$ for $\theta\in\mathbb{R}$, and take $k$ to be a Gaussian kernel. Then the tangent-span condition holds; see \Cref{lem:example} for the proof. 
\end{example}

\begin{rem}[Further interpretation of \Cref{ass:residual}]
     \Cref{ass:residual} implies that whenever parametric descent becomes stationary, i.e., zero-gradient, the induced distribution $\Pb_\theta$ is also stationary for the objective \(\Pb\mapsto \mmd_\ell(\Pb,\Qb)\); see \Cref{lem:residual_excludes_spurious_minima} for a proof. 
     This is a strong condition, but it is needed for our analysis. 
     The adaptive lengthscale scheme is designed to address the nonconvexity of the nonparametric MMD descent, as established in \Cref{thm:np_rate}. 
     However, it does not by itself remove additional nonconvexity introduced by the parametrization \(\theta\mapsto\Pb_\theta\). 
     This limitation is also apparent from the PGD update. 
     If the parametric objective reaches a spurious stationary point satisfying $\nabla_\theta \left[ \frac12 \mmd_\ell^2(\Pb_\theta,\Qb) \right] =0$, then preconditioning cannot create a descent direction, since the preconditioned gradient is zero as well.  
\end{rem}

\begin{prop}[Closeness of $\Pb_{\theta_t^\pgd}$ to $\Pb_t$] \label{prop:delta_theta_ast}
Let kernel $k$ satisfy \Cref{ass:kernel}.
Suppose \Cref{ass:regularity_G,ass:residual} hold.
Let $\Qb\in\calP_2(\R^d)$ and $\rho\in\calP_2(\calZ)$. 
Let both the nonparametric and parametric MMD descent schemes share the same initialization: $\Pb_0 = \Pb_{\theta_0}$.
Let the step size $\gamma_0 > 0$. 
Denote $\Pb_1$ as the next iterate under nonparametric update, and $\Pb^\pgd_{\theta_{1}}$ as the next iterate under our preconditioned parametric update. Then, taking $\lambda_0 \propto (\gamma_0 p^2 (1+\ell^{-1})  \ell_0^{-2})^{\frac{2}{5}}$, we have
\begin{align*}
    \mmd_{\ell}^2 (\Pb_{\theta^\pgd_{1}}, \Qb) -\mmd_{\ell}^2(\Pb_{1}, \Qb) \leq 2\gamma_0 \mathfrak{R} \|\nabla f_{\ell,\Pb_0,\Qb}\|_{L_2^d(\Pb_0)}^2 + C_k\gamma_0 \mathfrak{R} \ell^{-4}(\ell_0^2 - \ell^2) + o(\gamma_0) .
\end{align*}
Here, $C_k$ is a constant that depends only on the kernel profile.
\end{prop}
The proof is given in \Cref{sec:proof_delta_theta}.
\Cref{prop:delta_theta_ast} gives a \emph{one-step} comparison between the PGD update and nonparametric MMD descent, measured by their squared MMD distance to $\Qb$.
The leading term captures the residual component of the nonparametric descent that \emph{cannot} be represented, to first order, by any parametric update on the parameter $\theta$.
The second term comes from the mismatch induced by the adaptive kernel lengthscale: the MMD objective is evaluated at $\ell$, whereas the current iterate uses lengthscale $\ell_0$.
The residual term $o(\gamma_0)$ vanishes faster than linear; see Eq.~\eqref{eq:precise_proposition} for a more precise characterization.
In contrast, the standard gradient descent scheme does not
admit the same first-order comparison with the nonparametric MMD descent 

\begin{rem}[Tradeoff with $\lambda_0$]\label{rem:tradeoff}
A smaller $\lambda_0$ yields a smaller projection-approximation error, of order $\gamma_0 \sqrt{\lambda_0}\,\sqrt{d}\,\ell^{-1}$, relative to the exact projection.
By contrast, a larger $\lambda_0$ reduces the error, of order $\lambda_0^{-2}\gamma_0^2 \sqrt{d}\,\ell^{-2}\ell_0^{-2}$, arising from the second-order Taylor expansion.
The choice of $\lambda_0$ in \Cref{prop:delta_theta_ast} optimally balances these two terms.
A precise characterization is given in Eq.~\eqref{eq:tradeoff}.
\end{rem}

\begin{rem}[Connection to natural gradient]\label{rem:connection_ngd}
Natural gradient descent is another preconditioned gradient method requiring fewer iterations than standard gradient descent to reach convergence~\citep{amari1998natural,hoffman2013stochastic,pascanu2014revisiting,zhang2019fast,martens2020new}.
It can be interpreted as the steepest descent direction when the parameter space is endowed with the local geometry induced by the KL divergence~\citep{amari2000methods}.
For the MMD objective, a natural analogue is to endow the parameter space with the local geometry induced by the squared MMD distance~\citep[Section 2.3]{briol2019statistical}.
This leads to the preconditioning matrix (cf. \Cref{sec:ngd})
\begin{talign}\label{eq:precondition_ngd}
    \left\{ \E_{z,z'\sim\rho}\Big[ \bJ_\theta \Gmap_{\theta}(z)^\top \nabla_{1,2} k_{\ell}(\Gmap_{\theta}(z), \Gmap_{\theta}(z')) \bJ_\theta \Gmap_{\theta}(z') \Big] \right\}^{-1}.
\end{talign}
This matrix is different from our preconditioning matrix $\left\{ \E_{z\sim\rho}[ \bJ_\theta \Gmap_{\theta_t}(z)^\top \bJ_\theta \Gmap_{\theta_t}(z)] \right\}^{-1}$ in Eq.~\eqref{eq:parametric_mmd_update_new}, which arises instead from minimizing, after one step, the discrepancy between parametric and nonparametric descent.
\end{rem}

In practice, the target distribution $\Qb$ is only known with $m$ IID drawn samples. Denote $\widehat{\Qb}_m$ as the corresponding empirical distribution.
We therefore consider the PGD scheme in which the expectation with respect to $\Qb$ is approximated by an empirical average with $\widehat{\Qb}_m$.
Let $(\tildePtheta{t})_{t\in\N}$ be the resulting sequence of parametric distributions, defined by:  $\theta_{t+1} = \theta_t - \gamma_t \tildetheta{t}$, where
\begin{align}\label{eq:parametric_mmd_update_new_tilde}
    \tildetheta{t} =  \left\{ \E_{z\sim\rho}\left[ \bJ_\theta \Gmap_{\theta_t}(z)^\top \bJ_\theta \Gmap_{\theta_t}(z)\right] + \lambda_t \Id \right\}^{-1} \E_{z\sim\rho}\left[ \bJ_\theta \Gmap_{\theta_t}(z)^\top \nabla f_{\ell_t, \Pb_{\theta_t}, \widehat{\Qb}_m}(\Gmap_{\theta_t}(z)) \right] .
\end{align}

\begin{theorem}[Global convergence of PGD scheme]\label{thm:p_rate}
Let kernel $k$ satisfy \Cref{ass:kernel}.
Suppose \Cref{ass:regularity_G,ass:residual} hold.
Suppose $\Qb\in\calP_2(\R^d)$.
Let $C>0$ be a constant independent of $t$ and  $(\gamma_t)_{t\in\N}$ be a monotonically decreasing sequence of step sizes, such that $\gamma_t\in(0,\frac{1}{2C(1-\mathfrak{R})})$, $\lim_{t\to\infty} \gamma_t=0$ and $\sum_{t=0}^\infty \gamma_t=\infty$.
Let $(\ell_t)_{t\in\N}$ be a monotonically decreasing sequence of lengthscales satisfying $\lim_{t\to\infty} \ell_t = \ell > 0$.
Take $\lambda_t \propto (\gamma_t p^2 (1 + \ell^{-1}) \ell_t^{-2})^{\frac{2}{5}}$.
For a fixed horizon $T \in \mathbb{N}$, let $m=m(T) \in \mathbb{N}$ be the empirical target sample size, chosen such that $\sqrt{m(T)} \geq (\sum_{s=0}^T \gamma_s^{0.6}) \cdot \log(T+1)$.
\begin{enumerate}[leftmargin=6mm,labelsep=0.5em]
    \item[\textup{(i)}] \textbf{Well-specified case.} Suppose that, for all \(t\in[T]\), $$\ell_t^d \ell^{-d} \|\nabla f_{\ell, \tildePtheta{t},\Qb} \|_{L_2^d(\tildePtheta{t})}^2 \geq C \mmd_{\ell}^2(\tildePtheta{t}, \Qb).$$ 
    Then, there exists a remainder $r_T$ such that $\E_{m} \left[\mmd^2_{\ell}(\tildePtheta{T}, \Qb)\right] \leq r_T$, and $r_T\rightarrow 0$ as $T\to\infty$. 

    \item[\textup{(ii)}] \textbf{Mis-specified case.} Suppose that, for all \(t\in[T]\), $$\ell_t^d \ell^{-d} \|\nabla f_{\ell, \tildePtheta{t},\Qb} \|_{L_2^d(\tildePtheta{t})}^2 \geq C \Big(\mmd_{\ell}^2(\tildePtheta{t}, \Qb) - \min_{\theta\in\Theta} \mmd_{\ell}^2(\Pb_{\theta}, \Qb) \Big). $$
    Then, there exists a remainder $r_T$ such that $\E_{m} \left[\mmd^2_{\ell}(\tildePtheta{T}, \Qb)\right] - \min_{\theta\in\Theta} \mmd^2_{\ell}(\Pb_{\theta}, \Qb) \leq r_T$, and $r_T\rightarrow 0$ as $T\to\infty$. 
\end{enumerate}
\end{theorem}

The proof is given in \Cref{sec:proof_p_descent}.
The conditions on the step size $\gamma_t$ and lengthscales $\ell_t$ are analogous to the nonparametric MMD descent in \Cref{thm:np_rate}.
As the regularization parameter $\lambda_t$ depends on unknown quantities such as $\calU$, in practice it is selected by a search over a pre-specified candidate set.
The degenerate case $\mathfrak{R}=1$ is likewise manifested in the admissible range of step sizes.
The expectation $\E_m$ is taken over the empirical target sample $\widehat{\Qb}_m$ used in the PGD update. 
The condition on the target sample size $m$ ensures that the statistical error from replacing $\Qb$ by its empirical approximation
\(\widehat{\Qb}_{m}\) remains negligible compared with the descent achieved by the PGD updates.
In the case when $\gamma_t\asymp t^{-1}$, this condition becomes $m(T)\gtrsim T^{4 / 5}(\log T)^2$. 

\Cref{thm:p_rate} gives two convergence guarantees in expectation over the empirical target sample. 
In the well-specified case, the model class is rich enough to represent the target distribution in MMD, and the bound shows that our PGD drives terminal value
$\mmd_{\ell}^2(\tildePtheta{T},\Qb)$ tending to $0$ in expectation. 
In the misspecified case, however, exact recovery of $\Qb$ is not possible within the model class; the appropriate benchmark is therefore the best attainable MMD value,
$\min_{\theta\in\Theta} \mmd_{\ell}^2(\Pb_{\theta},\Qb)$.
The theorem above shows that our PGD drives the terminal PGD objective towards this oracle value under a different gradient-dominance condition;
this condition is harder to verify than its well-specified counterpart, since it involves the unknown best approximation within the model class.

A related result is Proposition 5.2 of \cite{cherief2022finite}, which establishes convergence under \emph{convexity} of the map $\theta \mapsto \mmd^2(\Pb_\theta, \Qb)$, a condition rarely satisfied in practice; see \Cref{fig:illustration}.
By contrast, our result does \emph{not} rely on convexity of the MMD objective. Instead, it uses two key ingredients: 1) adaptive choice of kernel lengthscale, and 2) a close approximation of nonparametric MMD descent.
Our expected gradient-dominance condition is nevertheless compatible with the classical convex setting.
In particular, if $\theta \mapsto \mmd^2(\Pb_\theta,\Qb)$ is $\alpha$-strongly convex and \Cref{ass:regularity_G} holds, then our expected gradient-dominance condition is satisfied with $C=(2p)^{-1}\alpha$; see \Cref{lem:strong_convex}.
To the best of our knowledge, although our result is only \emph{asymptotic} and relies on the stated gradient-dominance and projection-residual assumptions, it is among the first global convergence guarantees for finding minimum MMD estimators without assuming convexity of the MMD objective.
Extension to a non-asymptotic convergence rate, with respect to both $m$ and $t$, would likely require stronger conditions on $\gamma_t$, $\ell_t$, and $\mathfrak{R}$.

Below we outline a quick sketch of the proof of \Cref{thm:p_rate}. 
The key step is to prove the following \emph{descent lemma}: for any $t\in\N$,
\begin{align}\label{eq:descent_lemma_p_main}
    \hspace{-10pt} \E_m\left[ \mmd_{\ell}^2 (\tildePtheta{t+1}, \Qb) - \mmd_{\ell}^2 (\tildePtheta{t}, \Qb) \right] \leq -2(1-\mathfrak{R}) \gamma_t \|\nabla f_{\ell, \tildePtheta{t}, \Qb}\|_{L_2^d(\tildePtheta{t})}^2 + o(\gamma_t)
\end{align}
To establish this, we define two auxiliary distributions: 1) $\widehat{\Pb}_{t+1}:=(\Id - \gamma_t \nabla f_{\ell_t, \tildePtheta{t},\Qb})_\# \tildePtheta{t}$, which corresponds to the next distribution under the \emph{nonparametric} MMD update; and 2)  $\Pb_{\theta_{t+1}}^\pgd:=\Pb_{\theta_t -\gamma_t \Delta \theta_t^\pgd}$ which corresponds to the next distribution under the parametric update scheme with the \emph{true} target distribution $\Qb$ in Eq.~\eqref{eq:parametric_mmd_update_new_tilde}.
The proof of Eq.~\eqref{eq:descent_lemma_p_main} is then a combination of the following inequalities.
\begin{align*}
    \mmd_{\ell}^2 (\widehat{\Pb}_{t+1}, \Qb) -\mmd_{\ell}^2(\tildePtheta{t}, \Qb) &\leq -2\gamma_t \ell_t^d \ell^{-d} \|\nabla f_{\ell, \tildePtheta{t}, \Qb}\|_{L_2^d(\tildePtheta{t})}^2 + o(\gamma_t) \\
    \mmd_{\ell}^2 (\Pb_{\theta_{t+1}}^\pgd, \Qb) -\mmd_{\ell}^2(\widehat{\Pb}_{t+1}, \Qb) &\leq 2\gamma_t \mathfrak{R} \|\nabla f_{\ell, \tildePtheta{t}, \Qb}\|_{L_2^d(\tildePtheta{t})}^2 + o(\gamma_t) \\
    \E_m\left[ \mmd_{\ell}^2 (\tildePtheta{t+1}, \Qb)\right] -\mmd_{\ell}^2(\Pb_{\theta_{t+1}}^\pgd, \Qb) &\leq  C_k \gamma \lambda^{-1} p \ell^{-1}\ell_t^{-1} \sqrt{d} \cdot \frac{1}{\sqrt{m}} .
\end{align*}
The first inequality is proved in \Cref{prop:descent_np} and controls the decrease of squared MMD under the nonparametric MMD descent after one step. The second inequality is proved in \Cref{prop:delta_theta_ast} and controls the difference of squared MMD between two schemes after one step.
The third inequality is proved in \Cref{prop:tilde_no_tilde}, and bounds the finite-sample error from using $\widehat{\Qb}_m$ instead of the unknown true distribution $\Qb$. 
$C_k$ is a constant that depends only on the kernel profile.
Combining the three inequalities yields Eq.~\eqref{eq:descent_lemma_p_main}.
The proof of \Cref{thm:p_rate} is concluded with a discrete Gronwall's lemma.

\paragraph{Computational complexity}
We now turn to how to compute our PGD update $\tildetheta{t}$ in Eq.~\eqref{eq:parametric_mmd_update_new_tilde} in practice.
Given $n$ i.i.d samples $\{z_i\}_{i=1}^n$ from $\rho$, which may either be fixed throughout the algorithm or resampled at each iteration, and $m$ fixed i.i.d samples $\{y_i\}_{i=1}^m$ from $\Qb$, we can approximate $\tildetheta{t}$ with $\widehattheta{t}:= ( \widehat{H_{t}} + \lambda_t \Id)^{-1} \widehat{g_t}$, where
\begin{align*}
    \widehat{H_t} = \frac{1}{n} \sum_{i=1}^n \bJ_\theta \Gmap_{\theta_t}(z_i)^\top \bJ_\theta \Gmap_{\theta_t}(z_i), \quad \widehat{g}_t = \nabla_{\theta_t} \left\{ \frac{\sum_{i,j=1}^n k(\Gmap_{\theta_t}(z_i), \Gmap_{\theta_t}(z_j))}{n^2}  - \frac{2 \sum_{i,j=1}^{n,m} k(\Gmap_{\theta_t}(z_i), y_j)}{nm}  \right\} .
\end{align*}

Relative to Eq.~\eqref{eq:parametric_mmd_update_new_tilde}, this implementation additionally replaces the expectation with respect to $\rho$ by an empirical average over samples $\{z_i\}_{i=1}^n$. 
Here, $\widehat{g}_t$ before taking the derivative is the finite-sample V-statistic estimator of $\mmd_\ell^2(\Pb_\theta,\widehat{\Qb})$. We use the V-statistic rather than the U-statistic because the former is always nonnegative~\citep{gretton2012kernel,huang2023high}.
The derivative in $\widehat g_t$ can be computed by automatic differentiation, e.g., \cite{jax2018github}.
If $\widehat H_t$ is formed explicitly, its cost is $\calO(n \min(d,p) +ndp^2)$, since each Jacobian $\bJ_\theta \Gmap_{\theta_t}(z_i)\in\R^{d\times p}$ must be evaluated and multiplied to form a $p\times p$ matrix. The matrix inversion $(\widehat{H_{t}} + \lambda_t \Id)^{-1}$ incurs an additional $\calO(p^3)$ cost.
The cost of computing the preconditioning matrix is the same as that of Fisher information matrix in natural gradient descent~\citep{martens2020new}, and is cheaper than the preconditioner in Eq.~\eqref{eq:precondition_ngd} of \cite{briol2019statistical}, which involves a double expectation.
In practice, this extra cost is worthwhile as it gives a more faithful update direction that leads to global convergence, as proved in \Cref{thm:p_rate} and verified empirically next in \Cref{sec:experiments}.
The cost can be reduced by computing both $\widehat g_t$ and $\widehat{H_{t}}$ via mini-batches, although this introduces additional biases.
Moreover, when $p$ is large, the cost of explicitly forming $\widehat H_t$ can be avoided by using Jacobian-vector products. Since the experiments in \Cref{sec:experiments} are relatively small-scale, we leave these implementations to future work.

Our convergence result in \Cref{thm:p_rate} accounts only for the finite-sample error arising from replacing $\Qb$ by $\widehat{\Qb}_m$, and does not explicitly analyze the additional approximation error induced by using finitely many samples $\{z_i\}_{i=1}^n$ from $\rho$.
This additional error can be controlled by standard arguments of biased gradient descent~\citep{devolder2014first,ajalloeian2020convergence}, aided by the regularization parameter $\lambda_t$, which ensures stability of the empirical inverse.
We omit these details to streamline the presentation.

\begin{rem}[Population and empirical minimum MMD estimators]
Throughout the paper, our PGD scheme targets the population minimum MMD estimator $\theta^\ast \in \arg\min_{\theta\in\Theta}  \mmd^2(\Pb_\theta,\Qb)$.
This differs from the more common empirical minimum MMD estimator $\theta^\ast_m \in \arg\min_{\theta\in\Theta}  \mmd^2(\Pb_\theta,\widehat{\Qb}_m)$ with $\widehat{\Qb}_m := \frac{1}{m}\sum_{i=1}^m \delta_{y_i}$, where $y_1,\ldots,y_m$ are IID samples from $\Qb$; see, for example,
\cite[Section~2.2]{briol2019statistical} and
\cite[Section~2.2]{cherief2022finite}.
We focus on the population objective, because, for an empirical measure
$\widehat{\Qb}_m$, the well-specified setting is typically unachievable: in general,
$\widehat{\Qb}_m \notin \{\Pb_\theta:\theta\in\Theta\}$.
Consequently, the final estimator obtained from our PGD scheme should be viewed as a numerical approximation to the population estimator $\theta^\ast$ rather than $\theta^\ast_m$.
Nevertheless, both \cite{briol2019statistical,cherief2020mmd} also target the population quantity since $\theta^\ast_m$ approximates $\theta^\ast$ through the empirical measure $\widehat{\Qb}_m$ (see Proposition 1 of \cite{briol2019statistical} and Theorem 3.2 of \cite{cherief2025parametric}), whereas our implementation approximates $\theta^\ast$ through a finite-sample implementation of the PGD scheme.
\end{rem}

\vspace{1mm}
\textbf{A practical guide to choosing $\ell_t$, $\lambda_t$, and $\gamma_t$.}
The lengthscale schedule $(\ell_t)_{t\in\mathbb{N}}$ is the most important hyperparameter in PGD, with the limiting value $\ell$ chosen based on statistical considerations~\citep{briol2019statistical,cherief2020mmd}.
We use an exponentially decaying schedule of the form
$\ell_t=\max\{\ell, 100\ell\cdot \eta^t\}$,
for some $\eta\in[0.99,1)$ and find that it performs very well in our experiments.
For the step size $\gamma_t$, we use a polynomial decay $\gamma_t = \gamma_0 \cdot (1+t)^{-\beta}$ with $\beta\in(0,1)$ such that it satisfies the conditions in \Cref{thm:np_rate,thm:p_rate}.
For the regularization parameter $\lambda_t$, our theory suggests an adaptive schedule, while in the experiments we find fixed choices to be reasonably stable and use them to reduce tuning cost, in a similar spirit to regularization in 
kernel ridge regression~\citep{fischer2020sobolev,shen2025nonparametric}. 
Compared to standard gradient descent, our method introduces two additional parameters: $\ell_t$ and $\lambda_t$. In our experiments, the performance of PGD is relatively insensitive to $\lambda_t$ and requires only modest tuning of the step size $\gamma_t$. 
Additional ablation studies are provided in \Cref{sec:add_experiment}.

\begin{figure}[t]
    \centering
    \includegraphics[width=1.0\linewidth]{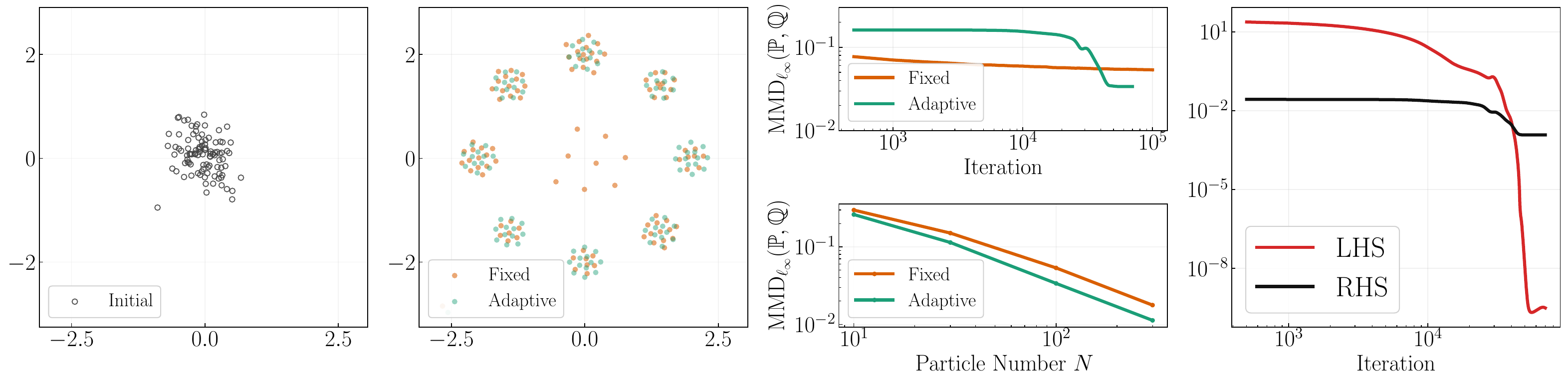}
    \vspace{-20pt}
    \caption{\textit{Mixture of Gaussians.} From left to right: first and final locations of $100$ particles; MMD versus iterations;
    MMD versus particle number $N$; the left- and right-hand sides of the condition in \Cref{thm:np_rate}. Results are over 10 seeds; curves show medians and bands the 25--75\% percentiles. }
    \label{fig:mmd_flow}
    \vspace{-15pt}
\end{figure}

\section{Experiments}\label{sec:experiments}
In this section, we provide empirical evidence to illustrate the effectiveness of our \emph{preconditioned} gradient descent (PGD) scheme. 
For comparison, we also consider standard gradient descent (\textsc{GD}) and \textsc{NGD}~\citep{briol2019statistical}, which uses Eq.~\eqref{eq:precondition_ngd} as the preconditioner. 
Code for reproducing all experiments is available at \url{https://github.com/seulkang0518/MDF_AL}. 

\vspace{1mm} 
\textbf{Nonparametric MMD descent.} 
We first consider a synthetic experiment where the target distribution is a mixture of 8 two-dimensional $(d=2)$ Gaussian distributions \citep{chen2025stationary}. 
This simple multi-modal setting is a good benchmark to evaluate the ability of sampling methods to capture all modes of a distribution. 
This experiment is intended to show that adaptive kernel lengthscales can improve convergence for the nonparametric MMD descent scheme, supporting the qualitative message of \Cref{thm:np_rate}. 
We use a particle implementation of nonparametric MMD descent~\citep{arbel2019maximum,chen2025stationary}. 
Let $\{y_i\}_{i=1}^m$ be $m$ i.i.d.\ samples from $\Qb$, and let $\{x_i^{(0)}\}_{i=1}^N$ be $N$ particles drawn IID from an initial distribution $\calN(0, 1)$. 
Then, at iteration $t\in\N$, the particles are updated according to
\begin{equation*}
    x_i^{(t+1)} = x_i^{(t)} - \gamma_t \nabla f_{\ell_t,\hat{\Pb}_{N,t},\widehat{\Qb}_m}\bigl(x_i^{(t)}\bigr), 
    \qquad \forall i\in[N],
\end{equation*}
where $\widehat{\Qb}_m=\frac{1}{m}\sum_{i=1}^m \delta_{y_i}$ and $\hat{\Pb}_{N,t}=\frac{1}{N} \sum_{i=1}^N \delta_{x_i^{(t)}}$.  
We compare two versions of the method: one with a fixed kernel lengthscale, $\ell_t=\ell=0.1$, and one with an adaptive kernel lengthscale $\ell_t=\max\{\ell, 10.0 \times 0.9999^t\}$. 
The decay factor $0.9999$ was selected by tuning over the prespecified grid $\{0.99,0.995,0.999,0.9995,0.9999\}$.
Both use an adaptive step size $\gamma_t = 0.01 / (1+t)^{0.1}$ which satisfies the step size conditions in \Cref{thm:np_rate}. 
The two plots in the third panel of \Cref{fig:mmd_flow} show that the adaptive scheme attains a smaller final $\mmd_{\ell}(\Pb_t,\Qb)$. 
Interestingly, the advantage of adaptive MMD descent over fixed-lengthscale MMD descent becomes apparent only once $\ell_t$ starts to approach its limiting value $\ell$. 
Although the initial decrease is slower because the algorithm minimizes a different MMD objective, the MMD decreases more rapidly as $\ell_t$ approaches $\ell$, avoiding the local traps observed in the fixed-lengthscale scheme; see the orange particles near the center of the second panel.
In the fourth panel of \Cref{fig:mmd_flow}, we also plot the left- and right-hand sides of the key condition in \Cref{thm:np_rate}, with constant $C=0.01$. 
We observe that this condition is satisfied throughout most of the dynamics until the final stage of the evolution. 
\begin{figure}
    \centering
    \includegraphics[width=1.0\linewidth]{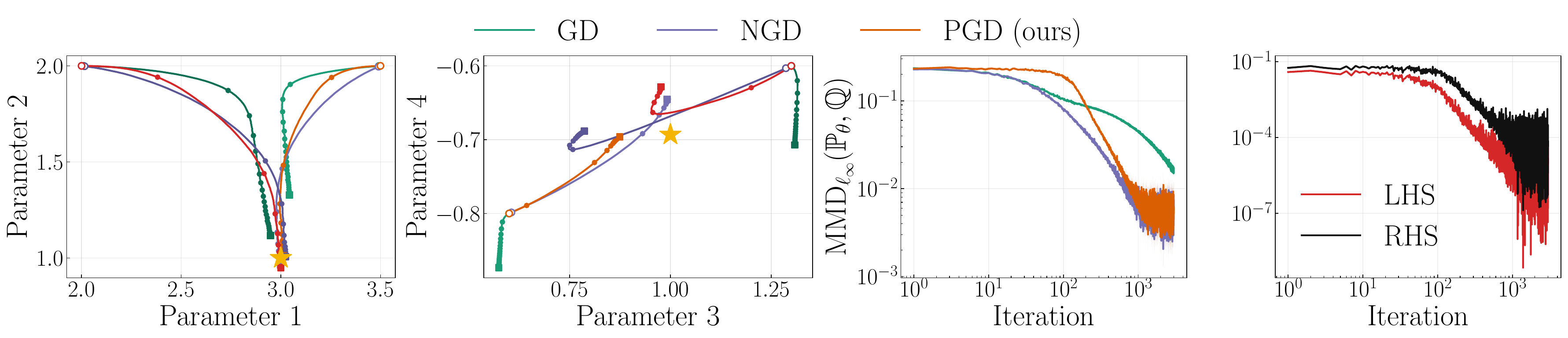}
    \vspace{-19pt}
    \caption{\textit{G-and-k distribution.} 
    From left to right: two optimization trajectories toward the global optimum (yellow star), MMD versus iteration, and left- and right-hand sides of the condition in \Cref{thm:p_rate}. 
    In the left two plots, two trajectories for each method are shown. 
    Results are over 10 seeds; curves show medians and bands the 25--75\% percentiles. } 
    \label{fig:g_and_k}
\end{figure}

\vspace{1mm}
\textbf{G-and-k distribution.} 
We next evaluate our PGD scheme on the g-and-k distribution, a commonly used benchmark for minimum distance estimators~\citep{briol2019statistical,dellaporta2022robust}. 
This distribution has widespread utility, including in the pricing of AirBnB rentals~\citep{rodrigues2020likelihood}, air pollution modelling~\citep{rayner2002numerical} and non-life
insurance~\citep{peters2016estimating}. Despite being a univariate model, it does not admit a closed-form likelihood: specifically, the statistical model is given by $(\Gmap_\theta)_{\#} \rho$, with $\rho$ being a uniform distribution over $[0,1]$, and $\Gmap_\theta$ defined in Eq.~\eqref{eq:G_theta_gandk}. 
The parameter of interest is $\theta=(a,b,c,\exp(k))$, whose four components represent the location, scale, skewness and tail-heaviness of the model distribution.
For this model, $\Gmap_\theta$ and its first- and second-order parameter derivatives have finite second moments for each $\theta$, and \Cref{ass:regularity_G} holds; see \Cref{sec:add_experiment}.

For this experiment, we draw $m=1000$ IID samples from the target distribution $\Qb=(\Gmap_{\theta^\ast})_{\#}\rho$ with $\theta^\ast=(3,1,1,\exp(-\log 2))$ following the setting of \cite{briol2019statistical}. 
This corresponds to a \emph{well-specified} setting. For GD and NGD, we use fixed step sizes $\gamma=0.1$; for our PGD, we use $\gamma_t = 0.1 / (1+t)^{0.1}$ which satisfies the step size condition in \Cref{thm:p_rate}. 
GD and NGD use a fixed lengthscale, $\ell_t=\ell=2$, whereas PGD uses an adaptive lengthscale, $\ell_t=\max\{\ell, 10.0 \times 0.99^t\}$. 
The decay factor $0.99$ was selected by tuning over the prespecified grid $\{0.99,0.995,0.999,0.9995,0.9999\}$. 
As is shown in \Cref{fig:g_and_k}, both PGD and NGD outperform standard GD. In \Cref{fig:g_and_k}, both PGD and NGD perform better than standard GD over the reported optimization horizon. In the first panel, both preconditioned methods reach the optimum in fewer iterations; and in the second panel, standard GD remains far from the global minimum over the optimization horizon, especially in the more challenging third and fourth parameters. 
These results suggest that preconditioning not only accelerates convergence, but also helps overcome the non-convexity of the MMD objective. 
NGD performs slightly better than our PGD scheme, albeit at a larger cost of computing the preconditioner due to the double expectation. 
More empirical studies are thus needed to better understand which preconditioner is preferable in practice. 
A comparison of all methods in terms of computational \emph{time} is given in \Cref{sec:add_experiment}. 

\begin{figure}[t]
\vspace{-10pt}
    \centering
    \includegraphics[width=1.0\linewidth]{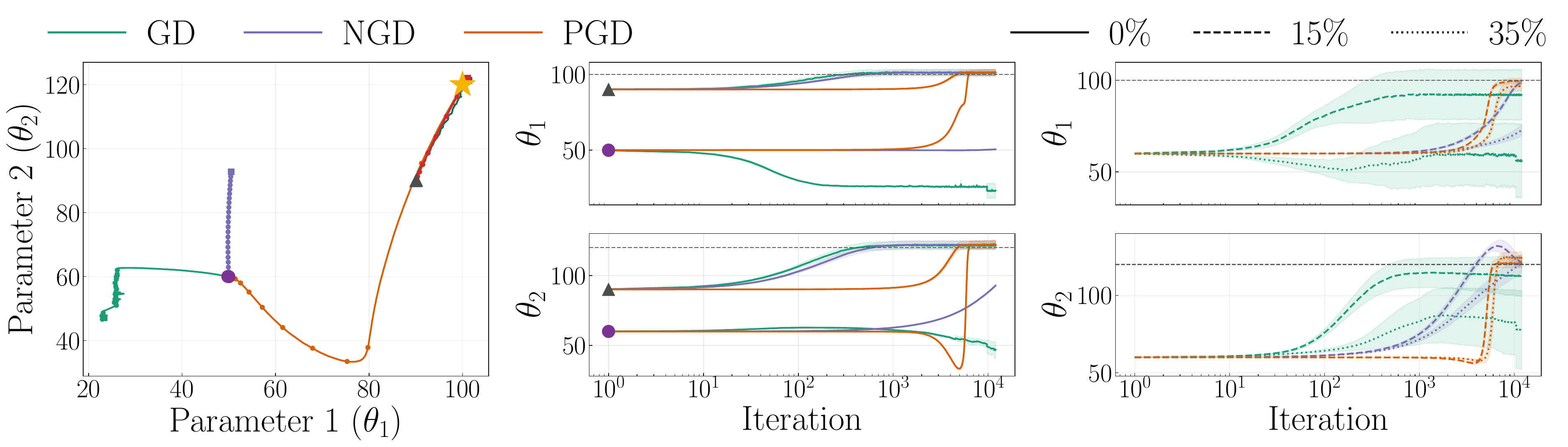}
    \vspace{-15pt}
    \caption{\textit{Lotka-Volterra model.} 
    \textbf{Left:} Two optimization trajectories toward the global optimum (yellow star) under two initializations. 
    \textbf{Middle:} Evolution of $\theta_1, \theta_2$ under different initializations.
    \textbf{Right:} Evolution of $\theta_1, \theta_2$ under different corruption levels. Results are over 10 seeds; curves show medians and bands the 25--75\% percentiles.} 
    \vspace{-10pt}
    \label{fig:lv}
\end{figure}

\vspace{1mm} 
\textbf{Inference for Lotka-Volterra model under corruptions.} 
We next consider an example where the model $\Pb_\theta$ is the output of a stochastic Lotka-Volterra model which describes the evolution of two species $(X_{1, t}, X_{2, t})$ through time via a pair of stochastic differential equations in Eq.~\eqref{eq:lv_sde}~\citep{lotka1927fluctuations,shen2025prediction,chen2026thinned}. 
This is a prototypical model
in population ecology.  
Following \cite{briol2019statistical}, the parameter of interest $\theta=(X_{1,0},X_{2,0})$ corresponds to the initial conditions.
For this simulator, we are not able to verify the uniform first- and second-order derivative bounds in \Cref{ass:regularity_G}; the square-root coefficients in the Euler--Maruyama recursion make such a verification delicate near the boundary of the positive orthant.

We study robustness to data corruption, a setting where minimum MMD estimators are particularly relevant. Specifically, following the setting of \cite{briol2019statistical}, we consider inferring the true initial condition $\theta^\ast=(100,120)$ from observations in which $\mathfrak{p}\%$ of the data is corrupted by samples generated from $\theta^\dagger=(50,50)$. 
For GD and NGD, we use fixed step sizes $\gamma=100$; for our PGD, we use $\gamma_t = 300 / (1+t)^{0.1}$ which satisfies the step size condition in \Cref{thm:p_rate}. 
GD uses a fixed lengthscale, $\ell_t=\ell=30$, whereas PGD uses an adaptive lengthscale, $\ell_t=\max\{30, 3000 \times 0.9995^t\}$. 
The decay factor $0.9995$ was selected by tuning over the prespecified grid $\{0.95, 0.99,0.995,0.999,0.9995,0.9999\}$. 

Our results are shown in \Cref{fig:lv}. In particular, the left and middle panels show that, when initialized at $\theta_0=(90,90)$, close to the global minimizer, all methods converge to the correct solution. In contrast, under a less favorable initialization at $\theta_0=(50,60)$, our PGD still recovers the global minimizer, whereas GD and NGD fail within the given number of iterations. 
The last panel of \Cref{fig:lv} compares the methods across different corruption levels, $\mathfrak{p}\in\{0,15,35\}$. As $\mathfrak{p}$ increases, recovering the ground truth becomes more challenging, since a larger fraction of the data is generated from the corrupted distribution rather than the true distribution. Consequently, GD becomes increasingly unreliable in identifying the ground truth. By contrast, our PGD scheme consistently recovers the global minimizer across all corruption levels. 
These results suggest that PGD exhibits superior convergence behavior, even in mis-specified settings. 
A comparison in terms of computational \emph{time} is given in \Cref{sec:add_experiment}. 

\begin{table}
\vspace{-1.3em}
\centering
\renewcommand{\arraystretch}{1.15}
\begin{tabular}{lcc|cc}
\toprule
Method & Type-I error ($\downarrow$) & Time & Power ($\uparrow$) & Time \\
\midrule
PGD ($\calI=1$) & 0.050 $\pm$ 0.018 & 8.6s & 0.829 $\pm$ 0.032 & 3.6s \\
\midrule
GD ($\calI=400$) & 0.107 $\pm$ 0.026 & 10.4s & 0.600 $\pm$ 0.041 & 4.9s \\
GD ($\calI=1$) & 0.071 $\pm$ 0.022 & 2.6s & 0.279 $\pm$ 0.038 & 1.1s \\
\bottomrule
\end{tabular}
\caption{\textit{Composite goodness-of-fit test.} Critical values are computed by parametric bootstrap with $B=50$ bootstrap samples. The nominal Type-I error is $\alpha=0.05$; lower Type-I error and higher power indicate better performance. All experiments are repeated over $140$ random seeds; results report the mean $\pm$ standard error.}
\label{tab:toggle_switch_gof}
\vspace{-0.5em}
\end{table}

\vspace{1mm} 
\textbf{Composite goodness-of-fit test.}
Finally, we consider a toggle-switch model: a stochastic dynamical model for interacting gene-expression levels over a fixed time horizon $\calT\in\N$~\citep{bonassi2011toggleswitch}. The model describes the coupled temporal evolution of two genes through a dynamical system parametrized by $\theta$. It poses a challenging inference problem, since each parameter evaluation requires simulating a discretized stochastic dynamical system, and the resulting trajectories can be highly sensitive to the parameter values. 
We denote by $\Pb_{\theta;\calT}$ the distribution of the final observations generated by this simulator. 
The explicit dependence on the time horizon $\calT$ is reflected in the notation. 
In this setting, $\Pb_{\theta;\calT} = (\Gmap_{\theta;\calT})_{\#} \rho$ as outlined in \Cref{sec:more_toggle}; and for a fixed $\calT$, the pushforward map $\Gmap_{\theta;\calT}$ satisfies the requirements in \Cref{ass:regularity_G}; see \Cref{sec:more_toggle}.

We evaluate a composite goodness-of-fit test for the model class
$\{\Pb_{\theta;\calT}:\theta\in\Theta\}$, with null hypothesis
$H_0:\Qb\in\{\Pb_{\theta;\calT}:\theta\in\Theta\}$.
Following~\cite{key2025composite}, we use the test statistic
$\Delta_m := m \min_{\theta\in\Theta} \mmd_{\ell}^2(\Pb_{\theta;\calT},\widehat{\Qb}_m)$, 
where $\widehat{\Qb}_m$ denotes the empirical distribution of $m$ IID samples from $\Qb$.
Computing $\Delta_m$ is precisely a minimum-MMD estimation problem over the parameter space.
First, we assess calibration by testing under the null, where
$\Qb=\Pb_{\theta^\ast;50}$ and the model class is $\{\Pb_{\theta;50}:\theta\in\Theta\}$. 
The true parameter $\theta^\ast$ is specified in \Cref{sec:more_toggle}. 
This corresponds to the well-specified setting where $\Qb\in\{\Pb_{\theta;\calT}:\theta\in\Theta\}$ is in the model class. 
Next, we assess testing power under the
alternative, where $\Qb=\Pb_{\theta^\ast;20}$ and the model class is $\{\Pb_{\theta;300}:\theta\in\Theta\}$. 
This corresponds to the mis-specified setting where $\Qb\notin\{\Pb_{\theta;\calT}:\theta\in\Theta\}$ is not in the model class. 
Critical values are computed by parametric bootstrap with
$B=50$ bootstrap samples at a significance level $\alpha=0.05$. 
For GD, we use a fixed step size $\gamma=0.2$ and a fixed lengthscale $\ell=120$; for our PGD, we use an adaptive step size $\gamma_t = 0.2 / (1+t)^{0.1}$ which satisfies the step size condition in \Cref{thm:p_rate}, and an adaptive lengthscale, $\ell_t=\max\{120, 2000 \times 0.95^t\}$. The decay factor $0.95$ was selected by tuning over the prespecified grid $\{0.95, 0.99,0.995,0.999,0.9995,0.9999\}$.

In \Cref{tab:toggle_switch_gof}, we compare PGD and GD in terms of the Type-I errors for testing under the null, and powers for testing under the alternative. 
Following Appendix~D.3.3 of \cite{key2025composite}, GD is run with a multi-start scheme: $\calI=400$ initial parameter values are sampled, and the run attaining the smallest final MMD is retained. 
This multi-start procedure is designed specifically to mitigate the non-convexity of the optimization problem. 
Accordingly, GD with only a single initialization, $\calI=1$, performs poorly, as shown in \Cref{tab:toggle_switch_gof}. 
By contrast, PGD is run from a single initialization, $\calI=1$, and achieves well-calibrated Type-I error and higher power at lower computational cost in this experiment, despite the additional step of preconditioning.
These results provide evidence that our PGD scheme can be efficient for locating the minimum-MMD estimator in this benchmark. 
Since Initializations are run in parallel on GPU rather than sequentially, so increasing $I=1$ to $I=400$ in GD cost only 4x in time rather than 400x.

\vspace{1mm} 
\textbf{Ablation study.}
Our proposed PGD scheme introduces two main components: an adaptive lengthscale and a preconditioner. To assess their respective contributions, we conduct an ablation study in \Cref{fig:cross_method}, comparing GD and PGD under both fixed and adaptive lengthscales, respectively. 
The results suggest that adaptive lengthscale plays a more prominent role in improving convergence for the Lotka--Volterra model, whereas the preconditioner is the more important component for the g-and-k distribution. 
The relative importance of the two components may be problem-dependent; in practice, we therefore recommend using both.

\begin{figure}
    \centering
    \includegraphics[width=0.75\linewidth]{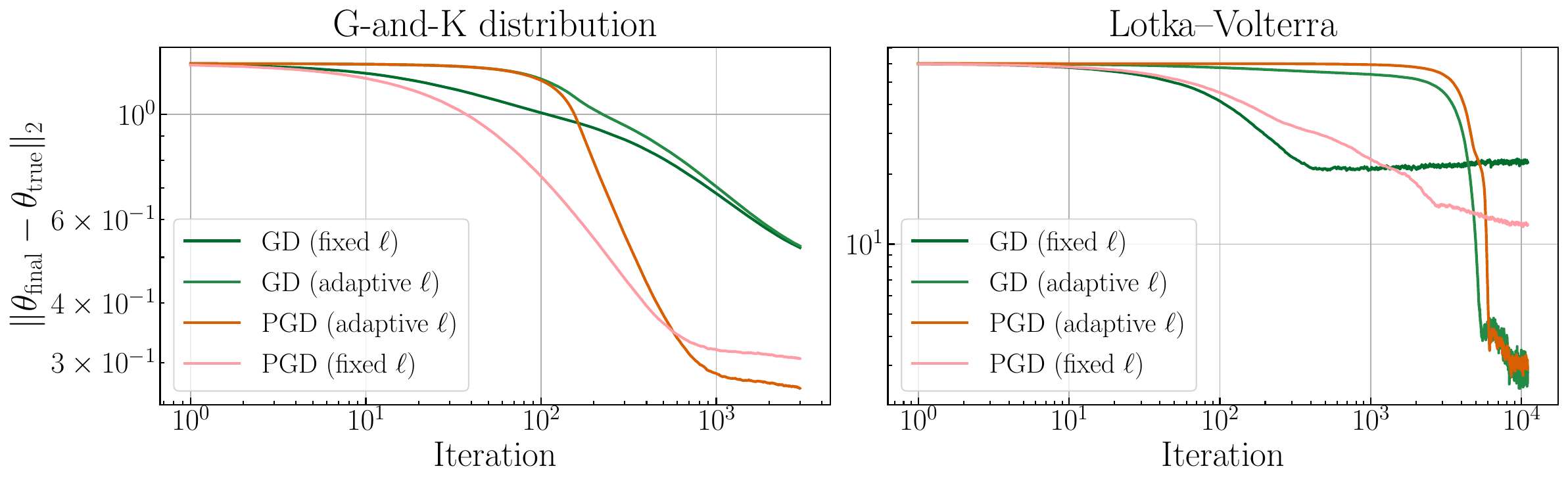}
    \vspace{-10pt}
    \caption{\textit{Ablation study:} Standard GD with fixed/adaptive kernel lengthscales and our PGD with fixed/adaptive kernel lengthscales. \textbf{Left}: G-and-K distribution and \textbf{Right}: Lotka-Volterra model in well-specified setting with $\theta_0=(90,90)$.}
    \label{fig:cross_method}
    \vspace{-10pt}
\end{figure}
\section{Conclusions, Limitations and  Future Work}
We propose a new preconditioned gradient descent scheme for computing minimum MMD estimators, and establish an asymptotic convergence guarantee to a global minimizer under explicit non-convex gradient-dominance and projection-residual conditions.
A key limitation of our current analysis is that it relies on an \emph{a priori} condition on the adaptive choice of kernel lengthscales. Although such conditions are common in the analysis of kernel-based gradient flows~\citep{arbel2019maximum,chen2025stationary}, removing or relaxing this requirement would make our analysis more appealing. 
Another limitation is the projection-residual condition in \Cref{ass:residual}, which we posit is necessary, as it reflects the unavoidable error incurred when approximating the nonparametric update by a parametric update.

Looking ahead, it would be interesting to extend our techniques to the optimization of other popular minimum-distance estimation methods~\citep{basu2011statistical}, such as minimum Stein discrepancy estimation~\citep{barp2019minimum}, minimum Fisher divergence estimation~\citep{hyvarinen2005estimation}, minimum Wasserstein distance estimation~\citep{bassetti2006minimum,zhang2012minimum,lang2026minimum}, and many others. 
A main challenge is that, unlike the MMD, these discrepancies typically do not yield a simple computable finite-sample expression for the associated nonparametric Wasserstein gradient. Minimum kernel Stein discrepancy estimation~\citep{barp2019minimum} is the closest to our minimum MMD estimation; however, it requires access to the model score, whereas in our pushforward setting $\Pb_\theta = (\Gmap_\theta)_\# \rho$ is specified only through simulation. 
It remains unclear how to adapt our projection-based gradient method to this objective.

\bibliographystyle{plainnat}
\bibliography{main}  

\newpage

\begin{appendices}

\crefalias{section}{appendix}
\crefalias{subsection}{appendix}
\crefalias{subsubsection}{appendix}

\makeatletter
\@addtoreset{equation}{section} %
\makeatother
\setcounter{equation}{0}
\renewcommand{\theequation}{\thesection.\arabic{equation}}
\renewcommand{\theHequation}{\thesection.\arabic{equation}} %

\newcommand{\appsection}[1]{
  \refstepcounter{section}
  \section*{Appendix \thesection: #1}
  \addcontentsline{toc}{section}{Appendix \thesection: #1}
}

\onecolumn

\section*{\LARGE\bf \centering Appendices
}

\Cref{sec:proof_np_descent} provides proof for main convergence theorem of the nonparametric descent scheme; \Cref{sec:proof_p_descent} provides proof for main convergence theorem of the parametric descent scheme. 
\Cref{sec:ngd} provides a detailed discussion on the connection of our projected gradient descent (PGD) scheme with natural gradient descent. 
\Cref{sec:add_experiment}
give additional details required to reproduce our experiments. 
\Cref{sec:aux} provides all the auxiliary results for the proof.

\section[Proof of Theorem 3.1]{Proof of \Cref{thm:np_rate}}\label{sec:proof_np_descent}

From \Cref{prop:descent_np} below, we have for any $t\in\N$,
\begin{align*}
    \mmd_{\ell}^2(\Pb_{t+1}, \Qb) -\mmd_{\ell}^2(\Pb_{t}, \Qb) &\leq -2\gamma_t \left( \ell_t \ell^{-1} \right)^d \|\nabla f_{\ell, \Pb_t, \Qb}\|_{L_2^d(\Pb_t)}^2 \\
    &\hspace{-30pt}+ 4 \gamma_t \left( \ell_t \ell^{-1} \right)^d \frac{C_{d,s} \sqrt{\ell_t^2 - \ell^2} }{\ell^3} + {\frac{C_k \gamma_t^2}{\ell^4} + C_k \gamma_t^2\frac{(\ell_t^2-\ell^2)^2}{\ell^8}}.
\end{align*}
Under the assumption that $\left( \ell_t \ell^{-1} \right)^d \|\nabla f_{\ell, \Pb_t,\Qb}\|_{L_2^d(\Pb_t)}^2 \geq C \mmd_{\ell}^2(\Pb_{t}, \Qb) $, we have
\begin{align*}
    \mmd_{\ell}^2 (\Pb_{t+1}, \Qb) &\leq \left( 1 -2 C \gamma_t \right) \mmd_{\ell}^2(\Pb_{t}, \Qb) \\
    &\hspace{-20pt} + 4\gamma_t \left( \ell_t \ell^{-1} \right)^d \frac{C_{d,s} \sqrt{\ell_t^2 - \ell^2} }{\ell^3} + {\frac{C_k \gamma_t^2}{\ell^4} + C_k \gamma_t^2\frac{(\ell_t^2-\ell^2)^2}{\ell^8}}.
\end{align*}
By the discrete Gronwall's lemma~\citep[Appendix~B.2, Gronwall's inequality]{evans2022partial}, iterating the above relation from $t=0, \ldots, T-1$, we obtain
\begin{align}
    \mmd_{\ell}^2(\Pb_T, \Qb) &\leq \exp \left(-2C \sum_{s=0}^{T-1} \gamma_s \right) \mmd_{\ell}^2(\Pb_0, \Qb) \label{eq:term_1} \\
    &\hspace{-80pt} + \sum_{t=0}^{T-1} \exp \left(-2 C \sum_{s=t+1}^{T-1} \gamma_s \right) \left[4 \gamma_t \left( \ell_t \ell^{-1} \right)^d \frac{C_{d,s} \sqrt{\ell_t^2 - \ell^2} }{\ell^3} + {\frac{C_k \gamma_t^2}{\ell^4} + C_k \gamma_t^2\frac{(\ell_t^2-\ell^2)^2}{\ell^8}} \right]\label{eq:term_2} .
\end{align}
The first term in Eq.~\eqref{eq:term_1} converges to $0$ as $T\to\infty$ by the condition that $\lim_{T\to\infty} \sum_{s=0}^{T-1} \gamma_s \to \infty$.
For the second term, the analysis is more involved.
Let $\varepsilon > 0$ be arbitrary small.
Since $\ell_t^2 - \ell^2 \to 0$, $\gamma_t\to 0$ and $\left( \ell_t \ell^{-1} \right)^d \to 1$ as $t\to\infty$ by design, there exists $t_0\in\N$ such that $4 \left( \ell_t \ell^{-1} \right)^d \frac{C_{d,s} \sqrt{\ell_t^2 - \ell^2} }{\ell^3}\leq \frac{\epsilon}{3}$, $\frac{C_k \gamma_t}{\ell^4} \leq \frac{\epsilon}{3}$, $C_k \gamma_t\frac{(\ell_t^2-\ell^2)^2}{\ell^8}\leq\frac{\epsilon}{3}$ all hold for all $t>t_0$.
Hence, we can split the second term in Eq.~\eqref{eq:term_2} into another two terms:
\begin{align*}
    \eqref{eq:term_2} &=
    \sum_{t=0}^{t_0} \exp \left(-2 C \sum_{s=t+1}^{T-1} \gamma_s \right)  \left[ 4 \gamma_t \left( \ell_t \ell^{-1} \right)^d \frac{C_{d,s} \sqrt{\ell_t^2 - \ell^2} }{\ell^3} + \frac{C_k \gamma_t^2}{\ell^4} + C_k \gamma_t^2\frac{(\ell_t^2-\ell^2)^2}{\ell^8} \right] \\
    & + \sum_{t=t_0+1}^{T-1} \exp \left(-2 C \sum_{s=t+1}^{T-1} \gamma_s \right)  \left[4 \gamma_t \left( \ell_t \ell^{-1} \right)^d \frac{C_{d,s} \sqrt{\ell_t^2 - \ell^2} }{\ell^3} + \frac{C_k \gamma_t^2}{\ell^4} + C_k \gamma_t^2\frac{(\ell_t^2-\ell^2)^2}{\ell^8} \right] .
\end{align*}
The first term is a finite sum which converges to $0$ as $T\to\infty$ since each term is multiplied by $\exp \left(-2 C \sum_{s=t+1}^{T-1} \gamma_s \right) \to 0$.
The second term, by the choice of $t_0$, is upper bounded by
$\epsilon \sum_{t=t_0+1}^{T-1} \exp\left(-2 C \sum_{s=t+1}^{T-1} \gamma_s \right)\gamma_t $. Note that, $\gamma_t \leq \frac{1}{C}(1-e^{-2C \gamma_t})$ since $0<\gamma_t\leq\frac{1}{2C}$ by the condition imposed in the theorem, and we have,
\begin{align}\label{eq:epsilon}
    \epsilon \sum_{t=t_0+1}^{T-1} \exp\left(-2 C \sum_{s=t+1}^{T-1} \gamma_s \right)\gamma_t
    &\leq
    \epsilon \sum_{t=t_0+1}^{T-1} \exp\left(-2 C \sum_{s=t+1}^{T-1} \gamma_s \right)\frac{1}{C} \bigl(1-e^{-2C\gamma_t}\bigr) \nonumber \\
    &=
    \frac{\epsilon}{C} \sum_{t=t_0+1}^{T-1}
    \left[
    \exp \left(-2 C \sum_{s=t+1}^{T-1} \gamma_s \right)
    -
    \exp\left(-2 C \sum_{s=t}^{T-1} \gamma_s \right)
    \right] \nonumber \\
    &=
    \frac{\epsilon}{C}\left(1-\exp\left(-2C \sum_{s=t_0+1}^{T-1}\gamma_s\right)\right) \leq \frac{\epsilon}{C} .
\end{align}
The sum is bounded. Since $C$ is a constant independent of $t$ and $\epsilon$ is arbitrary, the second term also converges to zero.
Therefore, we have proved that $\lim_{T\to\infty} \mmd_{\ell}^2(\Pb_T, \Qb) = 0$.

\begin{prop}[Descent Lemma] \label{prop:descent_np}
Let kernel $k$ satisfy \Cref{ass:kernel}. 
Let $t\in\N$ be fixed. 
Suppose $\Qb\in\calP_2(\R^d)$ and let 
$\Pb_{t+1} = (\Id - \gamma \nabla f_{\ell_t, \Pb_t,\Qb})_\# \Pb_t$ be the adaptive MMD descent defined in Eq.~\eqref{eq:adaptive_mmd_descent} with lengthscale $\ell_t>0$ and step size $\gamma >0$.
Let $0<\ell< \ell_t$ be another lengthscale. Then, 
\begin{align*}
    &\quad \mmd_{\ell}^2(\Pb_{t+1}, \Qb) -\mmd_{\ell}^2(\Pb_{t}, \Qb) \leq -2\gamma \left( \ell_t \ell^{-1} \right)^d \|\nabla f_{\ell, \Pb_t, \Qb}\|_{L_2^d(\Pb_t)}^2 \\
    &\hspace{20pt} + 4\gamma \left( \ell_t \ell^{-1} \right)^d \frac{C_{d,s} \sqrt{\ell_t^2 - \ell^2} }{\ell^3} + \frac{C_k \gamma^2}{\ell^4} + C_k \gamma^2\frac{(\ell_t^2-\ell^2)^2}{\ell^8}.
\end{align*}
Here, $C_k$ is a constant depending only on the kernel $k$ and is independent of $\ell$; and $C_{d,s}$ is a constant depending only on the dimension $d$ and kernel smoothness $s$.
\end{prop}
\begin{proof}
In the proof of this proposition, since the time index is always fixed at $t$, we use the shorthand notation that $f_{\ell} := f_{\ell, \Pb_t, \Qb}$.
From \Cref{lem:descent}, we have
\begin{align}\label{eq:descent_lemma_1}
\mmd_{\ell}^2(\Pb_{t+1}, \Qb) -\mmd_{\ell}^2(\Pb_{t}, \Qb) \leq  - 2\gamma \int \nabla f_{\ell}(x)^\top \nabla f_{\ell_{t}}(x) \dd \Pb_t(x) + \frac{C_k\gamma^2}{\ell^2} \|\nabla f_{\ell_{t}}\|^2_{L_2^d(\Pb_t)}.
\end{align}
The first term on the right hand side of Eq.~\eqref{eq:descent_lemma_1} can be further upper bounded by the following:
\begin{align*}
    &\quad -\int \nabla f_{\ell}(x)^\top \nabla f_{\ell_{t}}(x) \dd \Pb_t(x) \\
    &= - \left( \ell_t \ell^{-1} \right)^d \int \nabla f_{\ell}(x)^\top \E_{U\sim\nu} \Big[ \nabla f_{\ell}(x + U)\Big] \dd \Pb_t(x) \\
    &= - \left( \ell_t \ell^{-1} \right)^d \left[ \left\| \nabla f_{\ell} \right\|^2_{L_2^d(\Pb_t)} + \int \nabla f_{\ell}(x)^\top \left\{ \E_{U\sim\nu} \Big[ \nabla f_{\ell}(x + U)\Big] - \nabla f_{\ell}(x) \right\} \dd \Pb_t(x) \right] .
\end{align*}
The second last line holds by \Cref{lem:convolution}, where $\nu$  is a probability measure on $\R^d$ with $\E_{U\sim\nu}[\|U\|]\leq C_{d,s} \sqrt{{\ell_t}^2-{\ell}^2}$ for some constant $C_{d,s}$ depending only on the dimension $d$ and the smoothness parameter $s$. The last equality is just adding and subtracting the same term. 
Next, by the Cauchy-Schwarz inequality, we have
\begin{align*}
    &\leq -\left( \ell_t \ell^{-1} \right)^d\left[\left\| \nabla f_{\ell} \right\|^2_{L_2^d(\Pb_t)} - \int  \left\| \nabla f_{\ell}(x)\right\|\cdot \left\| \E_{U\sim\nu} \Big[ \nabla f_{\ell}(x + U)\Big] - \nabla f_{\ell}(x) \right\| \dd \Pb_t(x)\right] \\
    &\leq -\left( \ell_t \ell^{-1} \right)^d\left[\left\| \nabla f_{\ell} \right\|^2_{L_2^d(\Pb_t)} - \int  \left\| \nabla f_{\ell}(x)\right\|\cdot \Big| \E_{U\sim\nu} \left\| \nabla f_{\ell}(x + U) - \nabla f_{\ell}(x) \right\| \Big| \dd \Pb_t(x)\right] \\
    &\leq - \left( \ell_t \ell^{-1} \right)^d\left[ \left\| \nabla f_{\ell} \right\|^2_{L_2^d(\Pb_t)} - \frac{2 C_k}{\ell^2} \left\| \nabla f_{\ell} \right\|_{L_2^d(\Pb_t)} \cdot \E_{U\sim\nu}\left[ \|U\| \right]\right] \\
    &= - \left( \ell_t \ell^{-1} \right)^d \left[\left\| \nabla f_{\ell} \right\|^2_{L_2^d(\Pb_t)} - \left\| \nabla f_{\ell} \right\|_{L_2^d(\Pb_t)} \frac{C_{d,s}C_k \sqrt{\ell_t^2 - \ell^2} }{\ell^2}\right] .
\end{align*}
The second last inequality holds by the fact that $x\mapsto \nabla f_{\ell}(x)$ is $(2 C_k/\ell^2)$-Lipschitz with a kernel-dependent constant $C_k$ for the Gaussian and the Mat\'ern kernel in \Cref{ass:kernel}. 
The last inequality holds by the first moment bound on $U\sim\nu$. 
Then, for the second term on the right hand side of Eq.~\eqref{eq:descent_lemma_1}, we use the triangular-inequality and \Cref{lem:lengthscale_bound}:
\begin{align*}
    \|\nabla f_{\ell_{t}}\|^2_{L_2^d(\Pb_t)}
    &\leq 2\left\| \nabla f_{\ell} \right\|^2_{L_2^d(\Pb_t)} + 2\left\| \nabla f_{\ell_{t}} - \nabla f_{\ell}
    \right\|^2_{L_2^d(\Pb_t)}
    \leq 2\left\| \nabla f_{\ell} \right\|^2_{L_2^d(\Pb_t)} + C_k \left(\frac{\ell_t^2 - \ell^2}{\ell^3}\right)^2 .
\end{align*}
Here, $C_k$ is a constant depending only on the kernel $k$ and is independent of $\ell$.
Combine the above two upper bounds for Eq.~\eqref{eq:descent_lemma_1}, we obtain
\begin{align*}
     \mmd_{\ell}^2(\Pb_{t+1}, \Qb) -\mmd_{\ell}^2(\Pb_{t}, \Qb)
    &\leq -2\gamma \left( \ell_t \ell^{-1} \right)^d
    \left[
    \|\nabla f_{\ell}\|_{L_2^d(\Pb_t)}^2
    -
    \left\| \nabla f_{\ell} \right\|_{L_2^d(\Pb_t)} \frac{C_{d,s} C_k \sqrt{\ell_t^2 - \ell^2} }{\ell^2}
    \right]\\
    &\quad +
    \frac{C_k\gamma^2}{\ell^2} 
    \left[
    2\|\nabla f_{\ell}\|_{L_2^d(\Pb_t)}^2
    + C_k \left(\frac{\ell_t^2-\ell^2}{\ell^3}\right)^2 
    \right].
\end{align*}
Expanding, this gives,
\begin{align*}
    &\quad \mmd_{\ell}^2(\Pb_{t+1}, \Qb) -\mmd_{\ell}^2(\Pb_{t}, \Qb) \\&\leq -2\gamma \left( \ell_t \ell^{-1} \right)^d \|\nabla f_{\ell}\|_{L_2^d(\Pb_t)}^2 + 2\gamma \left( \ell_t \ell^{-1} \right)^d \|\nabla f_{\ell}\|_{L_2^d(\Pb_t)}\cdot \frac{C_{d,s} C_k \sqrt{\ell_t^2 - \ell^2} }{\ell^2}  \\
    &\qquad + \frac{C_k\gamma^2}{\ell^2}
    \|\nabla f_{\ell} \|_{L_2^d(\Pb_t)}^2 + C_k \gamma^2\frac{(\ell_t^2-\ell^2)^2}{\ell^8} \\
    &\leq -2\gamma \left( \ell_t \ell^{-1} \right)^d \|\nabla f_{\ell}\|_{L_2^d(\Pb_t)}^2 + 4\gamma \left( \ell_t \ell^{-1} \right)^d \frac{C_{d,s} C_k \sqrt{\ell_t^2 - \ell^2} }{\ell^3} + \frac{C_k \gamma^2}{\ell^4} + C_k \gamma^2\frac{(\ell_t^2-\ell^2)^2}{\ell^8},
\end{align*}
where the last step holds by $\|\nabla f_{\ell}\|_{L_2^d(\Pb_t)}\leq C_k \ell^{-1}$ for some constant $C_k$ depending only on the kernel.
\end{proof}

\section[Proof of Theorem 4.6]{Proof of \Cref{thm:p_rate}}\label{sec:proof_p_descent}
\begin{proof}
In the proof, we drop the superscript $\pgd$ for notational simplicity.
Throughout this proof, \Cref{ass:residual} is applied to the empirical trajectory along the PGD scheme by taking $\theta_t$ to be the current empirical iterate $\theta_{t,m}$. 
Let $\widehat{\Pb}_{t+1} = (\Id - \gamma_t \nabla f_{\ell_t, \tildePthetanopgd{t},\Qb})_\# \tildePthetanopgd{t}$ be the next distribution under the nonparametric MMD update scheme.
Let $\Pb_{\theta_{t+1}} = \Pb_{\theta_t -\gamma_t \Delta \theta_{t, m}^\pgd}$ be the next distribution under the parametric update scheme in Eq.~\eqref{eq:parametric_mmd_update_new}.
For comparison purposes, let $\tildePthetanopgd{t+1} = \Pb_{\theta_t -\gamma_t \tildetheta{t}}$ be the next distribution under the parametric update scheme in Eq.~\eqref{eq:parametric_mmd_update_new_tilde}. $\widehat{\Pb}_{t+1}, \Pb_{\theta_{t+1}}$ are auxiliary quantities, while $\tildePthetanopgd{t+1}$ is our next iterate of interest in practice.
First, from the descent lemma proved in \Cref{prop:descent_np}, we have
\begin{align*}
    \mmd_{\ell}^2 (\widehat{\Pb}_{t+1}, \Qb) -\mmd_{\ell}^2(\tildePthetanopgd{t}, \Qb) &\leq -2\gamma_t \left( \ell_t \ell^{-1} \right)^d \|\nabla f_{\ell, \tildePthetanopgd{t}, \Qb}\|_{L_2^d(\tildePthetanopgd{t})}^2 \\
    &\hspace{-30pt} + 4 \gamma_t \left( \ell_t \ell^{-1} \right)^d \frac{C_{d,s}  \sqrt{\ell_t^2 - \ell^2} }{\ell^3} + \frac{C_k \gamma_t^2}{\ell^4} + C_k \gamma_t^2\frac{(\ell_t^2-\ell^2)^2}{\ell^8} .
\end{align*}
Then, from \Cref{prop:delta_theta_ast}, or more precisely from \Cref{prop:delta_theta_est_precise}, we have
\begin{align*}
    &\quad \mmd_{\ell}^2 (\Pb_{\theta_{t+1}}, \Qb) -\mmd_{\ell}^2(\widehat{\Pb}_{t+1}, \Qb) \\
    &\leq 2\gamma_t \mathfrak{R} \|\nabla f_{\ell,\tildePthetanopgd{t},\Qb}\|_{L_2^d(\tildePthetanopgd{t})}^2 + C_k\gamma_t \mathfrak{R} \frac{\ell_t^2 - \ell^2}{\ell^4} + C_k d^{\frac{1}{2}} p^{\frac{2}{5}} \gamma_t^{\frac{6}{5}} \calU^{\frac{4}{5}} \ell^{-1}\left((1+\ell^{-1})\right)^{\frac{1}{5}} \ell_t^{-\frac{2}{5}} + C_k\gamma_t^2 \ell^{-2} \ell_t^{-2} .
\end{align*}
Finally, applying \Cref{prop:tilde_no_tilde} with evaluation lengthscale $\ell$ and update lengthscale $\ell_t$, and taking expectation over $\widehat{\Qb}_m$, gives
\begin{align*}
    \E_m\left[
    \mmd_{\ell}^2 ({\tildePthetanopgd{t+1}}, \Qb) - \mmd_{\ell}^2 (\Pb_{\theta_{t+1}}, \Qb)
    \right]
    \leq C_k \gamma_t \lambda_t^{-1} p \ell^{-1}\ell_t^{-1} \sqrt{d} \cdot \frac{1}{\sqrt{m}}.
\end{align*}
Adding the above three upper bounds and taking expectation over $\widehat{\Qb}_m$, we obtain
\begin{align}
    &\E_m\left[
    \mmd_{\ell}^2 (\tildePthetanopgd{t+1}, \Qb) - \mmd_{\ell}^2 (\tildePthetanopgd{t}, \Qb)
    \right] \nonumber\\
    &
    \leq -2(1-\mathfrak{R}) \gamma_t
    \E_m\left[
    \left( \ell_t \ell^{-1} \right)^d \|\nabla f_{\ell, \tildePthetanopgd{t}, \Qb}\|_{L_2^d(\tildePthetanopgd{t})}^2
    \right] \label{eq:descent_lemma_p} \\
    &\quad + C_k\gamma_t \mathfrak{R} \frac{\ell_t^2 - \ell^2}{\ell^4} + C_k d^{\frac{1}{2}} p^{\frac{2}{5}} \gamma_t^{\frac{6}{5}} \calU^{\frac{4}{5}} \ell^{-1}\left((1+\ell^{-1})\right)^{\frac{1}{5}} \ell_t^{-\frac{2}{5}} + C_k\gamma_t^2 \ell^{-2} \ell_t^{-2} \label{eq:mmd_parametric_descent_proof_residual_1} \\
    &\qquad + 4 \gamma_t \left( \ell_t \ell^{-1} \right)^d \frac{C_{d,s} \sqrt{\ell_t^2 - \ell^2} }{\ell^3} + \frac{C_k \gamma_t^2}{\ell^4} + C_k \gamma_t^2\frac{(\ell_t^2-\ell^2)^2}{\ell^8} \label{eq:mmd_parametric_descent_proof_residual_2}  \\
    &\qquad\qquad + C_k d^{\frac{1}{2}} \gamma_t^{\frac{3}{5}} p^{\frac{1}{5}} \ell^{-1}\left((1+\ell^{-1})\right)^{-\frac{2}{5}} \ell_t^{-\frac{1}{5}} \calU^{\frac{2}{5}} \cdot \frac{1}{\sqrt{m}}
    \label{eq:mmd_parametric_descent_proof_residual_3} .
\end{align}
Denote the sum of Eq.~\eqref{eq:mmd_parametric_descent_proof_residual_1} and Eq.~\eqref{eq:mmd_parametric_descent_proof_residual_2} as $\mathfrak{M}_t$, which is dependent on $t$ but independent of $m$.
Note that the term in Eq.~\eqref{eq:mmd_parametric_descent_proof_residual_3} except from the multiplier $\gamma_t^{\frac{3}{5}} m^{-1/2}$ can be upper bounded as $C_k d^{\frac{1}{2}} p^{\frac{1}{5}} \ell^{-1}\left((1+\ell^{-1})\right)^{-\frac{2}{5}}\ell^{-\frac{1}{5}} \calU^{\frac{2}{5}} =: \mathfrak{N}$, which is independent of either $t$ or $m$.
For the well-specified case, under the pathwise gradient-dominance assumption stated in the theorem,
\begin{align*}
    \E_m\left[\mmd_{\ell}^2 (\tildePthetanopgd{t+1}, \Qb)\right]
    \leq \Big(1 -2C(1-\mathfrak{R}) \gamma_t \Big)
    \E_m\left[\mmd_{\ell}^2 (\tildePthetanopgd{t}, \Qb)\right]
    +  \mathfrak{M}_t + \mathfrak{N} \cdot \gamma_t^{\frac{3}{5}} \frac{1}{\sqrt{m}}.
\end{align*}
The following strategy will be similar to the proof in the nonparametric case in \Cref{sec:proof_np_descent}. 
Fix the optimization horizon $T\in\N$.
By the discrete Gronwall lemma ~\citep[Appendix~B.2, Gronwall's inequality]{evans2022partial}, we iterate this recursion over $t=0,\ldots,T-1$ and obtain
\begin{align}
    \E_m\left[\mmd_{\ell}^2(\tildePthetanopgd{T}, \Qb)\right]
    &\leq \exp \left(-2C(1-\mathfrak{R}) \sum_{s=0}^{T-1} \gamma_s \right) \mmd_{\ell}^2(\Pb_{\theta_0}, \Qb) \label{eq:parametric_term_1} \\
    &\qquad+ \sum_{t=0}^{T-1} \exp \left(-2C (1-\mathfrak{R}) \sum_{s=t+1}^{T-1} \gamma_s \right) \mathfrak{M}_t \label{eq:parametric_term_2} \\
    &\qquad\qquad + \sum_{t=0}^{T-1} \exp \left(-2C (1-\mathfrak{R}) \sum_{s=t+1}^{T-1} \gamma_s \right) \mathfrak{N} \cdot \gamma_t^{\frac{3}{5}} \frac{1}{\sqrt{m}}
    \label{eq:parametric_term_3} .
\end{align}
The first term in Eq.~\eqref{eq:parametric_term_1} converges to $0$ as $T\to\infty$ by the condition that $\lim_{T\to\infty} \sum_{s=0}^{T-1} \gamma_s \to \infty$.
For the second term in Eq.~\eqref{eq:parametric_term_2}, the analysis is more involved.
Let $\varepsilon > 0$ be arbitrary small. By definition of $\mathfrak{M}_t$ as the sum of Eq.~\eqref{eq:mmd_parametric_descent_proof_residual_1} and Eq.~\eqref{eq:mmd_parametric_descent_proof_residual_2}, it satisfies $\lim_{t\to\infty} \mathfrak{M}_t \gamma_t^{-1} = 0$. 
As a result, there exists $t_0\in\N$ such that $\mathfrak{M}_t \gamma_t^{-1} < \varepsilon$ for all $t>t_0$. 
We can split the second term in Eq.~\eqref{eq:parametric_term_2} into another two terms:
\begin{align*}
    \eqref{eq:parametric_term_2} &=
    \sum_{t=0}^{t_0} \exp \left(-2C (1-\mathfrak{R}) \sum_{s=t+1}^{T-1} \gamma_s \right)  \mathfrak{M}_t + \sum_{t=t_0+1}^{T-1} \exp \left(-2C (1-\mathfrak{R}) \sum_{s=t+1}^{T-1} \gamma_s \right) \mathfrak{M}_t.
\end{align*}
The first term in Eq.~\eqref{eq:parametric_term_2} is a finite sum  which converges to $0$ as $T\to\infty$ since each term is multiplied by $\exp (-2 C(1-\mathfrak{R})\sum_{s=t+1}^{T-1} \gamma_s ) \to 0$.
For the second term in Eq.~\eqref{eq:parametric_term_2}, by the choice of $t_0$, is upper bounded by 
$\epsilon \sum_{t=t_0+1}^{T-1} \exp(-2 C (1-\mathfrak{R})\sum_{s=t+1}^{T-1} \gamma_s )\gamma_t$.
Since $0<\gamma_t\leq\frac{1}{2C(1-\mathfrak{R})}$, following the same derivations as in Eq.~\eqref{eq:epsilon}, the second term in Eq.~\eqref{eq:parametric_term_2} is upper bounded by $\epsilon$. As $\epsilon$ is arbitrary, the second term also converges to zero as $T\to\infty$.
Finally, for the third term in Eq.~\eqref{eq:parametric_term_3}, recall that $\sqrt{m(T)} \geq (\sum_{s=0}^T \gamma_s^{0.6})  \cdot \log(T+1)$ assumed in the statement of the Theorem. 
\begin{align*}
    \eqref{eq:parametric_term_3} \leq \sum_{t=0}^{T-1}  \mathfrak{N} \cdot \gamma_t^{\frac{3}{5}} \frac{1}{\sqrt{m(T)}} \leq \frac{\mathfrak{N}}{\log(T+1)} \rightarrow 0, \quad \text{as  } T \to \infty.
\end{align*}
Therefore, we have proved that all the three terms in
Eq.~\eqref{eq:parametric_term_1}, Eq.~\eqref{eq:parametric_term_2} and Eq.~\eqref{eq:parametric_term_3} vanish as $T\to\infty$. {Denote the deterministic upper bound given by the right-hand side of Eq.~\eqref{eq:parametric_term_1}--\eqref{eq:parametric_term_3} as $r_T$.}
Hence,
\begin{align*}
    \E_m\left[\mmd_{\ell}^2(\tildePthetanopgd{T}, \Qb)\right] \leq r_T,
\end{align*}
This proves part \textup{(i)}. 

For the mis-specified case, under the alternative gradient-dominance condition in part \textup{(ii)}, Eq.~\eqref{eq:descent_lemma_p} gives the same recursion for the excess objective:
\begin{align*}
    &\quad \mmd_{\ell}^2 (\tildePthetanopgd{t+1}, \Qb)-\min_{\theta\in\Theta}\mmd_{\ell}^2(\Pb_\theta,\Qb) \\
    & \leq \Big(1 -2C(1-\mathfrak{R}) \gamma_t \Big)
    \left(\mmd_{\ell}^2 (\tildePthetanopgd{t}, \Qb)-\min_{\theta\in\Theta}\mmd_{\ell}^2(\Pb_\theta,\Qb) \right)
    +  \mathfrak{M}_t + \mathfrak{N} \cdot \gamma_t^{\frac{3}{5}} \frac{1}{\sqrt{m}} .
\end{align*}
{The same expected-recursion and discrete Gronwall argument therefore imply that}
\[
\E_m\left[
\mmd_{\ell}^2(\tildePthetanopgd{T},\Qb)-\min_{\theta\in\Theta}\mmd_{\ell}^2(\Pb_\theta,\Qb)
\right] \leq r_T,
\]
{where the corresponding deterministic remainder satisfies $r_T\to0$ as $T\to\infty$ under the same sample-size schedule.} This proves part \textup{(ii)}.
\end{proof}

\subsection[Proof of Proposition 4.3]{Proof of \Cref{prop:delta_theta_ast}}\label{sec:proof_delta_theta}

\begin{prop}[Optimality of $\Delta \theta_t^\pgd$] \label{prop:delta_theta_est_precise}
Suppose the kernel $k$ satisfies \Cref{ass:kernel}. 
Suppose \Cref{ass:regularity_G,ass:residual} hold.
Let both the nonparametric and parametric MMD descent schemes share the same initialization: $\Pb_0 = \Pb_{\theta_0}$.
Let $\gamma_0 > 0$.
Denote $\Pb_1$ as the next iterate under Eq.~\eqref{eq:adaptive_mmd_descent}. Denote $\Pb^\pgd_{\theta_{1}}$ as the next iterate under Eq.~\eqref{eq:mmd_descent} with a preconditioned update $\Delta \theta^\pgd$. 
Then, taking $\lambda_0 \propto (\gamma_0 p^2 (1+\ell^{-1})  \ell_0^{-2}\calU^{-1})^{\frac{2}{5}}$, we have 
\begin{align}\label{eq:precise_proposition}
    &\quad \mmd_{\ell}^2 (\Pb_{\theta^\pgd_{1}}, \Qb) -\mmd_{\ell}^2(\Pb_{1}, \Qb) \nonumber \\
    &\leq 2\gamma_0 \mathfrak{R} \|\nabla f_{\ell,\Pb_0,\Qb}\|_{L_2^d(\Pb_0)}^2 + C_k d^{\frac{1}{2}}\gamma_0 \mathfrak{R} \frac{\ell_0^2 - \ell^2}{\ell^4} + C_k d^{\frac{1}{2}} p^{\frac{2}{5}} \gamma_0^{\frac{6}{5}} \calU^{\frac{4}{5}} \ell^{-1}\left((1+\ell^{-1})\right)^{\frac{1}{5}} \ell_0^{-\frac{2}{5}} + C_k\gamma_0^2 \ell^{-2} \ell_0^{-2} .
\end{align}
Here, $C_k$ depends only on the kernel profile. 
\end{prop}
\begin{proof}
Let $\Pb_0 = \Pb_{\theta_0}=(\Gmap_{\theta_0})_{\#}\rho$.
In the sequel, since the time index is fixed at $t=0$, we drop the subscript in $\lambda_0$.
Then the next iterate produced by the parametric MMD descent scheme and by the MMD gradient descent scheme is given by, respectively,
\begin{align*}
\Pb_{\theta_{1}} = (\Gmap_{\theta_0 - \gamma_0 \Delta \theta})_\# \rho,
\qquad
\Pb_{1}
= (\Id - \gamma_0 \nabla f_{\ell_0, \Pb_0, \Qb})_\# \Pb_0
= (\Id - \gamma_0 \nabla f_{\ell_0, \Pb_0, \Qb})_\# (\Gmap_{\theta_0})_\# \rho .
\end{align*}
Note that here $\Delta \theta$ is a generic update in $\R^p$.
Next, we consider the squared MMD distance between the next parametric MMD descent and the target $\Qb$; as well as the squared MMD distance between the next nonparametric MMD descent and the target $\Qb$.
Specifically,
\begin{align}
\calI_1(\gamma_0) &:= \mmd_{\ell}^2\left(
(\Id-\gamma_0 \nabla f_{\ell_0,\Pb_0,\Qb})_\# \Pb_0,\Qb\right) \nonumber \\
&=
\E_{z\sim\rho,z'\sim \rho}
\Big[
k_{\ell}\left(
\Gmap_{\theta_0}(z)-\gamma_0 \nabla f_{\ell_0,\Pb_0,\Qb}(\Gmap_{\theta_0}(z)),
\Gmap_{\theta_0}(z')-\gamma_0 \nabla f_{\ell_0,\Pb_0,\Qb}(\Gmap_{\theta_0}(z'))
\right)
\Big] \label{eq:I_1_term_1} \\
&\quad
-2 \E_{z\sim \rho, y\sim \Qb}
\Big[
k_{\ell}\left(
\Gmap_{\theta_0}(z)-\gamma_0 \nabla f_{\ell_0,\Pb_0,\Qb}(\Gmap_{\theta_0}(z)),
y
\right)
\Big]
+\E_{y\sim\Qb,y'\sim \Qb}\left[k_{\ell}(y,y')\right] \label{eq:I_1_term_2} .
\end{align}
And
\begin{align}
\calI_2(\gamma_0) &:= \mmd_{\ell}^2\left(
(\Gmap_{\theta_0 - \gamma_0 \Delta \theta})_\# \rho,\Qb\right) \nonumber \\
&= \E_{z,z'\sim \rho}
\Big[
k_{\ell} \left(
\Gmap_{\theta_0 - \gamma_0 \Delta \theta}(z),
\Gmap_{\theta_0 - \gamma_0 \Delta \theta}(z')
\right)
\Big] \label{eq:I_2_term_1} \\
&\quad
-2 \E_{z\sim \rho, y\sim \Qb}
\Big[
k_{\ell}\left(
\Gmap_{\theta_0 - \gamma_0 \Delta \theta}(z),
y
\right)
\Big]
+\E_{y,y'\sim \Qb}\left[k_{\ell}(y,y')\right] \label{eq:I_2_term_2}
\end{align}
For the Gaussian kernel and for the Mat\'ern kernel with smoothness $s>1$, the map $(x,y)\mapsto k_\ell(x,y)$ is twice continuously differentiable with first- and second-order derivatives bounded by constants of order $\ell^{-1}$ and $\ell^{-2}$, respectively. Therefore, Taylor expansions with respect to $\gamma_0$ can be applied to both Eq.~\eqref{eq:I_1_term_1} and Eq.~\eqref{eq:I_1_term_2}.
We obtain
\begin{align*}
    \eqref{eq:I_1_term_1} &= \E_{z, z^{\prime}}\Big[k_{\ell}\left(\Gmap_{\theta_0}(z), \Gmap_{\theta_0}(z^{\prime})\right)\Big] \\
    &\quad - 2\gamma_0 \E_{z, z^{\prime}}\left[\nabla_1 k_{\ell}\left(\Gmap_{\theta_0}(z), \Gmap_{\theta_0}(z^{\prime})\right)^{\top} \nabla f_{\ell_0,\Pb_0, \Qb}\left(\Gmap_{\theta_0}(z)\right)\right] + \frac{\kappa_1}{2} \gamma_0^2 \\
    \eqref{eq:I_1_term_2} &= -2 \E_{z,y}\left[ k_{\ell}\left(\Gmap_{\theta_0}(z),y\right)
\right]
+2\gamma_0 \E_{z,y}\left[
\nabla_1 k_{\ell}\left(\Gmap_{\theta_0}(z),y\right)^\top
\nabla f_{\ell_0,\Pb_0,\Qb}\left(\Gmap_{\theta_0}(z)\right)
\right]
\\&\qquad + \frac{\kappa_2}{2} \gamma_0^2  + \E_{y,y'\sim \Qb}\left[k_{\ell}(y,y')\right] .
\end{align*}
Denote $y_s = \Gmap_{\theta_0}(z)-s\gamma_0 \nabla f_{\ell_0,\Pb_0,\Qb}(\Gmap_{\theta_0}(z)) \in\R^d$ for $s\in[0,1]$, and similarly denote $y_s'$.
Here, $\kappa_1$ and $\kappa_2$ are remainder terms that satisfy the following bounds:
\begin{align*}
    |\kappa_1| &\leq \sup _{s \in[0,1]}\E_{z, z^{\prime}} \left[ \left|\nabla f_{\ell_0,\Pb_0,\Qb}(\Gmap_{\theta_0}(z))^{\top} \nabla_{11}^2 k_{\ell}(y_s, y_s') \nabla f_{\ell_0,\Pb_0,\Qb}(\Gmap_{\theta_0}(z))\right| \right. \\
    &\qquad \left. + 2\left|\nabla f_{\ell_0,\Pb_0,\Qb}(\Gmap_{\theta_0}(z))^{\top} \nabla_{12}^2 k_{\ell}(y_s, y_s') \nabla f_{\ell_0,\Pb_0,\Qb}(\Gmap_{\theta_0}(z'))\right| \right. \\
    &\qquad \left. + \left|\nabla f_{\ell_0,\Pb_0,\Qb}(\Gmap_{\theta_0}(z'))^{\top} \nabla_{22}^2 k_{\ell}(y_s, y_s') \nabla f_{\ell_0,\Pb_0,\Qb}(\Gmap_{\theta_0}(z'))\right| \right], \\
    |\kappa_2| &\leq \sup _{s \in[0,1]}\E_{z, y} \left[ \left|\nabla f_{\ell_0,\Pb_0,\Qb}(\Gmap_{\theta_0}(z))^{\top} \nabla_{11}^2 k_{\ell}(y_s, y) \nabla f_{\ell_0,\Pb_0,\Qb}(\Gmap_{\theta_0}(z))\right| \right], \\
    \max\{\left|\kappa_1\right|, \left|\kappa_2 \right|\}
    &\leq C_k \ell^{-2} \ell_0^{-2} .
\end{align*}
The last inequality follows from $\sup_{y,y'} \| \nabla_{11}^2 k_{\ell}(y,y') \|_{\mathrm{op}} \leq C_k\ell^{-2}$, $\sup_{y,y'} \| \nabla_{12}^2 k_{\ell}(y,y') \|_{\mathrm{op}} \leq C_k\ell^{-2}$, $\sup_{y,y'} \| \nabla_{22}^2 k_{\ell}(y,y') \|_{\mathrm{op}} \leq C_k\ell^{-2}$, and $\sup_{x}\|\nabla f_{\ell_0, \Pb_0, \Qb}(x)\| \leq C_k \ell_0^{-1}$.
Combining the above two equalities on Eq.~\eqref{eq:I_1_term_1} and Eq.~\eqref{eq:I_1_term_2}, we obtain
\begin{align}\label{eq:I_1_taylor}
\calI_1(\gamma_0) = \mmd_{\ell}^2(\Pb_{\theta_0}, \Qb) - 2\gamma_0 \cdot \E_z\left[
\nabla f_{\ell,\Pb_0,\Qb}\left(\Gmap_{\theta_0}(z)\right)^\top
\nabla f_{\ell_0,\Pb_0,\Qb}\left(\Gmap_{\theta_0}(z)\right)
\right] + \frac{\kappa_1+\kappa_2}{2} \gamma_0^2.
\end{align}
Similarly, we can perform the same analysis for $\calI_2(\gamma_0)$. Note that $\theta\mapsto \Gmap_\theta(z)$ is second-order continuously differentiable by \Cref{ass:regularity_G}.
Under Taylor expansion, we have
\begin{align*}
    \eqref{eq:I_2_term_1} &= \E_{z, z^{\prime}}\left[k_{\ell}\left(\Gmap_{\theta_0}(z), \Gmap_{\theta_0}(z^{\prime})\right)\right] - 2\gamma_0 \E_{z, z^{\prime}}\left[\nabla_1 k_{\ell}\left(\Gmap_{\theta_0}(z), \Gmap_{\theta_0}(z^{\prime})\right)^{\top} \bJ_\theta \Gmap_{\theta_0}(z) \Delta\theta \right] + \kappa_3 \frac{\gamma_0^2}{2}  , \\
    \eqref{eq:I_2_term_2} &= -2 \E_{z,y}\left[k_{\ell}\left(\Gmap_{\theta_0}(z),y\right) \right] + 2\gamma_0 \E_{z,y}\left[
    \nabla_1 k_{\ell} \left(\Gmap_{\theta_0}(z),y\right)^\top
    \bJ_\theta \Gmap_{\theta_0}(z) \Delta\theta
    \right] + \kappa_4 \frac{\gamma_0^2}{2} .
\end{align*}
Denote $\theta_s = \theta_0 - s \gamma_0 \Delta \theta \in\Theta$ for $s\in[0,1]$, then
$\kappa_3,\kappa_4\in\R$ above is the remainder term that admits the following bound:
\begin{align*}
    \max\{|\kappa_3|, |\kappa_4|\}
    &\leq \sup_{s\in[0,1]} \E_{z, z^{\prime}} \left[
    \big(\bJ_\theta \Gmap_{\theta_s}(z)\Delta\theta\big)^\top
    \nabla_{11}^2 k_{\ell}\left(\Gmap_{\theta_s}(z),\Gmap_{\theta_s}(z')\right)
    \big(\bJ_\theta \Gmap_{\theta_s}(z)\Delta\theta\big)\right] \\
    &\quad + \sup_{s\in[0,1]} \E_{z, z^{\prime}} \left[
    \big(\bJ_\theta \Gmap_{\theta_s}(z)\Delta\theta\big)^\top
    \nabla_{12}^2 k_{\ell} \left(\Gmap_{\theta_s}(z),\Gmap_{\theta_s}(z')\right)
    \big(\bJ_\theta \Gmap_{\theta_s}(z')\Delta\theta\big)\right] \\
    &\quad + \sup_{s\in[0,1]} \E_{z, z^{\prime}} \left[
    \nabla_1 k_{\ell} \left(\Gmap_{\theta_s}(z),\Gmap_{\theta_s}(z')\right)^\top
    \big(\Delta\theta^\top \nabla_\theta^2 \Gmap_{\theta_s}(z)\Delta\theta\big)\right] \\
    &\leq \sup_{x,x'}\|\nabla_{11}^2 k_{\ell}(x,x')\|_{\mathrm{op}}
    \sup_{s\in[0,1]} \E_{z} \left[\|\bJ_\theta \Gmap_{\theta_s}(z)\|_{\mathrm{op}}^2\right] \cdot \|\Delta\theta\|^2 \\
    &\quad + \sup_{x,x'}\|\nabla_{12}^2 k_{\ell}(x,x')\|_{\mathrm{op}}
    \sup_{s\in[0,1]} \left[ \E_{z}\|\bJ_\theta \Gmap_{\theta_s}(z)\|_{\mathrm{op}} \right]^2 \cdot \|\Delta\theta\|^2 \\
    &\quad + \sup_{x,x'}\|\nabla_1 k_{\ell}(x,x')\|
    \sup_{s\in[0,1]} \E_{z} \left[\|\nabla_\theta^2 \Gmap_{\theta_s}(z)\|_{\mathrm{op}} \right]\cdot \|\Delta\theta\|^2 \\
    &\leq C_k p\sqrt{d}(\ell^{-2}+\ell^{-1}) \|\Delta\theta\|^2.
\end{align*}
The second last step holds by \Cref{ass:regularity_G} along with the bounds that
\begin{align*}
    &\sup_s \E_{z} \left[\|\bJ_\theta \Gmap_{\theta_s}(z)\|_{\mathrm{op}}^2\right] \leq \sup_{\theta\in\Theta} \E_{z\sim\rho} \left[\|\bJ_\theta \Gmap_{\theta}(z)\|_{\mathrm{op}}^2\right] \leq \sup_{\theta\in\Theta} \sum_{j=1}^p \E_{z\sim\rho}\left[ \|\partial_{\theta_j} \Gmap_\theta(z)\|^2 \right] \leq p \\
    &\sup_s \E_{z} \left[ \|\nabla_\theta^2 \Gmap_{\theta_s}(z)\|_{\mathrm{op}} \right] \leq \sup_{\theta\in\Theta} \sqrt{\E_{z} \left[ \|\nabla_\theta^2 \Gmap_{\theta}(z)\|^2_{\mathrm{op}} \right]} \leq \sup_{\theta\in\Theta} \sqrt{\sum_{i,j=1}^p \E_{z\sim\rho}\left[ \|\partial_{\theta_j}\partial_{\theta_i} \Gmap_\theta(z)\|^2 \right]} \leq p .
\end{align*}
and by that $\sup_{x,x'} \| \nabla_{1} k_{\ell}(x,x') \|\leq C_k\sqrt{d} \ell^{-1}$, $\sup_{x,x'} \| \nabla_{1,2}^2 k_{\ell}(x,x') \|_{\mathrm{op}} \leq C_k\sqrt{d} \ell^{-2}$.
Here, $\nabla_\theta^2 \Gmap_{\theta_s}$ shall be interpreted as a bilinear mapping
\begin{align*}
    \nabla_\theta^2 \Gmap_\theta(z)[u, v]:=\left(u^{\top} \nabla_\theta^2 [\Gmap_{\theta}(z)]_1 v, \ldots, u^{\top} \nabla_\theta^2 [\Gmap_{\theta}(z)]_d v\right)\in\R^d, \quad \forall u,v\in\R^p ,
\end{align*}
and its operator norm is defined as $\left\|\nabla_\theta^2 \Gmap_\theta(z)\right\|_{\mathrm{op}}:=\sup _{\|u\|=1,\|v\|=1}\left\|\nabla_\theta^2 \Gmap_\theta(z)[u, v]\right\|$.

Combining the above two equalities on Eq.~\eqref{eq:I_2_term_1} and Eq.~\eqref{eq:I_2_term_2}, we obtain
\begin{align}\label{eq:I_2_taylor}
\calI_2(\gamma_0) = \mmd_{\ell}^2(\Pb_{\theta_0}, \Qb) - 2\gamma_0 \cdot \E_z\left[
\nabla f_{\ell,\Pb_0,\Qb}\left(\Gmap_{\theta_0}(z)\right)^\top
\bJ_\theta \Gmap_{\theta_0}(z) \Delta\theta
\right] + \frac{\kappa_3+\kappa_4}{2} \gamma_0^2 .
\end{align}
Now, from the Taylor expansion of $\calI_1(\gamma_0)$ in Eq.~\eqref{eq:I_1_taylor} and $\calI_2(\gamma_0)$ in Eq.~\eqref{eq:I_2_taylor}, we are going to study the difference of the following two squared MMD distances,
\begin{align}
&\quad \left| \mmd_{\ell}^2\left(
(\Id-\gamma_0 \nabla f_{\ell_0,\Pb_0,\Qb})_\# \Pb_0,\Qb\right) - \mmd_{\ell}^2\left(
(\Gmap_{\theta_0 - \gamma_0 \Delta \theta})_\# \rho,\Qb\right) \right| := | \calI_1(\gamma_0) - \calI_2(\gamma_0)|\nonumber \\
&= 2\gamma_0 \left| \E_z\left[
\nabla f_{\ell,\Pb_0,\Qb} \left(\Gmap_{\theta_0}(z)\right)^\top
\Big( \nabla f_{\ell_0,\Pb_0,\Qb} \left(\Gmap_{\theta_0}(z)\right)
- \bJ_\theta \Gmap_{\theta_0}(z) \Delta\theta \Big) \right] \right| \nonumber \\
&\qquad + \left| \frac{\kappa_1+\kappa_2}{2} \gamma_0^2 -\frac{\kappa_3+\kappa_4}{2} \gamma_0^2  \right| . \nonumber \\
&\leq 2\gamma_0 \|\nabla f_{\ell,\Pb_0,\Qb}\|_{L_2^d(\Pb_0)} \cdot \sqrt{\underbrace{\E_z \left[ \left\| \nabla f_{\ell_0,\Pb_0,\Qb} \left(\Gmap_{\theta_0}(z)\right)
- \bJ_\theta \Gmap_{\theta_0}(z) \Delta\theta \right\|^2 \right] }_{\calJ(\Delta \theta)} } \nonumber \\
&\qquad + C_k\gamma_0^2 \left( p\sqrt{d}(\ell^{-2}+\ell^{-1}) \|\Delta\theta\|^2 + \ell^{-2}\ell_0^{-2} \right) \label{eq:diff_I_1_I_2} .
\end{align}
The last inequality holds by applying Cauchy-Schwarz for the first term, and by plugging in the bounds on $|\kappa_1|, |\kappa_2|, |\kappa_3|, |\kappa_4|$ derived above.
Recall from \Cref{ass:residual} that $\min_{u\in\R^p}\calJ(u)
\le \mathfrak{R}^2 \|\nabla f_{\ell_0, \Pb_{\theta_0}, \Qb}\|_{L_2^d(\Pb_{\theta_0})}^2 = \mathfrak{R}^2 \|\nabla f_{\ell_0, \Pb_0, \Qb}\|_{L_2^d(\Pb_0)}^2$,
and also that $u_0^\ast:= \arg\min \calJ(u)$ and $\|u_0^\ast\|\leq\calU<\infty$.
These conditions imply that there exists an $r\in L_2^d(\rho)$ such that $\nabla f_{\ell_0,\Pb_0,\Qb}(\Gmap_{\theta_0}( \cdot )) =\bJ_\theta \Gmap_{\theta_0}(\cdot) u_0^\ast + r(\cdot)$ with $\|r\|_{L_2^d(\rho)}^2 \leq \mathfrak{R}^2 \, \|\nabla f_{\ell_0, \Pb_0, \Qb}\|_{L_2^d(\Pb_0)}^2$.
Recall $\Delta \theta^\pgd:=\Delta \theta_0^\pgd$ defined in Eq.~\eqref{eq:parametric_mmd_update_new},
\begin{align*}
    \Delta \theta^\pgd =  \left\{ \E_{z\sim\rho}\left[ \bJ_\theta \Gmap_{\theta_0}(z)^\top \bJ_\theta \Gmap_{\theta_0}(z)\right] + \lambda\Id \right\}^{-1} \E_{z\sim\rho}\left[ \bJ_\theta \Gmap_{\theta_0}(z)^\top \nabla f_{\ell_0, \Pb_{\theta_0}, \Qb}(\Gmap_{\theta_0}(z)) \right] .
\end{align*}
From \Cref{lem:source_condition}, we have that
\begin{align*}
    \calJ(\Delta \theta^\pgd) &= \E_{z\sim\rho} \left[ \left\|\nabla f_{\ell_0, \Pb_0, \Qb}(\Gmap_{\theta_0}(z)) - \bJ_\theta \Gmap_{\theta_0}(z) \Delta \theta^\pgd \right\|^2 \right] \\
    &\leq \lambda \|u_0^\ast\|^2 + \mathfrak{R}^2 \, \|\nabla f_{\ell_0, \Pb_0, \Qb}\|_{L_2^d(\Pb_0)}^2 \\
    &\leq \lambda \calU^2 + \mathfrak{R}^2 \, \|\nabla f_{\ell_0, \Pb_0, \Qb}\|_{L_2^d(\Pb_0)}^2 .
\end{align*}
We also have that,
\begin{align*}
    \| \Delta \theta^\pgd \|^2 \leq \frac{1}{\lambda^2} \left\| \E_{z\sim\rho}\left[ \bJ_\theta \Gmap_{\theta_0}(z)^\top \nabla f_{\ell_0, \Pb_0, \Qb}(\Gmap_{\theta_0}(z)) \right]\right\|^2 \leq \frac{C_kp}{\lambda^2} \ell_0^{-2} .
\end{align*}
Denote $\Pb_{\theta^\pgd_{1}}:=(\Gmap_{\theta_0 - \gamma_0 \Delta \theta^\pgd})_\# \rho$ as the next iterate under the update $\Delta \theta^\pgd$.
Therefore, continuing from Eq.~\eqref{eq:diff_I_1_I_2}, we obtain
\begin{align}
    &\quad \mmd_{\ell}^2 (\Pb_{\theta^\pgd_{1}}, \Qb) -\mmd_{\ell}^2(\Pb_{1}, \Qb) \label{eq:tradeoff} \\
    &\leq 2\gamma_0 \|\nabla f_{\ell,\Pb_0,\Qb}\|_{L_2^d(\Pb_0)} \cdot \sqrt{\left( \lambda \calU^2 + \mathfrak{R}^2 \, \|\nabla f_{\ell_0, \Pb_0, \Qb}\|_{L_2^d(\Pb_0)}^2 \right)} \nonumber \\
    &\qquad + C_k\gamma_0^2 \left( p\sqrt{d}(\ell^{-2}+\ell^{-1}) \|\Delta\theta^\pgd \|^2 + \ell^{-2}\ell_0^{-2} \right) . \nonumber \\
    &\leq 2\gamma_0 \mathfrak{R} \|\nabla f_{\ell,\Pb_0,\Qb}\|_{L_2^d(\Pb_0)} \cdot \|\nabla f_{\ell_0, \Pb_0, \Qb}\|_{L_2^d(\Pb_0)} + C_k\gamma_0 \sqrt{\lambda} \sqrt{d} \ell^{-1} \calU \tag{Approximation error} \\
    &\qquad +  C_k\gamma_0^2 \left( p \sqrt{d}(\ell^{-2}+\ell^{-1}) \frac{p}{\lambda^2} \ell_0^{-2}  + \ell^{-2}\ell_0^{-2} \right) \tag{Taylor-expansion error}.
\end{align}
The last step holds by $\sqrt{a+b}\leq\sqrt{a}+\sqrt{b}$, $\sup_z\|\nabla f_{\ell,\Pb_0,\Qb}(z)\|\leq C_k\sqrt{d} \ell^{-1}$, and by plugging in the upper bound of $\|\Delta\theta^\pgd\|^2$.
One can observe a tradeoff on $\lambda$ in the above formula; see \Cref{rem:tradeoff} for details. To balance the two terms regarding $\lambda$ above, we pick $\lambda \asymp \left(\gamma_0 p^2 (1+\ell^{-1}) \ell_0^{-2}\calU^{-1}\right)^{\frac{2}{5}}$. 
This results in
\begin{align*}
    &\quad \mmd_{\ell}^2 (\Pb_{\theta^\pgd_{1}}, \Qb) -\mmd_{\ell}^2(\Pb_{1}, \Qb) \\
    &\leq 2\gamma_0 \mathfrak{R} \|\nabla f_{\ell,\Pb_0,\Qb}\|_{L_2^d(\Pb_0)} \cdot \|\nabla f_{\ell_0, \Pb_0, \Qb}\|_{L_2^d(\Pb_0)} + C_k d^{\frac{1}{2}} p^{\frac{2}{5}} \gamma_0^{\frac{6}{5}} \calU^{\frac{4}{5}} \ell^{-1}\left((1+\ell^{-1})\right)^{\frac{1}{5}} \ell_0^{-\frac{2}{5}} + C_k\gamma_0^2 \ell^{-2} \ell_0^{-2} .
\end{align*}
Then, by \Cref{lem:lengthscale_bound}, $\left\| \nabla f_{\ell_0, \Pb_0,\Qb}(x) - \nabla f_{\ell, \Pb_0,\Qb}(x) \right\| \leq C_k \frac{\ell_0^2 - \ell^2}{\ell^3}$. We obtain
\begin{align*}
    &\quad \mmd_{\ell}^2 (\Pb_{\theta^\pgd_{1}}, \Qb) -\mmd_{\ell}^2(\Pb_{1}, \Qb) \leq 2\gamma_0 \mathfrak{R} \|\nabla f_{\ell,\Pb_0,\Qb}\|_{L_2^d(\Pb_0)}^2\\
    & \quad + C_k\gamma_0 \mathfrak{R} \|\nabla f_{\ell,\Pb_0,\Qb}\|_{L_2^d(\Pb_0)} \cdot \frac{\ell_0^2 - \ell^2}{\ell^3} + C_k d^{\frac{1}{2}} p^{\frac{2}{5}} \gamma_0^{\frac{6}{5}} \calU^{\frac{4}{5}} \ell^{-1}\left((1+\ell^{-1})\right)^{\frac{1}{5}} \ell_0^{-\frac{2}{5}} + C_k\gamma_0^2 \ell^{-2} \ell_0^{-2} \\
    &\leq 2\gamma_0 \mathfrak{R} \|\nabla f_{\ell,\Pb_0,\Qb}\|_{L_2^d(\Pb_0)}^2 + C_k d^{\frac{1}{2}}\gamma_0 \mathfrak{R} \frac{\ell_0^2 - \ell^2}{\ell^4} + C_k d^{\frac{1}{2}} p^{\frac{2}{5}} \gamma_0^{\frac{6}{5}} \calU^{\frac{4}{5}} \ell^{-1}\left((1+\ell^{-1})\right)^{\frac{1}{5}} \ell_0^{-\frac{2}{5}} + C_k\gamma_0^2 \ell^{-2} \ell_0^{-2}.
\end{align*}
This concludes the proof.
\end{proof}

\begin{lem}[Residual error] \label{lem:source_condition}
Let $\theta\in\Theta$ and $\ell>0$ be fixed.
Suppose \Cref{ass:regularity_G} holds.
Suppose there exist $u^\dagger\in\R^p$ and $r\in L_2^d(\rho)$ such that
\[
\nabla f_{\ell,\Pb,\Qb}(\Gmap_\theta(z))
=
\bJ_\theta \Gmap_\theta(z)u^\dagger + r(z),
\qquad \rho\text{-a.e. } z\in\calZ. 
\]
Then, for $\Delta\theta^\pgd$ defined in Eq.~\eqref{eq:parametric_mmd_update_new},
\[
\Delta \theta^\pgd =  \left\{ \E_{z\sim\rho}\left[ \bJ_\theta \Gmap_{\theta}(z)^\top \bJ_\theta \Gmap_{\theta}(z)\right] + \lambda\Id \right\}^{-1} \E_{z\sim\rho}\left[ \bJ_\theta \Gmap_{\theta}(z)^\top \nabla f_{\ell, \Pb, \Qb}(\Gmap_{\theta}(z)) \right],
\]
there holds
\[
    \E_{z\sim\rho} \left[ \left\|\nabla f_{\ell, \Pb, \Qb}(\Gmap_{\theta}(z)) - \bJ_\theta \Gmap_{\theta}(z) \Delta \theta^\pgd \right\|^2 \right]
    \leq
    \E_{z\sim\rho}\big[\|r(z)\|^2\big] + \lambda\|u^\dagger\|^2.
\]
\end{lem}

\begin{proof}
Define an operator $A:\R^p\to L_2^d(\rho)$ by $(Au)(z):=\bJ_\theta \Gmap_\theta(z)u$.
Its adjoint operator $A^\ast:L_2^d(\rho)\to\R^p$ is given by
$A^\ast h=\E_{z\sim\rho}\big[\bJ_\theta \Gmap_\theta(z)^\top h(z)\big]$.
Under \Cref{ass:regularity_G}, both $A$ and $A^\ast$ are well-defined and bounded.
Then, $A^\ast A = \E_{z\sim\rho}\left[ \bJ_\theta \Gmap_{\theta}(z)^\top \bJ_\theta \Gmap_{\theta}(z)\right]\in\R^{p\times p}$ is a matrix.

In the following proof, write $f:=f_{\ell,\Pb,\Qb}$ for brevity, and define
$g:=\nabla f\circ \Gmap_\theta \in L_2^d(\rho)$.
Then, $g=Au^\dagger+r$ by assumption in the lemma. Note that $\Delta\theta^\pgd=(A^\ast A+\lambda I)^{-1}A^\ast g$.
Now observe that $\Delta\theta^\pgd$ is the unique minimizer of the functional
\[
u\mapsto \|g-Au\|_{L_2^d(\rho)}^2+\lambda\|u\|^2.
\]
Therefore, by optimality of $\Delta\theta^\pgd$, taking $u=u^\dagger$ yields
\[
\|g-A\Delta\theta^\pgd\|_{L_2^d(\rho)}^2+\lambda\|\Delta\theta^\pgd\|^2
\le
\|g-Au^\dagger\|_{L_2^d(\rho)}^2+\lambda\|u^\dagger\|^2.
\]
Since $g-Au^\dagger=r$, it follows that
\[
\|g-A\Delta\theta^\pgd\|_{L_2^d(\rho)}^2+\lambda\|\Delta\theta^\pgd\|^2
\le
\|r\|_{L_2^d(\rho)}^2+\lambda\|u^\dagger\|^2.
\]
Dropping the nonnegative term $\lambda\|\Delta\theta^\pgd\|^2$, we obtain
\[
\|g-A\Delta\theta^\pgd\|_{L_2^d(\rho)}^2
\le
\|r\|_{L_2^d(\rho)}^2+\lambda\|u^\dagger\|^2.
\]
Rewriting this in the original notation gives
\[
\E_{z\sim\rho} \left[ \left\|\nabla f_{\ell, \Pb, \Qb}(\Gmap_{\theta}(z)) - \bJ_\theta \Gmap_{\theta}(z) \Delta \theta^\pgd \right\|^2 \right]
\le
\E_{z\sim\rho}\big[\|r(z)\|^2\big]+\lambda\|u^\dagger\|^2.
\]
This completes the proof.
\end{proof}
\begin{rem}[Preconditioning comes from regression]
    The proof of \Cref{lem:source_condition} is based on the observation that $\Delta \theta^\pgd$ is the solution to the regularized least-square regression problem: $$\Delta \theta^\pgd=\arg\min_{\Delta \theta} \E_{z\sim \rho} \left[\left\| \bJ_{\theta} \Gmap_{\theta}(z) \Delta \theta - \nabla f_{\ell, \Pb,\Qb}(\Gmap_{\theta}(z)) \right\|^2 \right] + \lambda \|\Delta\theta\|^2 . $$
    The target of interest to be upper bounded corresponds to the ``approximation error'' in least-square regression; see \cite{fischer2020sobolev,shen2025nonparametric}.
    The proof of \Cref{lem:source_condition} is therefore a standard technique.
\end{rem}

\begin{prop}[Closeness of $\Pb_{\theta_{1}}^\pgd$ and $\tildePtheta{1}$]
\label{prop:tilde_no_tilde}
Let $\gamma>0$, $\lambda>0$, evaluation lengthscale $\ell>0$, and update lengthscale $\ell_0>0$ be fixed.
Suppose \Cref{ass:kernel,ass:regularity_G} hold. 
Given a fixed parameter $\theta_0\in\Theta$, let $\Pb_{\theta_{1}}^\pgd = (\Gmap_{\theta_0 - \gamma \Delta \theta_0^\pgd})_\# \rho$ be the next distribution under the population PGD scheme in Eq.~\eqref{eq:parametric_mmd_update_new}, and let $\tildePtheta{1} = (\Gmap_{\theta_0 - \gamma \tildetheta{0}})_\# \rho$ be the next distribution under the empirical PGD scheme in Eq.~\eqref{eq:parametric_mmd_update_new_tilde}. 
Then
\begin{align*}
    \E_m \left[
    \left| \mmd_{\ell}^2 (\tildePtheta{1}, \Qb) -\mmd_{\ell}^2(\Pb_{\theta_{1}}^\pgd, \Qb) \right|
    \right] \leq {C_k \gamma \lambda^{-1} p \ell^{-1}\ell_0^{-1} \sqrt{d}} \cdot \frac{1}{\sqrt{m}} . 
\end{align*}
Here, $C_k$ depends only on the kernel profile.
\end{prop}
\begin{proof}
For the Gaussian kernel and for the Mat\'ern kernel with smoothness $s>1$, we have $\|k_\ell(x, \cdot)-k_\ell(y, \cdot) \|_{\mathcal{H}_\ell} \leq C_k\ell^{-1} \|x-y\|$.
Hence, we have
\begin{align*}
    &\quad \left| \mmd_{\ell} (\tildePtheta{1}, \Qb) -\mmd_{\ell}(\Pb_{\theta_{1}}^\pgd, \Qb) \right| \leq \mmd_{\ell} (\tildePtheta{1}, \Pb_{\theta_{1}}^\pgd) \\
    &\leq {C_k\ell^{-1}}\left\| \Gmap_{\theta_0 - \gamma \Delta \theta_0^\pgd}(\cdot) - \Gmap_{\theta_0 - \gamma \tildetheta{0}}(\cdot) \right\|_{L_2^d(\rho)} \leq {C_k\sqrt{p} \ell^{-1} \gamma} \left\| \Delta \theta_0^\pgd - \tildetheta{0} \right\|.
\end{align*}
The last step holds by $\|\Gmap_\theta-\Gmap_{\theta^{\prime}}\|_{L_2^d(\rho)} \leq \sqrt{p}\|\theta-\theta^{\prime}\|$ from \Cref{ass:regularity_G}.
By definition, recall that
\begin{align*}
    \Delta \theta^\pgd &=  \left\{ \E_{z\sim\rho}\left[ \bJ_\theta \Gmap_{\theta_0}(z)^\top \bJ_\theta \Gmap_{\theta_0}(z)\right] + \lambda \Id \right\}^{-1} \E_{z\sim\rho}\left[ \bJ_\theta \Gmap_{\theta_0}(z)^\top \nabla f_{\ell_0, \Pb_{\theta_0}, \Qb}(\Gmap_{\theta_0}(z)) \right] \\
    \tildetheta{0} &=  \left\{ \E_{z\sim\rho}\left[ \bJ_\theta \Gmap_{\theta_0}(z)^\top \bJ_\theta \Gmap_{\theta_0}(z)\right] + \lambda \Id \right\}^{-1} \E_{z\sim\rho}\left[ \bJ_\theta \Gmap_{\theta_0}(z)^\top \nabla f_{\ell_0, \Pb_{\theta_0}, \widehat{\Qb}_m}(\Gmap_{\theta_0}(z)) \right] .
\end{align*}
So we have
\begin{align*}
    \left\| \Delta \theta^\pgd - \tildetheta{0} \right\| &\leq \lambda^{-1} \cdot \left| \E_{z\sim\rho}\left[ \bJ_\theta \Gmap_{\theta_0}(z)^\top \left( \nabla f_{\ell_0, \Pb_{\theta_0}, \Qb}(\Gmap_{\theta_0}(z)) - \nabla f_{\ell_0, \Pb_{\theta_0}, \widehat{\Qb}_m}(\Gmap_{\theta_0}(z)) \right) \right] \right| \\
    &\leq \lambda^{-1} \sqrt{p} \cdot \E_{z\sim\rho} \left[\left\| \nabla f_{\ell_0, \Pb_{\theta_0}, \Qb}(\Gmap_{\theta_0}(z)) - \nabla f_{\ell_0, \Pb_{\theta_0}, \widehat{\Qb}_m}(\Gmap_{\theta_0}(z)) \right\|\right] \\
    &\leq C_k\lambda^{-1} \sqrt{p} \ell_0^{-1} \sqrt{d} \mmd_{\ell_0}(\Qb, \widehat{\Qb}_m).
\end{align*}
The second last step holds by \Cref{ass:regularity_G}; and
the last step holds by the following: for each coordinate $j\in\{1, \ldots, d\}$ and any $x\in\R^d$,
\begin{align*}
    \left|\int \partial_{1,j} k_{\ell_0}(x,y) \,\dd(\widehat{\Qb}_m - \Qb)(y)\right|
    \leq
    \left\|\partial_{1, j} k_{\ell_0}(x,\cdot)\right\|_{\calH_{\ell_0}}
    \, \mmd_{{\ell_0}}(\widehat{\Qb}_m,\Qb) \leq C_k\ell_0^{-1} \mmd_{{\ell_0}}(\widehat{\Qb}_m,\Qb).
\end{align*}
Since the kernel is bounded by $1$, the standard concentration bound gives
$\E_m[\mmd_{\ell_0}(\widehat{\Qb}_m,\Qb)]\leq C_k m^{-1/2}$~\citep[Theorem 3.4]{muandet2017kernel}.
Taking expectation in the preceding display therefore gives
\begin{align*}
    &\quad \E_m\left[
    \left| \mmd_{\ell}^2 (\tildePtheta{1}, \Qb) -\mmd_{\ell}^2(\Pb_{\theta_{1}}^\pgd, \Qb) \right|
    \right] \\
    &\leq
    4\E_m\left[
    \left| \mmd_{\ell} (\tildePtheta{1}, \Qb) -\mmd_{\ell}(\Pb_{\theta_{1}}^\pgd, \Qb) \right|
    \right]
    \leq C_k \gamma \lambda^{-1} p \ell^{-1}\ell_0^{-1} \sqrt{d} \cdot \frac{1}{\sqrt{m}}.
\end{align*}
\end{proof}

\section{Connection with natural gradient}\label{sec:ngd}
Natural gradient descent is a popular optimization method in probabilistic machine learning as an alternative to
standard Euclidean gradient descent~\citep{amari1998natural,hoffman2013stochastic,pascanu2014revisiting,martens2020new}.
Let $\Pb_\theta$ be a probabilistic model with parameter $\theta\in\R^p$.
Let $\gamma>0$ be the step size,
and $g\in\R^p$ be the standard Euclidean gradient with respect to some objective, then the natural gradient is defined as
\begin{align}\label{eq:natural_gradient}
    h = \arg\min_{h\in\R^p} - h^\top g + \frac{1}{2 \gamma} \kl(\Pb_\theta\|\Pb_{\theta + h}) .
\end{align}
Under the well-known approximation of the KL divergence and the Fisher information matrix $\kl(\Pb_\theta\|\Pb_{\theta + h}) \approx h^\top F h$~\citep{pascanu2014revisiting,martens2020new}, the solution to Eq.~\eqref{eq:natural_gradient} can be solved in closed form: $h = \gamma F^{-1} g \in \R^p$, hence natural gradient.
Fast convergence has been proved for natural gradient method when the objective is a convex quadratic~\citep[Theorem 5]{martens2020new} and when training wide two-layer neural networks~\citep{zhang2019fast}.

A straight-forward MMD analogue of the natural gradient descent would be defined as the following:
\begin{align}\label{eq:mmd_minimizer}
    h = \arg\min_{h\in\R^p} - h^\top g + \frac{1}{2\gamma} \mmd_{\ell}^2(\Pb_\theta, \Pb_{\theta + h}) .
\end{align}
For the Gaussian kernel and for the Mat\'ern kernel with smoothness $s>1$, the map $h \mapsto \mmd_{\ell}^2(\Pb_\theta, \Pb_{\theta + h})$ admits a second-order expansion at $h=0$, and hence
\begin{align*}
    \mmd_{\ell}^2(\Pb_\theta, \Pb_{\theta + h}) &= 0 + h^\top \frac{\dd }{\dd h}\Big|_{h=0}  \mmd_{\ell}^2(\Pb_\theta, \Pb_{\theta + h}) \\
    &\qquad + \frac{1}{2} h^\top \frac{\dd^2 }{\dd h^2}\Big|_{h=0}  \mmd_{\ell}^2(\Pb_\theta, \Pb_{\theta + h})h + {o(\|h\|^2)}.
\end{align*}
In the following derivation, we focus on the case where $\Pb_\theta=(\Gmap_\theta)_{\#} \rho$ for some simple base probability measure $\rho$, to better align with the settings considered in the main text. We have
\begin{align*}
    \frac{\dd }{\dd h}\Big|_{h=0} \frac{1}{2} \mmd^2(\Pb_\theta, \Pb_{\theta + h}) = 0 .
\end{align*}
And, 
\begin{align*}
    \frac{\dd^2 }{\dd h^2}\Big|_{h=0}  \frac{1}{2} \mmd_{\ell}^2(\Pb_\theta, \Pb_{\theta + h}) &= \iint \bJ_\theta \Gmap_{\theta}(z)^\top \nabla_1 \nabla_2 k(\Gmap_\theta(z), \Gmap_\theta(z')) \bJ_\theta \Gmap_{\theta}(z') \dd \rho(z) \dd \rho(z') \\
    &= \E_{z,z'\sim\rho}\Big[ \bJ_\theta \Gmap_{\theta}(z)^\top \nabla_1 \nabla_2 k_\ell(\Gmap_{\theta}(z), \, \Gmap_{\theta}(z'))\bJ_\theta \Gmap_{\theta}(z') \Big].
\end{align*}
Therefore, we have
\begin{align*}
    \mmd_{\ell}^2(\Pb_\theta, \Pb_{\theta + h}) &= h^\top \left\{ \E_{z,z'\sim\rho}\Big[ \bJ_\theta \Gmap_{\theta}(z)^\top \nabla_1 \nabla_2 k_\ell(\Gmap_{\theta}(z), \, \Gmap_{\theta}(z'))\bJ_\theta \Gmap_{\theta}(z') \Big] \right\} h + o(\|h\|^2).
\end{align*}
The resulting minimizer of Eq.~\eqref{eq:mmd_minimizer} thus admits a closed-form \citep{briol2019statistical}
\begin{align}\label{eq:mmd_geometry}
    h = \gamma \left\{ \E_{z,z'\sim\rho}\Big[ \bJ_\theta \Gmap_{\theta}(z)^\top \nabla_1 \nabla_2 k_\ell(\Gmap_{\theta}(z), \, \Gmap_{\theta}(z'))\bJ_\theta \Gmap_{\theta}(z') \Big] \right\}^{-1} g,
\end{align}
where $g$ is the standard Euclidean gradient.
The preconditioning matrix here is distinct from the preconditioning matrix in Eq.~\eqref{eq:parametric_mmd_update_new} in the main text, which we recall below:
\begin{align*}
    h = \gamma \left\{ \E_{z\sim\rho}\Big[ \bJ_\theta \Gmap_{\theta_t}(z)^\top \bJ_\theta \Gmap_{\theta_t}(z)\Big] \right\}^{-1} g .
\end{align*}
Other preconditioning methods directly on $\calP_2(\R^d)$ have also been proposed
~\citep{dong2023particle,pfau2025wasserstein,bonet2024mirror}, which are also relevant yet distinct from our approaches.

\begin{figure}
    \centering
    \includegraphics[width=1.0\linewidth]{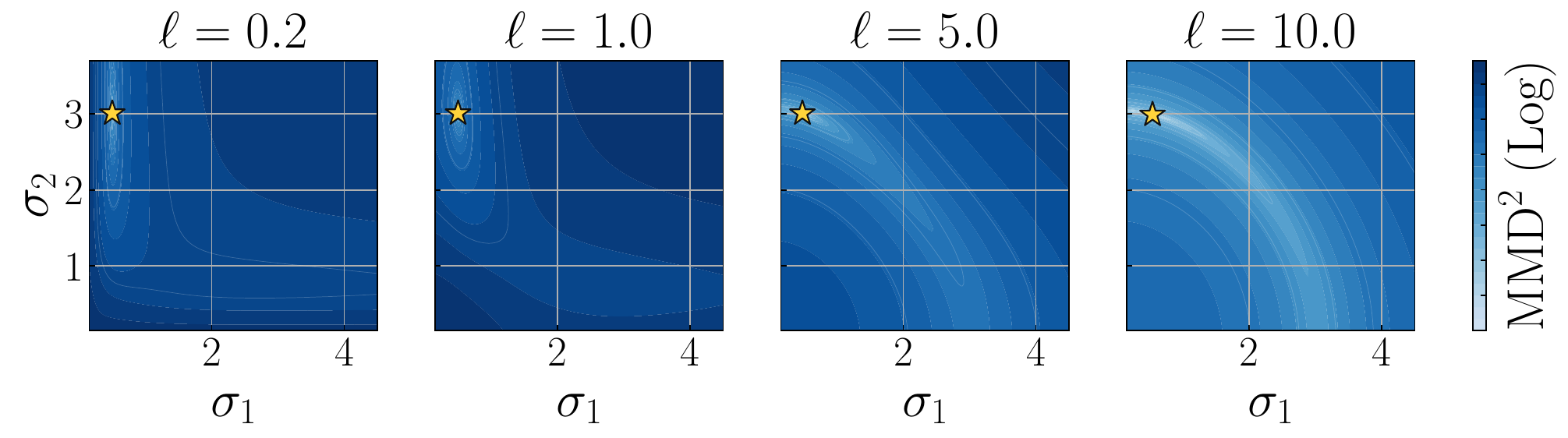}
    \vspace{-20pt}
    \caption{\textit{Landscape} of $\mmd_{\ell}^2(\Pb_\theta, \Qb)$ for the illustrative experiment with $\ell=\{0.2, 1.0, 5.0, 10.0\}$. }
    \label{fig:lengthscale_landscape}
    \vspace{-10pt}
\end{figure}

\section{Additional Experimental Details}\label{sec:add_experiment}
In this section, we provide additional details for reproducing our experiments in \Cref{sec:experiments}, as well as additional experiments to further support our claims made in the main text. All experiments were run on a workstation equipped with an AMD Ryzen 9 5950X 16-Core Processor and an NVIDIA GeForce RTX 4090 GPU. All kernels used in our experiments are Gaussian kernels.

\subsection{Illustrative experiments (mean-covariance Gaussian mixture)}\label{sec:illustrative_mean_cov}
Here, we provide additional details for the illustrative experiments in \Cref{fig:illustration}.
The target distribution is a symmetric two-component Gaussian mixture
$\Qb=\tfrac12\,\mathcal N(-0.25,\,0.5^2)+\tfrac12\,\mathcal N(0.25,\,3.0^2)$.
The model family is parameterized by component standard deviations,
$\Pb_{\theta}=\tfrac12\,\mathcal N(-0.25,\,\sigma_1^2)+\tfrac12\,\mathcal N(0.25,\,\sigma_2^2)$ with $\theta=(\sigma_1,\sigma_2)$ and $\sigma_1,\sigma_2>0$.
The kernel for computing the $\mmd_{\ell}^2(\Pb_{\theta},\Qb)$ is a Gaussian kernel with kernel lengthscale $\ell =1.0$.

For standard gradient descent (GD), we use a fixed step size $\gamma=10^{-4}$, initialization $\theta_0=(4.0,3.5)$, and $T=650$ iterations.
For our preconditioned gradient descent (PGD), we use $n=512$ particles drawn from $\Pb_\theta$ to estimate the preconditioning matrix, ridge regularization $\lambda=10^{-3}$, initialization $\theta_0=(4.0,3.5)$, step size $\gamma=100.0$ and $T=650$ iterations.
PGD uses an adaptive schedule
$\ell_t=\max\{\ell,\,5.0 \times 0.99^t\}$.
The choice of step size is optimized over a pre-specified candidate set.
In this setting, since everything is Gaussian, $\mmd_\ell(\Pb_\theta, \Qb)$ can be computed in closed form~\citep{chen2024conditional} and hence no samples from $\Qb$ are required to estimate the gradient for both GD and PGD.
This model can be written as a pushforward of a Bernoulli variable and a standard Gaussian variable, with $\Gmap_\theta(B,X)=(-0.25+\sigma_1 X)\mathbf{1}\{B=1\}+(0.25+\sigma_2 X)\mathbf{1}\{B=2\}$. Hence $\Gmap_\theta$ and its parameter derivatives are square-integrable, and the derivative bounds in \Cref{ass:regularity_G} hold uniformly when $\Theta$ is chosen as a compact convex subset of $\{\sigma_1,\sigma_2>0\}$.

From the left panel of \Cref{fig:illustration}, we can see that the PGD descent scheme converges to the global minimum marked as the yellow star; while the standard GD scheme gets stuck at a local minimum.
This behaviour can also be reflected in the right panel of Figure 1 where $\mathrm{MMD}_{\ell}^2\left(\mathbb{P}_{\theta_t}^{\mathrm{PGD}}, \mathbb{Q}\right)$ converges to 0 while $\mathrm{MMD}_{\ell}^2\left(\mathbb{P}_{\theta_t}^{\mathrm{GD}}, \mathbb{Q}\right)$ does not and gets stuck at around $10^{-2}$.
Furthermore, \Cref{fig:lengthscale_landscape} plots the objective landscape of $\operatorname{MMD}_{\ell}^2\left(\mathbb{P}_\theta, \mathbb{Q}\right)$ for several kernel lengthscales $\ell \in\{0.2,1.0,5.0,10.0\}$. For small lengthscales, the landscape is non-convex, with two separated basins in which gradient descent can become trapped. By contrast, for larger lengthscales, the landscape remains non-convex but becomes effectively unimodal, such that gradient-based optimization methods could converge to the right minimum.

\begin{figure}[t]
    \centering
    \includegraphics[width=1.0\linewidth]{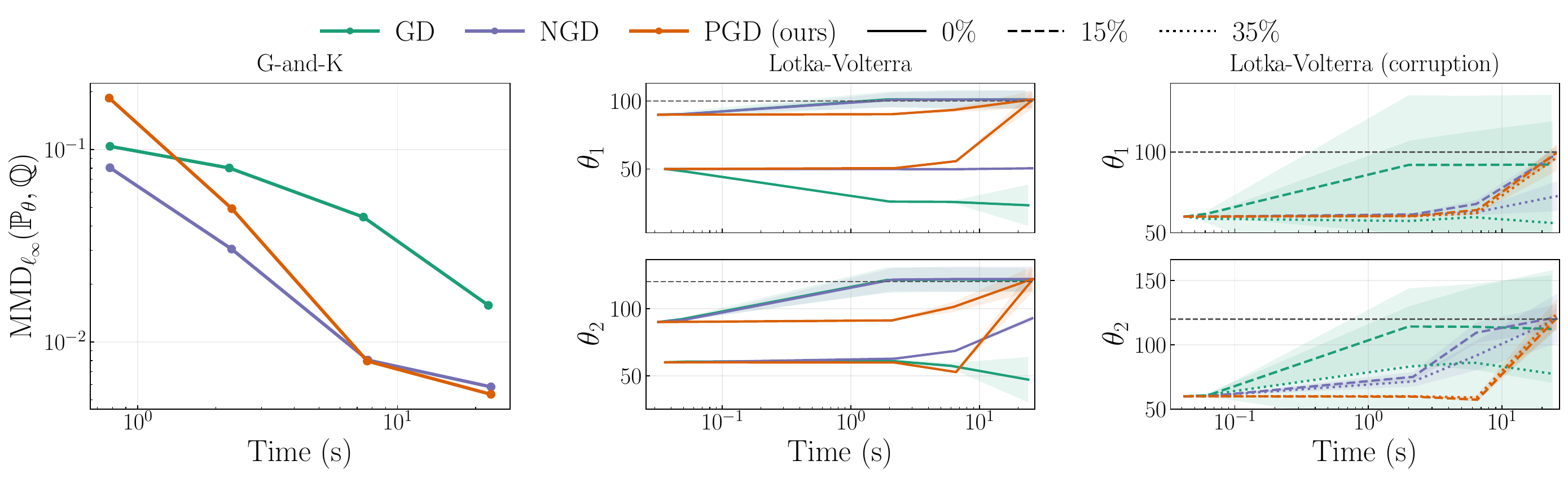}
    \vspace{-10pt}
    \caption{\textit{Computational time}. \textbf{Left:} G-and-k distribution. \textbf{Middle \& Right:} Lotka-Volterra model.}
    \label{fig:time_plot}
    \vspace{-10pt}
\end{figure}
\subsection{Mixture of Gaussians}
The target distribution $\Qb$ is an equally weighted mixture of eight Gaussian distributions on $\R^2$ whose centers lie on a circle of radius $2.0$.
Specifically, the $i$-th component has mean
$\left(2\cos(\frac{2\pi i}{8}), 2\sin(\frac{2\pi i}{8})\right)$ for $i=\{0,\ldots,7\}$,
and a diagonal covariance matrix $0.04 \Id$.
The initial $N$ particles are drawn independently from an initialization distribution $\calN(0, 1.0)$.
Both fixed- and adaptive-lengthscale schemes use the same adaptive step size $\gamma_t=0.01/(1+t)^{0.1}$.

\subsection{G-and-k distribution}
The statistical model is given by $(\Gmap_\theta)_{\#} \rho$, with $\rho$ being a uniform distribution over $[0,1]$, and
\begin{align}\label{eq:G_theta_gandk}
    \Gmap_\theta(z):=a+b\left(1+0.8 \frac{1-\exp \left(-c \Phi^{-1}(z ; 0,1)\right)}{1+\exp \left(-c \Phi^{-1}(z ; 0,1)\right)}\right)\left(1+\left(\Phi^{-1}(z ; 0,1)\right)^2\right)^{\exp(k)} \Phi^{-1}(z ; 0,1),
\end{align}
where $\Phi^{-1}(z ; 0,1)$ is the quantile function of a standard normal distribution. In this case, $p = 4$ and $d = 1$. 
The true value is set to $\theta^\ast=(3,1,1,\exp(-\log 2))$ following the setting of \cite{briol2019statistical}, and the parameter domain is $\Theta=[-10,10] \times[0.1,10] \times[-10,10] \times[0,2]$ for $(a,b,c,\exp(k))$.
We next check the integrability requirements in \Cref{ass:regularity_G}. Let $x:= \Phi^{-1}(z ; 0,1) \sim\calN(0,1)$. Then, we have
\begin{align*}
\Gmap_\theta(z)=a+b h_c(x) \left(1+x^2\right)^{\exp(k)} x,\quad h_c(x):=1+0.8 \tanh (c x / 2).
\end{align*}
The first-order derivatives are
\begin{align*}
&\partial_a \Gmap_\theta(Z)=1, \quad
\partial_b \Gmap_\theta(Z)=h_c(x) \left(1+x^2\right)^{\exp(k)} x, \\
&\partial_c \Gmap_\theta(Z)=0.4 b x^2 \left(1+x^2\right)^{\exp(k)} \operatorname{sech}^2(c x / 2), \quad
\partial_{\exp(k)} \Gmap_\theta(Z)=b h_c(x) \left(1+x^2\right)^{\exp(k)} x \log \left(1+x^2\right) .
\end{align*}
For each fixed $\theta$, $\Gmap_\theta$ and all displayed first-order derivatives are bounded by polynomials in $|x|$, and hence have finite second moments under $x\sim\calN(0,1)$.
The same argument applies to the second-order derivatives. Therefore, the integrability requirements hold pointwise in $\theta$, and the uniform bounds in \Cref{ass:regularity_G} hold since $\Theta$ is compact and convex. 

For our PGD scheme, at each iteration we draw $n=600$ samples from $\rho$ to compute both the preconditioner $\widehat H_t$ and the gradient $\widehat g_t$ in \Cref{sec:parametric}.
We have access to a fixed set of $m=1000$ IID samples from the target $\Qb$, and approximate expectations with respect to $\Qb$ using minibatches of size $600$ randomly subsampled at each iteration.
The same set of samples is used for standard GD and for the PGD scheme of \cite{briol2019statistical}.
GD used a fixed lengthscale, $\ell_t=\ell=2.0$, whereas PGD used an adaptive lengthscale, $\ell_t=\max\{2.0, 10.0 \times 0.99^t\}$.

In the left panel of \Cref{fig:time_plot}, we report the computational time of GD, PGD~\citep{briol2019statistical}, and our PGD method on the g-and-k distribution. Although both PGD schemes are more expensive per iteration due to the additional preconditioning step, they attain lower MMD objectives under the same computational budget. This suggests that the extra cost of preconditioning is justified by improved optimization efficiency.
Moreover, we see that our PGD scheme slightly outperforms NGD under larger computational budget, because the preconditioner of NGD is more costly than our preconditioner, and hence our method becomes more effective with more iterations.

In the top row of \Cref{fig:ablation_summary}, we present an extensive ablation study of the hyperparameters in our proposed PGD scheme. Specifically, we vary the step size $\gamma$, the lengthscale decay rate (both exponential $\ell_t=\max\{\ell, 3000 \cdot \eta^t\}$ and also polynomial $\ell_t=\max\{\ell, 3000\cdot (1+t)^{-\beta} \}$), the number of samples $n$ drawn from $\rho$ for estimating the preconditioner and gradient, and the regularization parameter $\lambda$. We do not include an ablation over the terminal lengthscale $\ell$,
since its choice primarily affects the statistical properties of the minimum MMD estimator, whereas our focus here is on the optimization behaviour of the proposed method.
The first panel shows that our PGD scheme performs consistently well across a reasonable range of step sizes, $\gamma\in\{0.03,0.1,0.3,1.0\}$. The second panel demonstrates similar robustness to the choice of lengthscale decay rate, where the adaptive lengthscale is set as $\ell_t=\max\{2.0, 10.0\times \eta^t\}$ with various choices of $\eta$. The third panel shows that the number of samples used to estimate the preconditioner and gradient can be as small as $n=50$ while still yielding stable performance; however, when $n$ is too small, the resulting estimator has high variance and the final MMD values become worse.
Finally, the last panel indicates that our PGD scheme generally favors smaller values of the regularization parameter $\lambda$ than large values.

\begin{figure}[t]
    \centering
    \begin{subfigure}{1.0\linewidth}
        \centering
        \includegraphics[width=\linewidth]{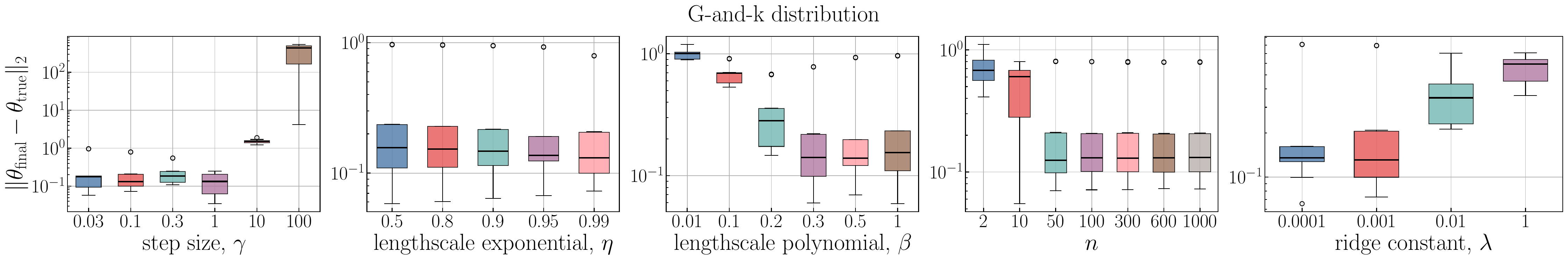}
    \end{subfigure}
    \begin{subfigure}{1.0\linewidth}
        \centering
        \includegraphics[width=\linewidth]{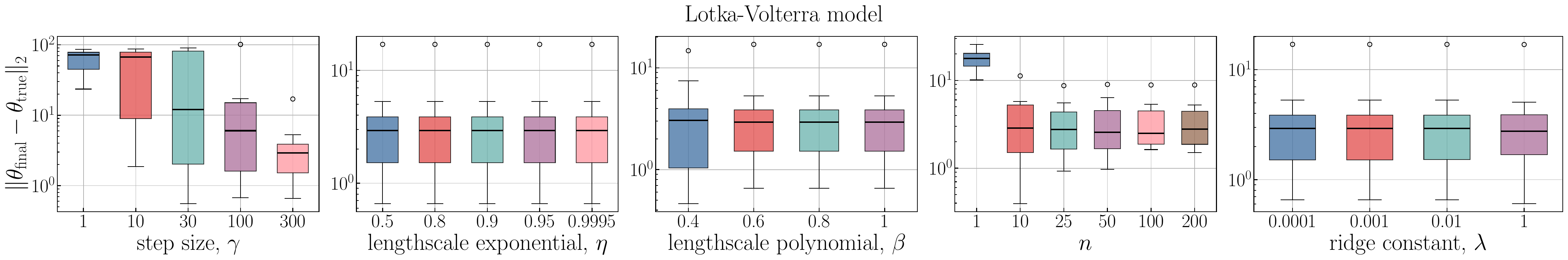}
    \end{subfigure}
    \vspace{-10pt}
    \caption{\textit{Ablation of all hyperparameters in PGD} \textbf{Top row:} G-and-k distribution. \textbf{Bottom row:} Lotka-Volterra model in the well-specified case with initial parameter $\theta_0=(90, 90)$. }
    \vspace{-10pt}
    \label{fig:ablation_summary}
\end{figure}

\subsection{Lotka-Volterra model}
The stochastic Lotka-Volterra model consists of a pair of nonlinear differential equations describing the evolution of two species through time~\citep{lotka1927fluctuations}:
\begin{align}\label{eq:lv_sde}
\mathrm{d}\binom{X_{1, t}}{X_{2, t}} & =\left[\binom{1}{0} \kappa_{11} X_{1, t}+\binom{-1}{1} \kappa_{12} X_{1, t} X_{2, t}+\binom{0}{-1} \kappa_{13} X_{2, t}\right] \mathrm{d} t \\
& +\binom{1}{0} \sqrt{\kappa_{11} X_{1, t}} \mathrm{~d} W_t^{(1)}+\binom{-1}{1} \sqrt{\kappa_{12} X_{1, t} X_{2, t}} \mathrm{~d} W_t^{(2)}+\binom{0}{-1} \sqrt{\kappa_{13} X_{2, t}} \mathrm{~d} W_t^{(3)}, \nonumber
\end{align}
Here, the coefficients $(\kappa_{11},\kappa_{12},\kappa_{13}) =(5,0.025,6)$ are known, and $\theta=(X_{1,0},X_{2,0})$ represents the initial conditions. $\Pb_\theta$ is the distribution of the final output of the above stochastic Lotka-Volterra model simulated via an explicit Euler--Maruyama scheme with $\calT$ steps over the interval $[0,1]$ so that the discretization step size is $\Delta t = 1/\calT$.
The density of $\Pb_\theta$ is intractable; however, samples can be generated by forward simulation of the discretized stochastic dynamics.
Let $z= (z_{\mathfrak{t}, r})_{\mathfrak{t}=0, \ldots, \calT-1; r=1,2,3}$ and the base distribution $\rho = \mathcal{N}\left(0, \Id_{3 \calT} \right)$.
The pushforward map $\Gmap_\theta:\R^{3\calT}\to\R^2$ maps  $(z_{\mathfrak{t}, r})_{\mathfrak{t}=0, \ldots, \calT-1; r=1,2,3}$ to the output
$(X_{1,\calT}, X_{2,\calT})$
defined via the following recursion:
for $\mathfrak{t}=0, \ldots, \calT-1$,
\begin{align*}
&\begin{aligned}
X_{1, \mathfrak{t}+1}= & X_{1, \mathfrak{t}} + \left(\kappa_{11} X_{1, \mathfrak{t}} -\kappa_{12} X_{1, \mathfrak{t}} X_{2, \mathfrak{t}}\right) \Delta t +\sqrt{\kappa_{11} X_{1, \mathfrak{t}} \Delta t} \, z_{\mathfrak{t}, 1}-\sqrt{\kappa_{12} X_{1, \mathfrak{t}} X_{2, \mathfrak{t}} \Delta t} \, z_{\mathfrak{t}, 2}, \\
X_{2, \mathfrak{t}+1}= & X_{2, \mathfrak{t}} + \left(\kappa_{12} X_{1, \mathfrak{t}} X_{2, \mathfrak{t}}-\kappa_{13} X_{2, \mathfrak{t}}\right) \Delta t +\sqrt{\kappa_{12} X_{1, \mathfrak{t}} X_{2, \mathfrak{t}} \Delta t} \, z_{\mathfrak{t}, 2}-\sqrt{\kappa_{13}
X_{2, \mathfrak{t}} \Delta t} \, z_{\mathfrak{t}, 3} .
\end{aligned}
\end{align*}
We are unable to verify \Cref{ass:regularity_G} for this simulator. Although the Euler--Maruyama map is differentiable away from the boundary of the positive orthant, the square-root diffusion coefficients make the required uniform second-moment bounds for the first and second order derivatives delicate when a simulated population approaches zero.

Here, $m=100$ IID samples are drawn from the target distribution $\Qb=\Pb_{\theta^\ast}$ with $\theta^\ast=(100,120)$ and are then fixed throughout. At each iteration, $n=50$ samples are drawn from $\rho$ to compute both the preconditioner $\widehat H_t$ and the gradient $\widehat g_t$ in \Cref{sec:parametric}.
Since both $m$ and $n$ are relatively small here, we use the full batch of samples without subsampling.
The same set of samples is used for standard GD and the PGD scheme of \cite{briol2019statistical}.
GD used a fixed lengthscale, $\ell_t=\ell=30$, whereas PGD used an adaptive lengthscale, $\ell_t=\max\{30, 3000 \times 0.9995^t\}$.

In the right two panels of \Cref{fig:time_plot}, we report the computational time of GD, PGD~\citep{briol2019statistical}, and our PGD method on the Lotka-Volterra model. Although both PGD schemes are more expensive per iteration due to the additional preconditioning step, they attain lower MMD objectives under the same computational budget. This suggests that the extra cost of preconditioning is justified by improved optimization efficiency.
Also, in the rightmost panel of \Cref{fig:time_plot}, GD does not recover the true $\theta^\ast$ under data corruption.
Moreover, we see that our PGD scheme outperforms NGD for $\theta_1, \theta_2$ in the middle panel of \Cref{fig:time_plot} and for $\theta_1$ in the right panel of \Cref{fig:time_plot}; whereas NGD outperforms our PGD for $\theta_2$ in the right panel of \Cref{fig:time_plot}.
We were unable to reproduce the results in Figure~6 of \cite{briol2019statistical}, where convergence is achieved within only a few tens of iterations, since the code is not publicly available.
We suspect that they might use a substantially larger step size than we do.
More experimental evidence is therefore required to better understand which preconditioner is preferable in practice.

In the bottom row of \Cref{fig:ablation_summary}, we present an extensive ablation study of the hyperparameters in our proposed PGD scheme.
Specifically, we vary the initial step size $\gamma_0$, the lengthscale decay rate (both exponential $\ell_t=\max\{\ell, 3000 \cdot \eta^t\}$ and also polynomial $\ell_t=\max\{\ell, 3000\cdot (1+t)^{-\beta} \}$), the number of samples $n$ drawn from $\rho$ for estimating the preconditioner and gradient, and the regularization parameter $\lambda$. We do not include an ablation over the terminal lengthscale $\ell$,
since its choice primarily affects the statistical properties of the minimum MMD estimator, whereas our focus here is on the optimization behaviour of the proposed method.
The first panel shows that our PGD scheme performs consistently well across a reasonable range of initial step sizes, $\gamma\in\{30, 100, 300\}$.
The second and third panels demonstrate similar robustness to reasonable choices of lengthscale decay rate.
The third panel shows that the number of samples used to estimate the preconditioner and gradient can be as small as $n=10$ while still yielding stable performance; however, when $n$ is as small as $n=1$, the resulting estimator has high variance and the final MMD values become worse.
Finally, the last panel indicates that our PGD scheme is not very sensitive to the regularization parameter $\lambda$.

\subsection{Toggle switch model}\label{sec:more_toggle}
Starting from fixed initial conditions $(u_0, v_0) = (10,10)$, the latent states evolve as
\begin{align}
u_{t+1}
&\sim \mathcal{N}_+\!\left(
u_t + \frac{\alpha_1}{1+v_t^{\beta_1}} - (1+0.03\,u_t),\; 0.5
\right),
\\
v_{t+1}
&\sim \mathcal{N}_+\!\left(
v_t + \frac{\alpha_2}{1+u_t^{\beta_2}} - (1+0.03\,v_t),\; 0.5
\right),
\end{align}
for $t=0,...,\calT-1$, where $\mathcal{N}_+$ denotes a positively-truncated normal distribution. The learnable parameters are $\theta=(\alpha_1,\alpha_2, \beta_1, \beta_2, \mu, \sigma, \gamma)$. 
The parameter domain is set to be $\Theta = [0.01,50] \times[0.01,50] \times[0.01,5] \times[0.01,5] \times[250,450] \times[0.01,0.50] \times[0.01,0.40]$. 
This parametric model is  $\Pb_{\theta;\calT}=\tilde{\mathcal{N}}\left(\mu+u_T, \mu \sigma / u_T^\gamma\right)$ with an explicit dependence on the time horizon $\calT$.
$\Pb_{\theta;\calT}$ is likelihood-free: samples can be generated by simulating the dynamics, but its density is not available in closed form due to its dependence on the dynamical system.
Following Appendix C.4 of \cite{dellaporta2022robust}, the toggle-switch model can be written as a push-forward distribution
$\Pb_{\theta;\calT} = (\Gmap_{\theta;\calT})_\#\rho$, with $\rho=\operatorname{Unif}\bigl([0,1]^{2\calT+1}\bigr)$ and $\Gmap_{\theta;\calT}$ outlined below in Eq.~\eqref{eq:G_theta_toggle} .
Let
$\xi=(\xi_{0,1}, \xi_{0,2},\ldots, \xi_{\calT-1,1}, \xi_{\calT-1,2},\xi_{\calT})^\top
\sim \rho$ denote the simulator input.
Starting from fixed initial conditions $(u_0,v_0)=(10,10)$, define recursively, for $t=0,\ldots,\calT-1$,
\[
\begin{aligned}
    \widetilde u_{t+1}
    =
    u_t+\frac{\alpha_1}{1+v_t^{\beta_1}}
    -(1+0.03u_t), \quad \widetilde v_{t+1}
    =
    v_t+\frac{\alpha_2}{1+u_t^{\beta_2}}
    -(1+0.03v_t),
\end{aligned}
\]
and
\[
\begin{aligned}
    u_{t+1}
    &=
    \widetilde u_{t+1}
    +0.5\,\Phi^{-1}
    \left[
        \Phi(-2\widetilde u_{t+1})
        +\xi_{t,1}\{1-\Phi(-2\widetilde u_{t+1})\}
    \right], \\
    v_{t+1}
    &=
    \widetilde v_{t+1}
    +0.5\,\Phi^{-1}
    \left[
        \Phi(-2\widetilde v_{t+1})
        +\xi_{t,2}\{1-\Phi(-2\widetilde v_{t+1})\}
    \right],
\end{aligned}
\]
where $\Phi$ is the standard normal cumulative distribution function (CDF).
The simulator output is then
\begin{align}\label{eq:G_theta_toggle}
\begin{aligned}
    \Gmap_{\theta;\calT}(\xi)
    &=
    \mu+u_{\calT}
    +
    \frac{\mu\sigma}{u_{\calT}^{\gamma}}
    \Phi^{-1}
    \left[
        \Phi\left(
            -\frac{(\mu+u_{\calT})u_{\calT}^{\gamma}}{\mu\sigma}
        \right)
        +
        \xi_{\calT}
        \left\{
            1-
            \Phi\left(
                -\frac{(\mu+u_{\calT})u_{\calT}^{\gamma}}{\mu\sigma}
            \right)
        \right\}
    \right].
\end{aligned}
\end{align}
We explicitly emphasize its dependence on the time horizon $\calT$.
The compact product domain $\Theta$ keeps all parameters uniformly bounded and avoids the singular values $\mu=0$ and $\sigma=0$. 
The upper bound $\gamma\leq 0.4$ also controls the negative powers of $u_{\calT}$ appearing in the observation map. 
Since $u_{\calT}$ is positively truncated Gaussian conditional on the previous latent states, its negative moments are finite up to order strictly smaller than one; hence the factors $u_{\calT}^{-\gamma}$ appearing in the observation model have finite second moments. The inverse-CDF representation is smooth for simulator inputs $\xi\in (0,1)^{2\calT+1}$, and Mills-ratio bounds imply that derivatives of the truncated-Gaussian quantile maps grow at most polynomially in the truncated-normal variates~\citep{gordon1941values}. 
Thus, for any fixed horizon $\calT$, the regularity requirement on $\Gmap_{\theta;\calT}$ in \Cref{ass:regularity_G} is satisfied. 

Composite goodness-of-fit testing asks whether some unknown target distribution $\Qb$ belongs to a prescribed parametric family,
\[
    H_0:\Qb\in\{\Pb_{\theta;\calT}:\theta\in\Theta\},
    \qquad
    H_1:\Qb\notin\{\Pb_{\theta;\calT}:\theta\in\Theta\}.
\]
Let $\widehat{\Qb}_m$ denote the empirical distribution consisting of $m$ IID samples from $\Qb$.
Following \cite{key2025composite}, we use the smallest discrepancy between $\widehat{\Qb}_m$ and the parametric model class as the test statistic:
\[
\Delta_m := m \min_{\theta\in\Theta} \mmd_{\ell}^2 (\Pb_{\theta;\calT},\widehat{\Qb}_m),
\]
which is precisely our minimum MMD estimator.
Here, we set the target distribution $\Qb$ to be $\Pb_{\theta^\ast;\calT}$ with $\theta^\ast = (22.0, 12.0, 4.0, 4.5, 325.0, 0.25, 0.15)$.
For testing under the null,
$\Qb=\Pb_{\theta^\ast;50}$, whereas for testing power under the
alternative, $\Qb=\Pb_{\theta^\ast;20}$.

Following Appendix D.3.3 of \cite{key2025composite}, $\calI = 200$ initial parameter values are sampled from a uniform distribution.
Among these initializations, they retain the $15$ candidates with the smallest initial MMD values and run gradient descent on the MMD objective from each of them for $300$ iterations, using a learning rate of $0.04$.
The final estimate is then selected as the iterate, among the $15$ runs, with the smallest MMD value.
This extensive initialization search is intended to find starting points close to a global minimizer, thereby alleviating the difficulty of the subsequent non-convex optimization problem.
In contrast, when gradient descent is run from a single initialization, $\calI=1$, the resulting test is poorly calibrated. 
Under the null, the population minimum MMD is zero; however, poor initialization can cause GD to converge to a suboptimal local minimum, leaving a large empirical MMD value and leading to spurious rejections. 
Thus, the multi-start procedure plays a crucial role in maintaining calibration, but it also incurs a substantial computational overhead.

In contrast, to find the minimum MMD estimator, our PGD scheme is well-calibrated even starting with only a single initialization $\calI=1$.
This indicates that our PGD scheme can reliably find low-MMD solutions without the extensive initialization search used by GD.
Since our PGD does not require extensive initialization search, it is computationally much cheaper than the GD scheme used in \cite{key2025composite}.

\section{Auxiliary Results}\label{sec:aux}
\begin{lem}\label{lem:lengthscale_lip_bound}
Let $\varphi:[0, \infty) \rightarrow \mathbb{R}$ be either the Gaussian profile
$
\varphi(r)=\exp \left(-r^2 / 2\right),
$
or the Matérn profile with smoothness parameter $s>1$,
$\varphi_s(r)=\frac{2^{1-s}}{\Gamma(s)}(\sqrt{2 s} r)^s K_s(\sqrt{2 s} r)$, 
where $K_s$ is the modified Bessel function of the second kind. For $a>0$, define $\psi_a(t):=\varphi\left(\frac{\|t\|}{\sqrt{a}}\right)$.
Then for any $a_1>a_2>0$,
\begin{align*}
    \sup _{t \in \mathbb{R}^d}\left|\psi_{a_1}(t)-\psi_{a_2}(t)\right| \leq C_0 \frac{\left|a_1-a_2\right|}{a_2}, \quad \sup _{t \in \mathbb{R}^d}\left\| \nabla \psi_{a_1}(t)- \nabla \psi_{a_2}(t) \right\| \leq C_1 \frac{\left|a_1-a_2\right|}{a_2^{3 / 2}} .
\end{align*}
Here, $C_0,C_1$ are two constants depending only on the kernel.
\end{lem}
\begin{proof}
Let $\varrho:=\frac{\|t\|}{\sqrt{a}}$.
We first control the zeroth-order term. Differentiating $\psi_a(t)$ with respect to $a$ gives
$\partial_a \psi_a(t) = -\frac{1}{2 a} \varrho \varphi^{\prime}(\varrho)$.
Therefore, $\left|\partial_a \psi_a(t)\right| \leq \frac{1}{a}\left(\frac{1}{2} \sup _{\varrho \geq 0} \varrho\left|\varphi^{\prime}(\varrho)\right|\right)$.
Define
$$
C_0:=\frac{1}{2} \sup _{\varrho \geq 0} \varrho\left|\varphi^{\prime}(\varrho)\right| .
$$
For the Gaussian profile, $C_0<\infty$ follows directly from
$\varrho\left|\varphi^{\prime}(\varrho)\right|=\varrho^2 e^{-\varrho^2 / 2}$.
For the Matérn profile with $s>1$, the standard asymptotics of $K_s$ give, as $\varrho\downarrow 0$,
\[
\varphi_s(\varrho)=1-\frac{s}{2(s-1)}\varrho^2+o(\varrho^2),
\]
and hence $\varrho|\varphi_s'(\varrho)|=O(\varrho^2)$ near the origin. At infinity, the exponential decay of $K_s$ gives exponential decay of $\varphi_s'$ up to polynomial factors, so $\varrho|\varphi_s'(\varrho)|$ is bounded. Hence $C_0<\infty$ in both cases.
By the mean value theorem, for any $a_1>a_2>0$,
$
\left|\psi_{a_1}(t)-\psi_{a_2}(t)\right| \leq\left|a_1-a_2\right| \sup _{a \in\left[a_2, a_1\right]}\left|\partial_a \psi_a(t)\right|$.
Using the previous bound,
$$
\left|\psi_{a_1}(t)-\psi_{a_2}(t)\right| \leq C_0\left|a_1-a_2\right| \sup _{a \in\left[a_2, a_1\right]} \frac{1}{a}=C_0 \frac{\left|a_1-a_2\right|}{a_2} .
$$
Taking the supremum over $t \in \mathbb{R}^d$ proves the first claim.

We next control the gradient. For $t \neq 0$, let $e_t:=\frac{t}{\|t\|}$. Then $\nabla_t \psi_a(t)=a^{-1 / 2} \varphi^{\prime}(\varrho) e_t$.
For the Gaussian profile and the Matérn profile with $s>1$, the expansion above gives $\varphi'(0)=0$, so this gradient extends continuously to $t=0$ by setting it equal to zero.
Differentiating with respect to $a$, and noting that $e_t$ is independent of $a$, gives
$$
\partial_a (\nabla_t \psi_a(t)) =-\frac{1}{2 a^{3 / 2}} \varphi^{\prime}(\varrho) e_t+a^{-1 / 2} \varphi^{\prime \prime}(\varrho)\left(-\frac{\varrho}{2 a}\right) e_t .
$$
Hence, $\partial_a (\nabla_t \psi_a(t)) = -\frac{1}{2 a^{3 / 2}}\left[\varphi^{\prime}(\varrho)+\varrho \varphi^{\prime \prime}(\varrho)\right] e_t$. Therefore,
$\left\|\partial_a (\nabla_t \psi_a(t)) \right\| \leq \frac{1}{a^{3 / 2}} (\frac{1}{2} \sup _{\varrho \geq 0}|\varphi^{\prime}(\varrho)+\varrho \varphi^{\prime \prime}(\varrho)| )$.
Define
$$
C_1:=\frac{1}{2} \sup _{\varrho \geq 0}\left|\varphi^{\prime}(\varrho)+\varrho \varphi^{\prime \prime}(\varrho)\right| .
$$
For the Gaussian profile, $\varphi^{\prime}(\varrho)=-\varrho e^{-\varrho^2 / 2}, \quad \varphi^{\prime \prime}(\varrho)=\left(\varrho^2-1\right) e^{-\varrho^2 / 2}$,
so $\varphi^{\prime}(\varrho)+\varrho \varphi^{\prime \prime}(\varrho)=\left(\varrho^3-2 \varrho\right) e^{-\varrho^2 / 2}$,
which is bounded on $[0, \infty)$. Hence $C_1<\infty$ for the Gaussian kernel.
For the Matérn profile, the assumption $s>1$ ensures bounded second-order behavior at the origin. More precisely, the expansion above gives
$
\varphi_s^{\prime}(\varrho)+\varrho \varphi_s^{\prime \prime}(\varrho)=O(\varrho) \quad \text { as } \varrho \downarrow 0$.
At infinity $\varrho\approx\infty$, the exponential decay of $K_s$ implies exponential decay of $\varphi_s$ and its derivatives. Therefore, $\sup _{\varrho \geq 0}\left|\varphi_s^{\prime}(\varrho)+\varrho \varphi_s^{\prime \prime}(\varrho)\right|<\infty$.
Thus $C_1<\infty$ for the Matérn kernel with smoothness $s>1$.
Applying the mean-value theorem again concludes the proof.
\end{proof}

\begin{lem}[Strong convexity implies gradient dominance]\label{lem:strong_convex}
    Suppose that $\theta \mapsto \mmd_\ell^2(\Pb_\theta, \Qb)$ is $\alpha$ strongly convex. Suppose \Cref{ass:regularity_G} holds. Then, $\mmd_\ell^2(\Pb_\theta, \Qb) - \min_{\theta\in\Theta} \mmd_{\ell}^2(\Pb_{\theta}, \Qb)\leq \frac{2p}{\alpha}  \E_{z \sim \rho}\left[\left\|\nabla f_{\ell, \mathbb{P}_\theta, \mathbb{Q}}\left(\Gmap_\theta(z)\right)\right\|^2\right]$.
\end{lem}
\begin{proof}
Since $\alpha$ strongly convex implies gradient dominance~\citep{karimi2016linear}, we have
\begin{align*}
    &\quad \mmd_\ell^2(\Pb_\theta, \Qb) - \min_{\theta\in\Theta} \mmd_{\ell}^2(\Pb_{\theta}, \Qb) \leq \frac{1}{2\alpha} \|\nabla_\theta\mmd_\ell^2(\Pb_\theta, \Qb) \|^2 \\
    &= \frac{2}{\alpha} \left\| \E_{z\sim\rho}\left[ \bJ_\theta \Gmap_{\theta}(z)^\top \nabla f_{\ell, \Pb_{\theta}, \Qb}(\Gmap_{\theta}(z)) \right]\right\|^2 = \frac{2}{\alpha} \sum_{j=1}^p\left|\E_{z \sim \rho}\left[\partial_{\theta_j} \Gmap_\theta(z)^\top \nabla f_{\ell, \mathbb{P}_\theta, \mathbb{Q}}\left(\Gmap_\theta(z)\right) \right]\right|^2 \\
    &\leq \frac{2}{\alpha}  \sum_{j=1}^p \E_{z \sim \rho}\left[\left\|\partial_{\theta_j} \Gmap_\theta(z)\right\|^2\right] \E_{z \sim \rho}\left[\left\|\nabla f_{\ell, \mathbb{P}_\theta, \mathbb{Q}}\left(\Gmap_\theta(z)\right)\right\|^2\right] \leq \frac{2p}{\alpha}  \E_{z \sim \rho}\left[\left\|\nabla f_{\ell, \mathbb{P}_\theta, \mathbb{Q}}\left(\Gmap_\theta(z)\right)\right\|^2\right].
\end{align*}
The last inequality holds by \Cref{ass:regularity_G}.
\end{proof}

\begin{figure}[t]
    \centering
    \includegraphics[width=0.3\linewidth]{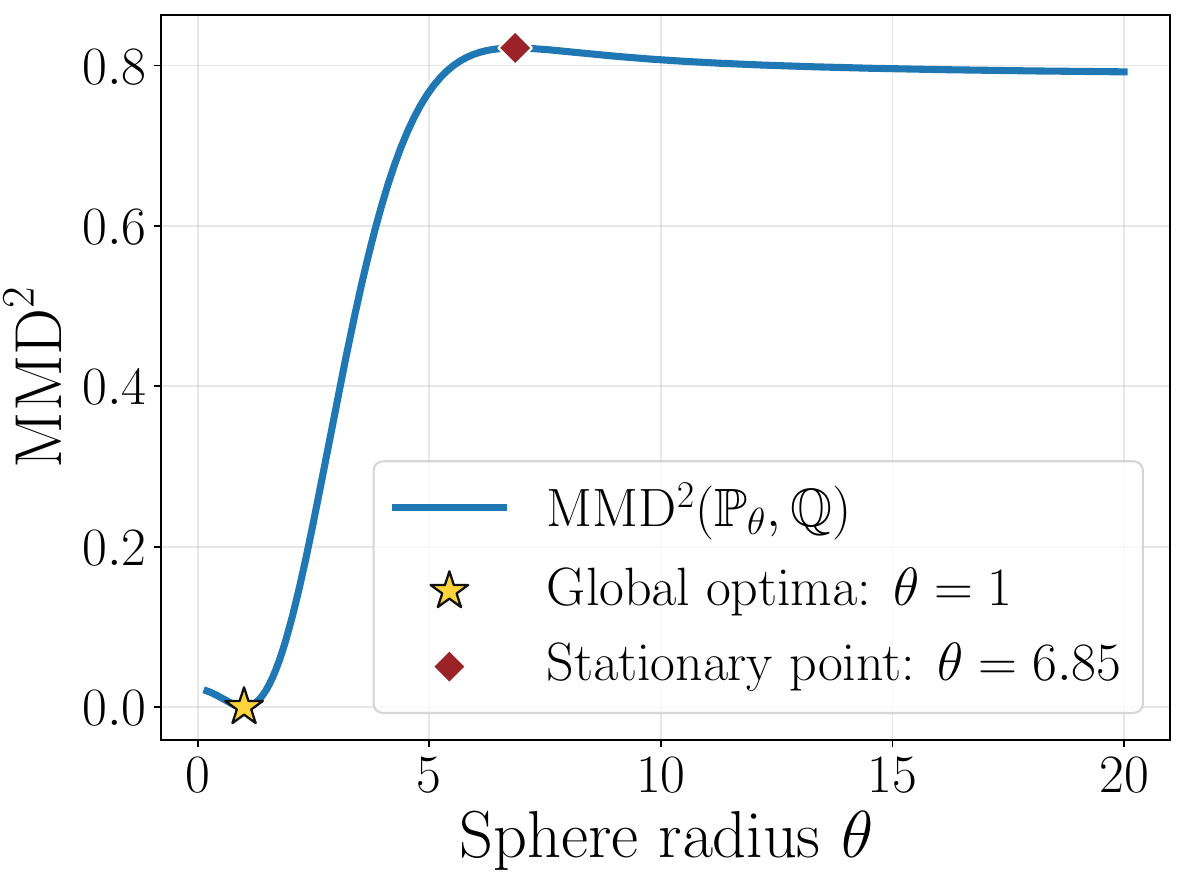}
    \caption{The example of $\Qb= \operatorname{Unif}(\mathbb{S}^{d-1})$ has stationary points that are not global minima.}
    \label{fig:sphere}
\end{figure}

\begin{lem}[Local minima of parametric and nonparametric descent]
\label{lem:residual_excludes_spurious_minima}
Let $\theta^\star$ be a stationary point of the parametric objective $\theta\mapsto \mmd_\ell^2(\Pb_\theta, \Qb)$. 
Then, for \Cref{ass:residual} to hold, $\Pb_{\theta^\star}$ must also be a stationary point of the objective $\Pb\mapsto \mmd_\ell^2(\Pb, \Qb)$, with respect to the Wasserstein-2 gradient. 
\end{lem}
\begin{proof}
Since $\theta^\star$ is a stationary point, then, for any direction $u\in\R^p$, the first order condition gives
\begin{equation*}
    0 = \mathbb{E}_{z\sim\rho}\left[\nabla f_{\ell,\Pb_{\theta^\star},\Qb}(\Gmap_{\theta^\star}(z))^\top \bJ_\theta \Gmap_{\theta^{\star}}(z) u \right] .
\end{equation*}
In other words, at such a stationary point, the function $z\mapsto \nabla f_{\ell,\Pb_{\theta^\star},\Qb} (\Gmap_{\theta^\star}(z))$ is orthogonal in $L_2(\rho)$ to every tangent direction of the form $z\mapsto \bJ_\theta \Gmap_{\theta^\star}(z) u$ with $u\in\R^p$. Indeed, consider the least-squares residual from \Cref{ass:residual} at $\theta^\star$:
\begin{align*}
    &\mathbb{E}_{z\sim\rho}\left[\left\| \bJ_\theta \Gmap_{\theta^\star}(z) u - \nabla f_{\ell,\Pb_{\theta^\star},\Qb} (\Gmap_{\theta^\star}(z))\right\|^2 \right]  \\
    &=
    \mathbb{E}_{z\sim\rho}\left[\left\| \bJ_\theta \Gmap_{\theta^\star}(z) u\right\|^2 \right]
    +
    \mathbb{E}_{z\sim\rho}\left[\left\|  \nabla f_{\ell,\Pb_{\theta^\star},\Qb} (\Gmap_{\theta^\star}(z))\right\|^2 \right]
    -
    2 \mathbb{E}_{z\sim\rho}\left[  \nabla f_{\ell,\Pb_{\theta^\star},\Qb} (\Gmap_{\theta^\star}(z))^\top \bJ_\theta \Gmap_{\theta^\star}(z) u  \right] \\
    &=
    \mathbb{E}_{z\sim\rho}\left[\left\| \bJ_\theta \Gmap_{\theta^\star}(z) u\right\|^2 \right]
    +
    \mathbb{E}_{z\sim\rho}\left[\left\|  \nabla f_{\ell,\Pb_{\theta^\star},\Qb} (\Gmap_{\theta^\star}(z))\right\|^2 \right]
\end{align*}
where the cross term vanishes due to the first-order condition from above. Thus the minimiser of the residual is obtained by any direction $u$ which satisfies $\mathbb{E}_{\rho}\left[\left\| \bJ_\theta \Gmap_{\theta^\star}(z) u\right\|^2 \right]=0$, in which case
\begin{align*}
    \min_{u\in\R^p}\mathbb{E}_{z\sim\rho}\left[\left\| \bJ_\theta \Gmap_{\theta^\star}(z) u - \nabla f_{\ell,\Pb_{\theta^\star},\Qb} (\Gmap_{\theta^\star}(z))\right\|^2 \right] = \left\|  \nabla f_{\ell,\Pb_{\theta^\star},\Qb}\right\|_{L_2(\Pb_{\theta^\star})}^2
\end{align*}
For \Cref{ass:residual} to hold, we would then require
\begin{align*}
    \left\|  \nabla f_{\ell,\Pb_{\theta^\star},\Qb} \right\|_{L_2(\Pb_{\theta^\star})}^2   \leq \mathfrak{R}^2 \left\|  \nabla f_{\ell,\Pb_{\theta^\star},\Qb} \right\|_{L_2(\Pb_{\theta^\star})}^2
\end{align*}
for some fixed $0<\mathfrak{R}<1$. This can only occur when $\left\|  \nabla f_{\ell,\Pb_{\theta^\star},\Qb} \right\|_{L_2(\Pb_{\theta^\star})}=0$, meaning that $\Pb_{\theta^\star}$ is a stationary point, in the sense of Wasserstein gradient, of the objective $\Pb\mapsto \mmd^2_\ell(\Pb,\Qb)$~\citep{chen2025stationary}.
\end{proof}

\begin{lem}[Example for \Cref{ass:residual}]\label{lem:example}
    Let $\Qb$ be a radial distribution on $\R^d$. Let $\Pb_\theta = (\Gmap_\theta)_\# \operatorname{Unif} (\mathbb{S}^{d-1})$ be a parametric family with a pushforward map $\Gmap_\theta: z\mapsto \theta z$ and $\theta\in\R$. 
    Let $k_\ell$ be a Gaussian kernel with lengthscale $\ell>0$. Then, for any fixed $\theta\in\R$, $\nabla f_{\ell, \Pb_\theta, \Qb}(\Gmap_\theta(\cdot )) \in \mathrm{span}\{ \bJ_\theta \Gmap_\theta(\cdot) u : u\in\R\}$. Moreover, the representing coefficient is uniformly bounded over $\theta\in\R$. 
\end{lem}
\begin{proof}
    Since the Gaussian kernel is radial, $k_\ell(x,y)=\exp(-(2\ell^2)^{-1} \|x-y\|^2)$ with lengthscale $\ell>0$, the witness function, $f_{\ell, \Pb_{\theta}, \Qb}(x)=\int k_{\ell}(x, y) \; \dd ( \Pb_{\theta}- \Qb)(y)$ is the convolution of radial objects, and hence also radial. Therefore, there exists a scalar function $\varphi_\theta:[0, \infty) \rightarrow \mathbb{R}$ such that $f_{\ell, \Pb_\theta, \Qb}(x) = \varphi_\theta(\|x\|)$. 
    For $\theta\neq0$ and $z\in\mathbb{S}^{d-1}$,
    \[
        \nabla f_{\ell, \Pb_\theta, \Qb}(\Gmap_\theta(z))
        = \varphi_\theta^{\prime}(|\theta|)
        \frac{\theta z}{|\theta|}
        = \operatorname{sign}(\theta)\varphi_\theta^{\prime}(|\theta|) z.
    \]
    When $\theta=0$, $f_{\ell,\Pb_0,\Qb}$ has zero gradient at the origin, so the conclusion holds trivially by taking the representing coefficient to be zero. 
    On the other hand, we have $\bJ_\theta \Gmap_\theta(z) = z$. 
    Hence, there exists a $u=\operatorname{sign}(\theta)\varphi_\theta^{\prime}(|\theta|)\in\R$, 
    depending only on $\theta$, such that $\nabla f_{\ell, \Pb_\theta, \Qb}(\Gmap_\theta(z)) = \bJ_\theta \Gmap_\theta(z) u$ holds for any $z\in\mathbb{S}^{d-1}$.
    Finally, the representing coefficient satisfies
    $|u|\leq \sup_x \|\nabla f_{\ell,\Pb_\theta,\Qb}(x)\| \leq 2 \ell^{-1}$, since the Gaussian-kernel gradient is uniformly bounded by $\ell^{-1}$. 
    The proof is thus concluded.
\end{proof}

\begin{rem}
    In \Cref{example:ass}, even in the simplest case where $\Qb=\operatorname{Unif}(\mathbb{S}^{d-1})$ and $\Pb_\theta = (\Gmap_\theta)_\# \operatorname{Unif} (\mathbb{S}^{d-1})$ with $\theta$ representing the radius, the objective $\theta\mapsto \mmd^2(\Pb_\theta,\Qb)$ still has local minima that are not global minima for a fixed lengthscale $\ell>0$; see \Cref{fig:sphere}. This further illustrates the importance of adaptive kernel lengthscale in our PGD scheme.
\end{rem}

\begin{lem}[Continuity in lengthscale]\label{lem:lengthscale_bound}
Suppose \Cref{ass:kernel} holds. Let lengthscale $\ell > 0$. 
Let $f_{\ell, \Pb,\Qb}(\cdot) = \int k_\ell(x, \cdot) \dd (\Pb-\Qb)(x)$ be the associated witness function.
Then, for $\ell_1 > \ell_2 > 0$, we have
\begin{align*}
    \left\| \nabla f_{\ell_1, \Pb,\Qb}(x) - \nabla f_{\ell_2, \Pb,\Qb}(x) \right\| \leq 2 C_k \frac{\ell_1^2 - \ell_2^2}{\ell_2^3} .
\end{align*}
Here, $C_k$ is a constant depending only on the kernel but not on the lengthscale. 
\end{lem}

\begin{proof}
From \Cref{lem:lengthscale_lip_bound}, we have for any $x,y\in\R^d$,
\begin{align*}
    \left\| \nabla_2 k_{\ell_1}(x,y) - \nabla_2 k_{\ell_2}(x,y) \right\| \leq C_k \frac{\ell_1^2 - \ell_2^2}{\ell_2^3} ,
\end{align*}
where $C_k$ is a constant depending only on the kernel but not on the lengthscale.
So we have
\begin{align*}
    &\quad \left\| \nabla f_{\ell_1, \Pb,\Qb}(y) - \nabla f_{\ell_2, \Pb,\Qb}(y) \right\| = \left\| \int \nabla_2 k_{\ell_1}(x,y) - \nabla_2 k_{\ell_2}(x,y) \dd (\Pb - \Qb)(x) \right\| \\
    &\leq \int \left\| \nabla_2 k_{\ell_1}(x,y) - \nabla_2 k_{\ell_2}(x,y)  \right\| \dd |\Pb - \Qb|(x) \leq C_k \frac{\ell_1^2 - \ell_2^2}{\ell_2^3} \int \dd |\Pb - \Qb|(x) \leq 2C_k \frac{\ell_1^2 - \ell_2^2}{\ell_2^3} .
\end{align*}
In the above derivations, $|\Pb-\Qb|$ denotes the total variation measure associated with the signed measure $\Pb-\Qb$. The last factor of $2$ appears because the total variation norm is always bounded by $2$. 
\end{proof}

\begin{lem}[One-step descent inequality for MMD]\label{lem:descent}
Let kernel $k_\ell$ satisfy \Cref{ass:kernel} with lengthscale $\ell>0$. 
Let $G:\R^d\to\R^d$ and assume $G\in L_2^d(\Pb_0)$. 
Let $\Pb_0\in\calP_2(\R^d)$ and let the next iterate be $\Pb_1 = (\Id - \gamma G)_\# \Pb_0$.
Then, we have
\begin{align*}
    \mmd_\ell^2(\Pb_1 , \Qb) \leq \mmd_\ell^2(\Pb_0, \Qb) - 2\gamma \int \nabla f_{\ell, \Pb_0, \Qb}(x)^\top G(x) \dd \Pb_0(x) + \frac{C_k\gamma^2}{\ell^2} \|G\|^2_{L_2^d(\Pb_0)} .
\end{align*} 
Here, $C_k$ depends only on the kernel profile and not on the lengthscale. 
\end{lem}
\begin{proof}
    Define, for $t\in[0,1]$, $G_t(x):=x-t\gamma G(x)$ and $\Pb_t:=(G_t)_\#\Pb_0$, so that $\Pb_0=(G_0)_\# \Pb_0$ and $\Pb_1=(\mathrm{Id}-\gamma G)_\#\Pb_0$.
    Moreover, let $X_0\sim \Pb_0$, and set $X_t:=G_t(X_0)=X_0-t\gamma G(X_0)$, which gives us $X_t\sim \Pb_t$ and $\dot X_t=-\gamma G(X_0)$.
    Using the definition,
    \begin{align*}
        \mmd_\ell^2(\Pb_t, \mathbb{Q}) = \E\left[k_\ell(X_t,X_t')\right]-2\,\E\left[k_\ell(X_t,Z)\right] +\E\left[k_\ell(Z,Z')\right], \quad X_t, X_t' \sim \Pb_t, \quad
        Z, Z' \sim \Qb.
    \end{align*}
    We can differentiate with respect to $t$ and use that $\dot X_t = -\gamma G(X_0)$ and $\dot X_t' = -\gamma G(X_0')$ to obtain
    \begin{align*}
        \frac{\dd}{\dd t} \mmd_\ell^2(\Pb_t, \mathbb{Q}) =-2\gamma\,\E_{X_0\sim \Pb_0} \left[
        \nabla f_{\ell,\Pb_t,\Qb}(X_0-t\gamma G(X_0) )^\top G(X_0)\right] .
    \end{align*}
    Then, by the fundamental theorem of calculus~\citep{rudin1976principles}, we have
\begin{align*}
    \mmd_\ell^2(\Pb_1,\Qb)
    -\mmd_\ell^2(\Pb_0,\Qb)
    &=
    -2\gamma\int_0^1
    \E_{X_0\sim \Pb_0}\left[
    \nabla f_{\ell,\Pb_t,\Qb}(X_0-t\gamma G(X_0))^\top G(X_0)
    \right] \dd t.
\end{align*}
    By writing
\begin{align*}
    \nabla f_{\ell,\Pb_t,\Qb}(X_0-t\gamma G(X_0))
    &=
    \nabla f_{\ell,\Pb_0,\Qb}(X_0)
    +
    (
    \nabla f_{\ell,\Pb_t,\Qb}(X_0-t\gamma G(X_0))
    -
    \nabla f_{\ell,\Pb_0,\Qb}(X_0)
    ),
\end{align*}
    the previous expression gives
    \begin{align}
    &\quad \mmd_\ell^2(\Pb_1,\Qb) =\mmd_\ell^2(\Pb_0,\Qb)
    -2\gamma\int \nabla f_{\ell,\Pb_0,\Qb}(x)^\top G(x)\,\dd \Pb_0(x)
    \nonumber\\
    &-2\gamma\int_0^1
    \E_{X_0\sim \Pb_0} \left[
    \Bigl(
    \nabla f_{\ell,\Pb_t,\Qb}(X_0-t\gamma G(X_0))
    -
    \nabla f_{\ell,\Pb_0,\Qb}(X_0)
    \Bigr)^\top G(X_0)
    \right] \dd t.
    \label{eq:quad}
    \end{align} 
    For the last term of the right hand side, by adding and subtracting $\nabla f_{\ell,\Pb_t,\Qb}(X_0)$ inside $\nabla f_{\ell,\Pb_t,\Qb}(X_0-t\gamma G(X_0))-\nabla f_{\ell,\Pb_0,\Qb}(X_0)$,
    \begin{align}
    &\quad -2\gamma\int_0^1
    \E_{X_0\sim \Pb_0}\left[
    \Bigl(
    \nabla f_{\ell,\Pb_t,\Qb}(X_0-t\gamma G(X_0))
    -
    \nabla f_{\ell,\Pb_0,\Qb}(X_0)
    \Bigr)^\top G(X_0)
    \right] \dd t \nonumber \\
    &=
    -2\gamma \int_0^1
    \E_{X_0\sim \Pb_0} \left[
    \Bigl(\nabla f_{\ell,\Pb_t,\Qb}(X_0-t\gamma G(X_0))-\nabla f_{\ell,\Pb_t,\Qb}(X_0)\Bigr)^\top G(X_0)
    \right] \dd t
    \label{eq:first_term_lemma} \\
    &\qquad -2\gamma \int_0^1
    \E_{X_0\sim \Pb_0} \left[
    \Bigl(\nabla f_{\ell,\Pb_t,\Qb}(X_0)-\nabla f_{\ell,\Pb_0,\Qb}(X_0)\Bigr)^\top G(X_0)
    \right] \dd t. \label{eq:second_term_lemma}
    \end{align}
    Since $x\mapsto \nabla f_{\ell,\Pb_t,\Qb}(x)$ is $(C_k/\ell^2)$-Lipschitz for the Gaussian kernel and for the Mat\'ern kernel with smoothness $s>1$,
    \begin{align*}
    \bigl\|
    \nabla f_{\ell,\Pb_t,\Qb}(X_t)
    -
    \nabla f_{\ell,\Pb_t,\Qb}(X_0)
    \bigr\|
    \leq
    \frac{C_k}{\ell^2} \|X_t-X_0\|
    =
    \frac{C_k t\gamma}{\ell^2}\|G(X_0)\|.
    \end{align*}
    $C_k$ depends only on the kernel profile and not on the lengthscale. 
    Therefore,
    \begin{align*}
    \eqref{eq:first_term_lemma}
    &\le
    2\gamma \int_0^1
    \E_{X_0\sim \Pb_0} \Big[
    \|\nabla f_{\ell,\Pb_t,\Qb}(X_t)-\nabla f_{\ell,\Pb_t,\Qb}(X_0)\| \cdot
    \,\|G(X_0)\|
    \Big]dt
    \\
    &\le
    2\gamma \int_0^1
    \frac{C_k\gamma t}{\ell^2} \,
    \E_{X_0\sim \Pb_0}\left[\|G(X_0)\|^2\right]dt
    =
    \frac{ C_k\gamma^2}{\ell^2} \,\|G\|_{L^2(\Pb_0)}^2.
    \end{align*}
    Moreover, using the coupling $Y_t=Y_0-\gamma tG(Y_0)$ with $Y_0\sim \Pb_0$ and the $(C_k/\ell^2)$-Lipschitzness of the derivative of $k_\ell$, for every $x\in\mathbb{R}^d$, we have 
    \begin{align*}
    \bigl\|
    \nabla f_{\ell,\Pb_t,\Qb}(x)
    -
    \nabla f_{\ell,\Pb_0,\Qb}(x)
    \bigr\|
    \leq 
    \E_{Y_0,Y_t} \bigl[\|\nabla_1 k_\ell(x, Y_t)-\nabla_1 k_\ell(x, Y_0)\|]
    \leq
    \frac{C_k t\gamma}{\ell^2} \,
    \E\bigl[\|G(Y_0)\|\bigr].
    \end{align*}
    Hence, by Cauchy-Schwarz, integrating over $t\in[0,1]$,
    \begin{align*}
        \eqref{eq:second_term_lemma}
        \le
        2\gamma\int_0^1
        \frac{C_k t\gamma}{\ell^2} \,
        \E\bigl[\|G(X_0)\|^2\bigr] \dd t = \frac{C_k\gamma^2}{\ell^2} \|G\|_{L_2^d(\Pb_0)}^2.
    \end{align*}
    Combining the last two upper bounds on Eq.~\eqref{eq:first_term_lemma} and Eq.~\eqref{eq:second_term_lemma}, along with inequalities with \eqref{eq:quad}, we obtain
    \begin{align*}
    \mmd_\ell^2(\Pb_1 , \Qb) \leq \mmd_\ell^2(\Pb_0, \Qb) - 2\gamma \int \nabla f_{\ell, \Pb_0, \Qb}(x)^\top G(x) \dd \Pb_0(x) + \frac{C_k\gamma^2}{\ell^2}  \|G\|^2_{L_2^d(\Pb_0)} .
    \end{align*}
\end{proof}

\begin{lem}[Bandwidth adjustment via convolution with noise]\label{lem:convolution}
Let $k_\ell$ be either the Gaussian kernel or the Mat\'ern kernel with smoothness $s$ stated in \Cref{ass:kernel}.
Let $f_{\ell, \Pb,\Qb}(\cdot) = \int k_\ell(x, \cdot) \dd (\Pb-\Qb)(x)$ be the associated witness function.
Let $\ell' > \ell >0$.
There exists a probability measure $\nu$ on $\R^d$ such that, for any $x\in\R^d$,
\begin{align}
    \E_{U\sim\nu}\big[\nabla f_{\ell, \Pb, \Qb}(x+ U) \big]
    =
    \left( \ell \ell'^{-1}
    \right)^{d}
    \nabla  f_{\ell', \Pb, \Qb}(x), \quad \E_{U\sim\nu}[\|U\|]\leq C_{d,s} \sqrt{{\ell^\prime}^2-{\ell}^2}.
\end{align}
Here, $C_{d,s}$ is a constant depending only on the dimension $d$ and the smoothness parameter $s$.
\end{lem}
\begin{proof}
Notice that
\begin{align*}
    \E_{U\sim\nu} \bigl[\nabla f_{\ell,\Pb,\Qb}(x+ U) \bigr]
    &= \int_{\mathbb R^d} \int_{\mathbb R^d}\nabla_1 k_\ell(x+u,y)\,\dd (\Pb-\Qb)(y) \, \dd \nu(u)
    \\
    &= \int_{\mathbb R^d}\nabla_x \left(\int_{\mathbb R^d} k_\ell(x+u,y)\,\, \dd \nu(u) \right) \dd(\Pb-\Qb)(y),
\end{align*}
where the second equality is from using Fubini's theorem~\citep{rudin1976principles} given that $k_\ell$ is a Gaussian or a Mat\'ern kernel hence its gradient is uniformly bounded. Now, for fixed $y\in\mathbb R^d$, based on \Cref{lem:noise_pert}, the inner integral becomes
\begin{align*}
     \int_{\mathbb R^d} k_\ell(x+u,y)\,\dd \nu(u)
     &= \left( \ell \ell'^{-1} \right)^d k_{\ell'}(x,y) .
\end{align*}
And, we obtain
\begin{align*}
    \E_{U\sim\nu} \bigl[\nabla f_{\ell,\Pb,\Qb}(x+ U) \bigr]
    =\int_{\mathbb R^d}
    \left( \ell \ell'^{-1} \right)^d
    \nabla_1 k_{\ell'} (x,y) \,
    \dd (\Pb-\Qb)(y)
    =\left( \ell \ell'^{-1} \right)^d \nabla f_{\ell',\,\Pb,\Qb}(x).
\end{align*}
The first moment bound on $U\sim\nu$ is proved in \Cref{lem:noise_pert}.
\end{proof}

\noindent
\textbf{Notations:} We denote $\calF[f]:\R^d\to\R$ as the Fourier transform of a function $f:\R^d\to\R$.

\begin{lem}[Noise perturbation]\label{lem:noise_pert}
    Let $k_\ell$ be either the Gaussian kernel or the Mat\'ern kernel with smoothness $s$ stated in \Cref{ass:kernel}. Then, for any $\ell'>\ell>0$, there exists a probability measure $\nu$ on $\mathbb{R}^d$ such that $\E_{U\sim \nu}[k_\ell(x+U, y)] = \left( \ell \ell'^{-1} \right)^d k_{\ell'}(x,y)$ holds for any $x,y\in\R^d$.
    Moreover, $\E_{U\sim\nu}[\|U\|]\leq C_{d,s} \sqrt{{\ell^\prime}^2-{\ell}^2}$ for some constant $C_{d,s}$ depending only on the dimension $d$ and the smoothness parameter $s$.
\end{lem}
\begin{proof}
    Write $k_\ell(x,y)=\psi_\ell(x-y)$, and denote $S_\ell=\calF[\psi_\ell]:\R^d\to\R$ as the Fourier transform of $\psi_\ell$.
    It is enough to find a probability measure $\nu$ with finite first moment, symmetric with respect to $0$, and its characteristic function satisfies
    \begin{align}\label{eq:goal}
        \widehat \nu(\omega) := \int e^{-i \omega^{\top} u} \, \dd \nu(u) = \left(\ell \ell'^{-1}\right)^d \frac{S_{\ell'}(\omega)}{S_\ell(\omega)}.
    \end{align}
    Indeed, based on the convolution theorem for Fourier transforms (see \cite{katznelson2004introduction}, Chapter VI, Section 2, pp. 144--145), this gives
    \begin{align*}
        \mathcal{F}\left[ \E_{U\sim\nu}\psi_\ell(\cdot+U) \right](\omega) &= \calF\left[ \int_{\mathbb{R}^d} \psi_{\ell}(\cdot + u) \dd \nu(u) \right] = \calF\left[ \int_{\mathbb{R}^d} \psi_{\ell}(\cdot - u) \dd \nu(u) \right] \\
        &= S_\ell(\omega) \cdot \widehat \nu(\omega) = \left(\ell \ell'^{-1}\right)^d S_{\ell'}(\omega),
    \end{align*}
    and hence we obtain the desired result:
    $\E_{U\sim \nu}\big[k_\ell(x+U,y)\big] = \left(\ell \ell'^{-1}\right)^d k_{\ell'}(x,y).$

    Next, we are going to prove Eq.~\eqref{eq:goal} for the Gaussian and Mat\'{e}rn kernels. For the Gaussian kernel, we have
    $S_\ell(\omega) = C_d \ell^d \exp(-\frac{\ell^2\|\omega\|^2}{2})$.
    Therefore \[ \left(\ell \ell'^{-1}\right)^d \frac{S_{\ell'}(\omega)}{S_\ell(\omega)} = \exp\left( -\frac{(\ell'^2-\ell^2)\|\omega\|^2}{2} \right), \] which is the characteristic function of a normal distribution $\nu=\calN(0,(\ell'^2-\ell^2)\Id)$. It is symmetric with respect to $0$ and has bounded first moment $\E_{U\sim\nu}[\|U\|]\leq \sqrt{\E_{U\sim\nu}[\|U\|^2]} = \sqrt{d\left(\ell'^2-\ell^2\right)}$.
    Thus the claim follows in the Gaussian case.

    Next, for the Mat\'ern kernel with smoothness parameter $s$, its Fourier transform is of the following form~\citep{bach2024learning}:
    \[
    S_\ell(\omega) = C_{s,d}\ell^d \left( 1+\frac{\ell^2\|\omega\|^2}{2s} \right)^{-(s+d/2)} . \] 
    Hence
\begin{align*}
    \left(\ell \ell'^{-1}\right)^d \frac{S_{\ell'}(\omega)}{S_\ell(\omega)} = \left( \frac{1+\frac{\ell^2\|\omega\|^2}{2s}} {1+\frac{\ell'^2\|\omega\|^2}{2s}} \right)^{s+d/2}.
\end{align*}
    Since $\ell'>\ell$, we claim that the function
\begin{align}\label{eq:g_function}
    g:r\mapsto 
    \left(
    \frac{1+\frac{\ell^2 r}{2s}}
         {1+\frac{\ell'^2 r}{2s}}
    \right)^{s+d/2},
    \qquad r\geq 0,
\end{align}
is completely monotone\footnote{
A function $f:(0,\infty)\to\mathbb{R}$ is completely monotone if 
$f\in C^\infty(0,\infty)$ and 
$(-1)^n f^{(n)}(\lambda)\geq 0$ for all 
$n\in\mathbb{N}\cup\{0\}$ and $\lambda>0$; see 
\cite[Definition~1.3]{schilling2012bernstein}.
}
and satisfies $g(0)=1$. 
To prove the claim, set $h(r):=-\log g(r)$. Then, $h(r) = \left(s+\frac d2\right) \left[
\log\left(1+\frac{\ell'^2 r}{2s}\right) - \log\left(1+\frac{\ell^2 r}{2s}\right)
\right]$, and hence $h'(r) = \left(s+\frac d2\right)
\frac{(\ell'^2-\ell^2)/(2s)}
{\left(1+\frac{\ell^2 r}{2s}\right)
\left(1+\frac{\ell'^2 r}{2s}\right)}$.
 The prefactor is positive because $\ell'>\ell$. Moreover, for every $a>0$,
the map $r\mapsto (1+ar)^{-1}$ is completely monotone, since
\[
    (-1)^n \frac{\dd^n}{\dd r^n}(1+ar)^{-1}
    = n!a^n(1+ar)^{-n-1} \geq 0 , \quad \forall n\in\N. 
\]
Therefore both factors
\[
    r\mapsto \left(1+\frac{\ell^2 r}{2s}\right)^{-1},
    \qquad
    r\mapsto \left(1+\frac{\ell'^2 r}{2s}\right)^{-1}
\]
are completely monotone. Since products of completely monotone functions are
completely monotone~\citep[Corollary~1.6]{schilling2012bernstein}, it follows
that $h'$ is completely monotone.

Thus $h$ is a Bernstein function\footnote{
A nonnegative function $h\in C^\infty(0,\infty)$ is a Bernstein
function if $h'$ is completely monotone; see 
\cite[Definition~3.1]{schilling2012bernstein}.
}.
Indeed, $h(0)=0$, $h'\geq 0$, and $h'$ is completely monotone. By the
standard Bernstein-function characterization, if $h$ is Bernstein, then
$\exp(-u h)$ is completely monotone for every $u>0$
\citep[Theorem~3.7]{schilling2012bernstein}. Taking $u=1$, we conclude that
\[
    g(r)=\exp(-h(r))
\]
is completely monotone. Also $g(0)=1$ follows immediately from
Eq.~\eqref{eq:g_function}. 
By Schoenberg's theorem~\citep[Theorem~3]{schoenberg1938metric}, every
completely monotone function $g$ with $g(0)=1$ gives rise to a
positive-definite radial function
$\omega\mapsto g(\|\omega\|^2)$
on $\mathbb{R}^d$. Consequently,
\[
    \hat{\nu}(\omega)
    :=
    g(\|\omega\|^2)
    =
    \left(
    \frac{1+\frac{\ell^2\|\omega\|^2}{2s}}
         {1+\frac{\ell'^2\|\omega\|^2}{2s}}
    \right)^{s+d/2}
\]
is positive definite on $\mathbb{R}^d$. 
    Consequently, since $\hat{\nu}$ is continuous and $\hat{\nu}(0)=1$, Bochner's theorem~\citep{bochner2005harmonic} implies that it is the characteristic function of a probability measure on $\mathbb{R}^d$.
    Moreover, notice that the first-order derivative at the origin of the mapping in Eq.~\eqref{eq:g_function} equals $\left(s+\frac{d}{2}\right) \frac{\ell^2-{\ell^\prime}^2}{2 s}$, which is finite.
    Moreover,
    \begin{align*}
        \E_{U\sim\nu}[\|U\|^2]= -\operatorname{Tr}\left[\nabla_\omega^2\hat{\nu}(\omega)\mid_{\omega=0} \right] = -2 g^\prime(0) d = d \left(s+\frac{d}{2}\right) \frac{{\ell^\prime}^2-{\ell}^2}{s} < \infty.
    \end{align*}
    Thus, the associated probability measure $\nu$ has finite second moment. Hence its first moment satisfies $\E_{U\sim\nu}[\|U\|]\leq \sqrt{d \left(1+\frac{d}{2s}\right)} \cdot \sqrt{{\ell^\prime}^2-\ell^2}$. This concludes the proof of the Mat\'ern case.
\end{proof}

\begin{rem}[Extension to more general kernels]\label{rem:extension_general_kernel}
    \Cref{lem:convolution} shows that, for the Gaussian or Mat\'ern kernel, MMD flow with noise injection is equivalent to MMD flow with an increased kernel lengthscale. This observation was first made in \cite{glaser2020regularizing}.
    The above equivalence can be generalized to other translation-invariant kernels $k_\ell(x,y) = \psi_\ell(x-y)$ and a centered symmetric probability measure $\nu$, as long as there exists a $\ell'>0$ such that $\mathcal{F}[\psi_\ell]\,\widehat{\nu} = \mathcal{F}[\psi_{\ell'}]$.
    \cite{nishiyama2016characteristic} provides a more general family of kernels and distributions for which such relation holds, which they refer to as ``conjugacy''.
    In this paper, we focus explicitly on the Gaussian and Mat\'ern kernels, which are widely used in minimum MMD estimation~\citep{briol2019statistical,cherief2022finite}, leaving other kernels as future work.
\end{rem}

\end{appendices}

\end{document}